\documentclass[final,12pt]{colt2026} 

\title[Near-Optimal Regret for Distributed Adversarial Bandits]{
Near-Optimal Regret for Distributed Adversarial Bandits: \\ A Black-Box Approach
}
\usepackage{times}

\usepackage{bbm}
\usepackage{algorithm} 
\usepackage{bm}
\usepackage{cleveref}
\usepackage{enumitem}
\usepackage{mathabx}
\definecolor{Green}{rgb}{0.13, 0.65, 0.3}
\allowdisplaybreaks

\DeclareMathOperator*{\argmax}{arg\,max}
\DeclareMathOperator*{\argmin}{arg\,min}

\newcommand{\calT}{\mathcal{T}}

\newcommand{\calA}{\mathcal{A}}
\newcommand{\calB}{\mathcal{B}}

\newcommand{\calS}{\mathcal{S}}

\newcommand{\calF}{\mathcal{F}}
\newcommand{\ind}{\mathbbm{1}}

\newcommand{\bb}{\boldsymbol{b}}

\newcommand{\wt}{\widetilde}
\newcommand{\otil}{\wt{\order}}

\newcommand{\E}{\mathbb{E}}
\newcommand{\order}{\mathcal{O}}

\newcommand{\calP}{\mathcal{P}}

\newcommand{\inner}[1]{\left \langle#1\right \rangle}

\newcommand{\calN}{\mathcal{N}}

\newcommand{\alg}{\ensuremath{\mathcal{A}}}

\newcommand{\R}{\mathbb{R}}

\usepackage{thmtools}
\usepackage{thm-restate}

\renewcommand{\overline}{\widebar}
\newcommand{\bq}{\mathbf{q}}
\newcommand{\bp}{\mathbf{p}}
\newcommand{\bmu}{\boldsymbol{\mu}}
\newcommand{\bdelta}{\boldsymbol{\delta}}
\newcommand{\be}{\mathbf{e}}

\newcommand{\bell}{\boldsymbol{\ell}}
\newcommand{\ellbar}{\overline{\ell}}
\newcommand{\bellbar}{\overline{\bell}}
\newcommand{\ba}{\boldsymbol{a}}

\newcommand{\boldb}{\boldsymbol{b}}
\renewcommand{\be}{\boldsymbol{e}}
\newcommand{\bg}{\boldsymbol{g}}
\newcommand{\bu}{\boldsymbol{u}}
\newcommand{\bx}{\boldsymbol{x}}
\newcommand{\by}{\boldsymbol{y}}
\newcommand{\bv}{\boldsymbol{v}}
\newcommand{\bz}{\boldsymbol{z}}
\renewcommand{\bp}{\boldsymbol{p}}
\renewcommand{\bq}{\boldsymbol{q}}
\newcommand{\bxi}{\boldsymbol{\xi}}
\newcommand{\btheta}{\boldsymbol{\theta}}

\newcommand{\bthetahat}{\widehat{\btheta}}
\newcommand{\Reg}{{\mathrm{Reg}}}

\newcommand{\ellhat}{\widehat{\ell}}
\newcommand{\bellhat}{\widehat{\bell}}
\renewcommand{\bz}{\boldsymbol{z}}
\newcommand{\ztilde}{\widetilde{z}}
\newcommand{\bztilde}{\widetilde{\bz}}
\newcommand{\barz}{\overline{z}}
\newcommand{\barbz}{\overline{\bz}}
\newcommand{\barp}{\overline{p}}
\newcommand{\barbp}{\overline{\bp}}
\newcommand{\barq}{\overline{q}}
\newcommand{\barbq}{\overline{\bq}}
\newcommand{\barX}{\overline{X}}
\newcommand{\bzero}{\boldsymbol{0}}
\newcommand{\bone}{\boldsymbol{1}}

\DeclareMathOperator{\tr}{tr}
\newcommand{\Regdel}{\mathsf{R}^{\mathrm{del}}}
\newcommand{\wtz}{\widetilde{z}}
\newcommand{\wtbz}{\widetilde{\bz}}
\newcommand{\dmax}{d_{\mathrm{max}}}
\newcommand{\barM}{\overline{M}}
\newcommand{\barbx}{\overline{\bx}}
\newcommand{\wtbp}{\widetilde{\bp}}

\newcommand{\barzS}{\overline{z}^{\calS}}
\newcommand{\barbzS}{\overline{\bz}^{\calS}}
\newcommand{\blambda}{\boldsymbol{\lambda}}

\declaretheorem[name=Theorem]{rtheorem} %
\declaretheorem[name=Lemma, sibling=rtheorem]{rlemma}

\usepackage{prettyref}
\newcommand{\pref}[1]{\prettyref{#1}}

\newcommand{\savehyperref}[2]{\texorpdfstring{\hyperref[#1]{#2}}{#2}}
\newrefformat{eq}{\savehyperref{#1}{Eq.~\textup{(\ref*{#1})}}}
\newrefformat{eqn}{\savehyperref{#1}{Eq.~(\ref*{#1})}}
\newrefformat{lem}{\savehyperref{#1}{Lemma~\ref*{#1}}}
\newrefformat{def}{\savehyperref{#1}{Definition~\ref*{#1}}}
\newrefformat{line}{\savehyperref{#1}{Line~\ref*{#1}}}
\newrefformat{thm}{\savehyperref{#1}{Theorem~\ref*{#1}}}
\newrefformat{corr}{\savehyperref{#1}{Corollary~\ref*{#1}}}
\newrefformat{cor}{\savehyperref{#1}{Corollary~\ref*{#1}}}
\newrefformat{sec}{\savehyperref{#1}{Section~\ref*{#1}}}
\newrefformat{app}{\savehyperref{#1}{Appendix~\ref*{#1}}}
\newrefformat{assum}{\savehyperref{#1}{Assumption~\ref*{#1}}}
\newrefformat{asm}{\savehyperref{#1}{Assumption~\ref*{#1}}}
\newrefformat{ex}{\savehyperref{#1}{Example~\ref*{#1}}}
\newrefformat{fig}{\savehyperref{#1}{Figure~\ref*{#1}}}
\newrefformat{alg}{\savehyperref{#1}{Algorithm~\ref*{#1}}}
\newrefformat{rem}{\savehyperref{#1}{Remark~\ref*{#1}}}
\newrefformat{conj}{\savehyperref{#1}{Conjecture~\ref*{#1}}}
\newrefformat{prop}{\savehyperref{#1}{Proposition~\ref*{#1}}}
\newrefformat{proto}{\savehyperref{#1}{Protocol~\ref*{#1}}}
\newrefformat{prob}{\savehyperref{#1}{Problem~\ref*{#1}}}
\newrefformat{claim}{\savehyperref{#1}{Claim~\ref*{#1}}}
\newrefformat{que}{\savehyperref{#1}{Question~\ref*{#1}}}
\newrefformat{op}{\savehyperref{#1}{Open Problem~\ref*{#1}}}
\newrefformat{fn}{\savehyperref{#1}{Footnote~\ref*{#1}}}
\newrefformat{eve}{\savehyperref{#1}{Event~\ref*{#1}}}
\newrefformat{ineq}{\savehyperref{#1}{Eq.~(\ref*{#1})}}

\coltauthor{%
 \Name{Hao Qiu}\thanks{Equal contribution, in alphabetical order.} \Email{qiuhaosai@gmail.com}\\
 \addr Università degli Studi di Milano
 \AND
 \Name{Mengxiao Zhang}\footnotemark[1] \Email{mengxiao-zhang@uiowa.edu}\\
 \addr University of Iowa%
  \AND
 \Name{Nicolò Cesa-Bianchi} \Email{nicolo.cesa-bianchi@unimi.it}\\
 \addr Università degli Studi di Milano
}

\begin{document}

\maketitle

\begin{abstract}
We study distributed adversarial bandits, where $N$ agents cooperate to minimize the global average loss while observing only their own local losses. We show that the minimax regret for this problem is $\widetilde{\Theta}\Big(\sqrt{\left(\rho^{-1/2} + \frac{K}{N}\right)T}\Big)$, where $T$ is the horizon, $K$ is the number of actions, and $\rho$ is the spectral gap of the communication matrix. Our algorithm, based on a novel black-box reduction to bandits with delayed feedback, requires agents to communicate only through gossip. It achieves an upper bound that significantly improves over the previous best bound $\widetilde{\mathcal{O}}\left(\rho^{-1/3}(KT)^{2/3}\right)$ of \citet{yi2023doubly}. We complement this result with a matching lower bound, showing that the problem's difficulty decomposes into a communication cost $\rho^{-1/4}\sqrt{T}$ and a bandit cost $\sqrt{KT/N}$. We further demonstrate the versatility of our approach by deriving first-order and best-of-both-worlds bounds in the distributed adversarial setting. Finally, we extend our framework to distributed linear bandits in $\R^d$, obtaining a regret bound of $\widetilde{\mathcal{O}}\Big(\sqrt{\left(\rho^{-1/2} + \frac{1}{N}\right)dT}\Big)$, achieved with only $\mathcal{O}(d)$ communication cost per agent and per round via a volumetric spanner.
\end{abstract}


\section{Introduction}
\label{sec: intro}
Distributed online optimization~\citep{yan2012distributed,hosseini2013online,wan2024nearly} studies how multiple agents can cooperate to minimize regret in a decentralized fashion. The core challenge in this setting is the misalignment between the global objective and the local information: agents aim to minimize the average loss across the entire network, yet they only observe their own local loss functions. To bridge this gap, agents must exchange information over a communication network, typically interacting only with their immediate neighbors via gossip protocols~\citep{xiao2004fast,boyd2006randomized,liu2011accelerated}. This framework is naturally motivated by applications in federated learning with heterogeneous clients~\citep{blaser2024federated} and decentralized collaborative learning~\citep{zantedeschi2020fully}. In this work, we investigate distributed online optimization under the challenging regime of bandit feedback: unlike full information settings, where agents observe the local losses for all actions, here agents only observe the local loss of their chosen action at each round, which they must propagate to optimize the global objective.

{In distributed online convex optimization (D-OCO), the regret scales with the connectivity of the network, formalized by the spectral gap $\rho$ of the gossip matrix. A small $\rho$ indicates network bottlenecks, which typically slow down the propagation of local losses. In the full information setting, the cost of this decentralization is well understood: the pioneering work of \citet{yan2012distributed,hosseini2013online} extends classical OCO algorithms to the distributed online setting, achieving an $\order(N^{1/4} \rho^{-1/2} \sqrt{T})$ regret, where $N$ is the number of agents and $T$ is the horizon. Building on this, \citet{wan2024nearly} utilized an accelerated variant of gossip to establish regret bounds of $\otil(\rho^{-1/4}\sqrt{T})$, which are optimal up to logarithmic factors.\footnote{We use $\otil(\cdot)$ and $\widetilde{\Theta}(\cdot)$ to suppress logarithmic factors in $T,N,K$.} However, for distributed $K$-armed bandits, the picture is less clear. The current state-of-the-art~\citep{yi2023doubly} achieves $\otil(\rho^{-1/3}(KT)^{2/3})$ regret, which is significantly worse than the $\sqrt{T}$ rate achieved in the full-information case. Moreover, the bound of \citet{yi2023doubly} does not reveal how the regret should depend on the number of agents $N$. \citet{seldin2014prediction} show that $K$-armed bandits with $N$ observations per round have minimax regret $\Theta(\sqrt{KT/N})$, suggesting that multiple bandit observations can in principle lead to improved rates. Although distributed $K$-armed bandits differ from this setting, they also generate $N$ local observations per round. These contrasts motivate the following question:
\begin{center}
\textit{What is the minimax regret rate for distributed $K$-armed bandits?}
\end{center} 
\paragraph{Contributions and technical challenges.}
In this work, we answer this question by providing a regret upper bound of 
$\widetilde{\order}\big(\sqrt{\left(\rho^{-1/2}+\frac{K}{N}\right)T}\big)$, 
thus significantly improving upon the previous $\otil \left(\rho^{-1/3}(KT)^{2/3}\right)$ bound of \citet{yi2023doubly}. 
We complement this with a matching lower bound, characterizing the minimax rate as a combination of a \emph{communication cost} of order $\rho^{-1/4}\sqrt{T}$ and a \emph{bandit information cost} of order $\sqrt{KT/N}$.
A key technical obstacle---leading to the $T^{2/3}$-type behavior in existing analyses---is the seemingly natural \emph{round-by-round} update paradigm: each agent updates immediately from its local importance-weighted bandit estimate while gossiping.
Due to the high variance of these estimates, the information exchanged through local gossip cannot be sufficiently mixed within a single round.
Consequently, the local updates made by each agent are not close to those based on the global average loss estimate, causing a large error term in the regret. We resolve this obstacle by explicitly \emph{decoupling} learning from communication via a block-based approach that intentionally introduces delayed feedback. Specifically, agents freeze their action distributions and buffer local bandit observations for the duration of a block. Within this block, they execute accelerated gossip iterations to ensure that, when the next block starts, every agent possesses a high-precision approximation of the global average loss from the previous block. Crucially, we incorporate a small $\Theta(1/T)$ uniform exploration term to ensure that the importance weights remain uniformly bounded. This allows the gossip protocol to achieve an accurate approximation of the global average loss estimate in only $\order\big(\rho^{-1/2}\log(KTN)\big)$ rounds. Ultimately, this construction yields a black-box reduction from distributed bandits with local gossip to standard adversarial bandits with delayed feedback~\citep{cesa2019delay,thune2019nonstochastic,bistritz2019online, gyorgy2021adapting, van2023unified}.

Furthermore, because our reduction is algorithm-agnostic, instantiating it with adaptive delayed-bandit algorithms~\citep{van2022nonstochastic, masoudian2022best} immediately confers adaptive guarantees to the distributed setting, including small-loss and best-of-both-worlds bounds. Finally, we extend our framework to distributed $d$-dimensional linear bandits with $K$ actions, a setting that---to the best of our knowledge---has not been previously studied. Specifically, we provide a $\widetilde{\order}\big(\sqrt{\big(\rho^{-1/2}+1/N\big)dT}\big)$ upper bound and a ${\Omega}\big(\sqrt{\big(\rho^{-1/2}+d/N\big)T}\big)$ lower bound. Crucially, we achieve this using only gossip communication with $\mathcal{O}(d)$-sized messages per round, matching the communication efficiency of established D-OCO protocols~\citep{yan2012distributed, hosseini2013online, wan2024nearly, wan2025black}.

\subsection{Related works}

\paragraph{Distributed optimization and gossiping.}
Our framework builds on the established line of work on gossip protocols \citep{xiao2004fast, boyd2006randomized}. Gossip is used to aggregate information over a communication network when each agent can only exchange information with its neighbors. Due to their simplicity, gossip algorithms have become a standard technique in distributed optimization. Early work on distributed convex optimization leveraged gossip (peer-to-peer) communication to achieve consensus and optimize a global objective. In classic setups, each agent holds a local convex cost and they seek to minimize the sum of costs without a central coordinator. For example, \citet{nedic2009distributed} developed a distributed method where agents repeatedly take gradient steps on local functions and gossip their updates with neighbors, guaranteeing convergence to the global optimum. Building on such consensus ideas, \citet{duchi2011dual} proposed a distributed dual averaging algorithm using gossip at each iteration to share gradient information. This method provided rigorous convergence rate bounds in terms of the network spectral properties (e.g., the second-largest eigenvalue of the gossip matrix). In distributed optimization \citet{scaman2019optimal} identified optimal algorithms and matching lower bounds for decentralized convex optimization, including the convex and strongly convex.
\paragraph{Distributed online convex optimization.}
Distributed online convex optimization (D-OCO) studies a network of agents that sequentially make decisions and incur convex losses, with the objective of collectively minimizing regret with respect to the global average loss function. In the full-information setting, each agent observes the entire loss function (or gradient) after making their decision, allowing the use of powerful online learning algorithms in a decentralized manner.  The pioneering work of \cite{yan2012distributed, hosseini2013online} extend OGD and dual average into the D-OCO setting and achieve $\order(N^{1/4} \rho^{-1/2} \sqrt{T})$ regret bound. Nearly optimal regret bounds $\otil(\rho^{-1/4}\sqrt{T})$ for convex functions have been established by \citet{wan2024nearly} using an accelerated gossip strategy, see also \citet{liu2011accelerated,ye2023heterogeneous}. They also provide $\Omega(\rho^{-1/4}\sqrt{T})$ lower regret bounds for convex loss functions. For comprehensive surveys on D-OCO, we refer the readers to \citet{li2023survey} and \citet{yuan2024multi}.

\paragraph{Distributed $K$-armed bandits.}
The problem becomes significantly more challenging under {bandit feedback}, where each agent observes only the scalar loss of their chosen action rather than the full loss vector. Prior literature has investigated this setting in both the {stochastic} environment, where losses are drawn i.i.d.~from fixed distributions, and the {adversarial} environment, where the loss sequences can be chosen arbitrarily. In a stochastic setting, previous works \citet{zhu2021federated, zhu2023distributed, xu2023decentralized, ZhangWCQYG25, liu2025distributed} study distributed $K$-armed bandits where agents run a gossip protocol \emph{augmented with additional communication}, achieving regret guarantees of  $\order(\log T)$.
The closest work to ours is \citet{yi2023doubly}, which studies distributed $K$-armed bandits using only a gossip protocol in the adversarial setting, and proves a suboptimal regret bound of $\otil \left(\rho^{-1/3}(KT)^{2/3}\right)$. They also prove a lower bound for a specific family of communication graphs and for an arbitrary gossip matrix supported on these graphs. Our lower bound construction in \pref{thm:lowerbound} is an adaptation of theirs.

\paragraph{Cooperative adversarial bandits.}
Our work is also closely related to the literature on \emph{cooperative adversarial bandits}~\citep{cesa2019delay, bar2019individual, ito2020delay}. This setting can be viewed as a special case of distributed bandits where the local loss functions are identical for all agents at each time step. However, a key distinction lies in the communication protocol. While we rely on a gossip protocol, \citet{cesa2019delay, bar2019individual} employ a message-sending protocol, in which agents broadcast their local information rather than mixing information through gossip. 
Because the local loss functions are identical across agents, each agent can achieve $\mathcal{O}(\sqrt{T})$ regret even without communication, whereas $\Omega(T)$ regret would be unavoidable in the distributed bandits problem. 

\section{Preliminaries}\label{sec:pre}

For a positive integer $K$, we use the notation $[K] \triangleq \{1,2,\dots, K\}$. For a finite set $S$, we denote with $\Delta(S)$ the probability simplex over $S$, defined as $\Delta(S)\triangleq\big\{\bp\in\mathbb{R}^{|S|} \,:\, \sum_{k=1}^{|S|} p(k)=1, p(k)\geq 0, \forall k\in [|S|] \big\}$. If $S = [K]$, then we write $\Delta(K)$ instead of $\Delta([K])$. We use boldface to denote vectors $\bu, \bv(i)$ and $u(j), v(i,j)$ to denote their components. We let $\bone$ to be the all-one vector in an appropriate dimension and, given a set $S$, we let $\bone_S$ to be the vector with entries $\bone_S(i) = 1$ if $i \in S$ and $\bone_S(i) = 0$ otherwise. Agents are the vertices $V = [N]$ of an undirected, connected communication graph $G=(V,E)$ with neighbors $\calN(i)\triangleq\{j\in V:(i,j)\in E\} \cup\{i\} $ for $i \in V$. 

\paragraph{Distributed $K$-armed bandits.} %
We consider an oblivious adversary, specifying vectors $\bell_t(i) \in [0,1]^K$ denoting the local loss of agent $i \in V$ at time $t$.
At each round $t \in [T]$, each agent $i\in V$ maintains a distribution $\bp_t(i)\in\Delta(K)$ and samples an action
$
A_t(i) \sim \bp_t(i) \in [K]
$.
After playing $A_t(i)$, each agent $i$ observes its realized local loss $\ell_t\big(i,A_t(i)\big)$.
Define the global average loss at time $t$ by
$\bellbar_t \triangleq \frac{1}{N}\sum_{j=1}^N \bell_t(j)$,
so that $
\ellbar_t(k)=\frac{1}{N}\sum_{j=1}^N \ell_t(j,k)$.
Then the expected global average loss of agent $i \in V$ at round $t$ is
$\mathbb{E}\left[\ellbar_t\bigl(A_t(i)\bigr)\right] = \inner{ \bp_t(i), \bellbar_t }.$

After observing local feedback, each agent $i$ may communicate with neighbors in $\calN(i)$---according to the gossip protocol---and perform the update $\bp_t(i) \to \bp_{t+1}(i)$. The performance of agent $i\in V$ is measured by the (pseudo) regret
\begin{equation}
\label{eq: regret}
    \Reg_T(i)
\triangleq
    \mathbb{E}\left[\sum_{t=1}^T \ellbar_t\big(A_t(i)\big) \right] - \min_{k\in [K]} \E \left[\sum_{t=1}^T \ellbar_t(k)\right]
=
    \max_{k \in [K]} \mathbb{E}\!\left[\sum_{t=1}^T \big\langle \bp_t(i) - \be_k, \bellbar_t \big\rangle \right],
\end{equation}
where $\{\be_1,\dots,\be_K\}$ is the canonical basis of $\R^K$ and the expectation is with respect to the agents' internal randomization. We also define $\Reg_T\triangleq \max_{i\in V}\Reg_T(i)$. Without loss of generality, we assume $T\geq3$ and $K\geq 2$.

\paragraph{Distributed linear bandits.} In the linear bandits case, agents have a common set of $K$ actions with feature vectors $\Omega=\{\ba_1,\dots,\ba_K\}\subset \mathbb{R}^d$.
The oblivious adversary specifies local loss coefficients $\btheta_t(i)\in\mathbb{R}^{d}$ such that
for any agent $i\in V$, action $\ba_k\in\Omega$, and round $t \in [T]$, the linear local loss satisfies $\ell_t(i,k)=\inner{\btheta_t(i),\ba_k} \in 
[-1,1]
$.
At each round $t \in [T]$, each agent $i\in V$ maintains a distribution
$\bp_t(i)\in\Delta(K)$ and samples an action index
$
A_t(i) \sim \bp_t(i) \in [K]
$.
After playing $A_t(i)$, agent $i$ observes its realized local loss $\ell_t\big(i,A_t(i)\big)$. 
The regret is still defined by~\pref{eq: regret}, with $\bellbar_t$ formed from the linear losses:
$
\ellbar_t(k) = \frac{1}{N}\sum_{i=1}^N \langle \btheta_t(i),\ba_k\rangle
= \big\langle \frac{1}{N}\sum_{i=1}^N \btheta_t(i),\ba_k \big\rangle .
$

\paragraph{Gossip protocol.}
We assume that agents exchange information \emph{only} through a gossip protocol \citep{yan2012distributed,hosseini2013online,yi2023doubly, wan2024nearly}. 
A \emph{gossip matrix} is any symmetric and doubly stochastic matrix $W \in \mathbb{R}^{N \times N}$ supported on the graph $G$, i.e., $W(i,j) > 0$ only if $(i,j) \in E$ or $i=j$ and $\sum_{i}W(i,j)=\sum_{j} W(i,j)=1$. We use $ 0 \le \sigma_2(W) < 1$ to denote the second largest singular value of $W$ (as $W$ is doubly stochastic, the largest singular value is $1$). The \emph{spectral gap} is $\rho(W)\triangleq 1-\sigma_2(W)$; when the dependence on $W$ is clear,
we write $\rho$. 
In each gossip step, every agent $i\in V$ sends a message to its neighbors $\calN(i)$ on the communication network $G$
and updates its local state by forming a weighted average using a common gossip matrix $W$ supported on $G$.
In the $K$-armed setting, messages are $K$-dimensional; in the linear bandit setting, agents exchange
$\order(d)$-dimensional vectors. As in prior work, $W$ is a parameter shared among agents and not individually learned.
To describe gossip more formally, suppose each agent $i\in V$ holds a vector $\bx(i)\in\R^K$ (or $\R^{\order(d)}$ in the linear case)
and the goal of the agents is to compute an approximation of the global average $\barbx \triangleq \frac{1}{N}\sum_{i\in V}\bx(i)$. Depending on the learning problem, $\bx(i)$ may represent, e.g., (local) loss gradients \citep{hosseini2013online,wan2024nearly} or loss estimates \citep{yi2023doubly}.
To do so, agents run \(B\) rounds of (accelerated) gossip \citep{liu2011accelerated} where the update is defined by
\begin{align}\label{eq:accel_gossip}
    \bx^{b+1}(i) = (1+\kappa) \sum_{j \in \calN(i)} W(i,j) \bx^{b}(j) - \kappa\bx^{b-1}(i) \qquad b\ge 0,
\end{align}
with initialization $\bx^{-1}(i)=\bx^{0}(i)=\bx(i)$ for all $i\in V$, and mixing coefficient $\kappa \ge 0$. When $\kappa = 0$, this reduces to standard gossip \citep{xiao2004fast}: $\bx^{b+1}(i) = \sum_{j \in \calN(i)} W(i,j) \bx^{b}(j)$.

\section{Distributed \texorpdfstring{$K$}{K}-armed Bandits: A Black-Box Reduction to Delayed Feedback}
\label{sec: mab}
\begin{algorithm}[t]
   \caption{Reduction from distributed bandits to bandits with delay}
   \label{alg: black-box}
   \nl  \textbf{Input:} Agent index $i$, time horizon $T$, block length $B$, mixing coefficient $\kappa$, gossip matrix $W$,  delayed bandit algorithm $\alg$.
    
    \nl \textbf{Initialization:} Set $\bz_1^{-1}(i) = \bz_1^{0}(i) = \mathbf{0}\in\R^K$, exploration parameter $\alpha=\frac{1}{T}$.
    
    \nl \For{$\tau=1,\ldots,T/B$}{
    \nl Receive $\bp'_{\tau}(i)$ from $\alg$ and compute $\bp_{\tau}(i) = \left(1 - \alpha\right) \bp'_{\tau}(i) +  \frac{\alpha}{K} \bm{1}$ \label{line: uniform}
    
    \nl Set $b=0$ 
      
    \nl   \For{$t = (\tau-1) B+1, \ldots, \tau B$}{
                \nl Play $A_{t}(i) \sim  \bp_{\tau}(i)$ and observe $\ell_t(i,A_t(i))$
                
                \nl Compute the estimate $\bellhat_{t}(i)$: for each $k \in [K]$, ${\displaystyle \ellhat_{t}(i,k) = \frac{\ell_{t}\left(i,k\right)}{p_{\tau}(i,k)} \ind \{A_{t}(i)=k\}}$   \label{line:ipw}
                
                \nl Update $\bz_{\tau}^{b+1}(i)=(1+\kappa) \sum_{j \in \calN(i)} W(i,j) \bz_{\tau}^{b}(j)-\kappa \bz_{\tau}^{b-1}(i)$ and set $b \leftarrow b+1$ \label{line: gossip_alg}
        }
        \nl \lIf{$\tau\geq 2$}{
            send $\left(\tau-1, \bz_{\tau}^B(i)\right)$ to $\alg$
        }      
        \nl Set $\bz_{\tau+1}^{-1}(i)= \bz_{\tau+1}^{0}(i) = \sum_{t=(\tau-1) B+1}^{\tau B} \bellhat_{t}(i)$
    } 
\end{algorithm}
In this section, we present a black-box reduction that transforms a distributed $K$-armed bandit problem into a $K$-armed bandit problem with delayed feedback. The reduction, stated in~\pref{alg: black-box}, can be combined with any bandit algorithm $\alg$ that enjoys a regret guarantee under delayed feedback.

A central difficulty in distributed adversarial bandits is that agents observe only
noisy estimates of their local losses. If agents attempt to (i) update their policies every round and simultaneously (ii) aggregate these noisy estimates via insufficient gossip steps, the resulting local approximations remain far from the true global average loss estimate, leading to large regret. Previous work by \citet{yi2023doubly} attempts to mitigate this by maintaining explicit $\Omega(T^{-1/3})$ uniform exploration to control variance, but this comes at the cost of a suboptimal $T^{2/3}$ regret rate. \pref{alg: black-box} resolves this issue by \emph{decoupling learning from communication}:

\begin{itemize}[parsep=0pt,itemsep=0pt,leftmargin=*]
    \item \textbf{Blocking (allowing time for mixing).} 
    Inspired by \citep{garber2020improved, wan2024nearly}, we introduce a delay of one block to ensure sufficient information mixing across the network. By keeping the action distribution fixed for a block of length $B$ (to be defined later) and executing $B$ rounds of accelerated gossip within each block, we ensure that the local approximation of global average loss estimate fed to $\alg$ at the end of the block is close enough to the true global average estimate.
    \item \textbf{Small uniform exploration (bounded variance).} We mix in a $\frac{1}{T}$-amount of uniform exploration to ensure $p_\tau(i,k) \ge 1/(KT)$. This guarantees that the importance-weighted estimates are uniformly bounded by $KT$, which allows the gossip protocol to control the consensus error uniformly across all blocks without degrading the regret rate.
\end{itemize}
More specifically, we partition the time horizon $T$ into $T/B$ blocks indexed by $\tau$, where $\mathcal{T}_{\tau} \triangleq \{(\tau-1)B+1,\ldots,\tau B\}$ denotes the time steps in block $\tau$. The algorithm operates iteratively over these blocks as follows: 
\begin{itemize}[topsep=2pt,parsep=0pt,itemsep=0pt,leftmargin=*]
    \item First, at the beginning of block $\tau$, each agent $i$ queries its local instance of the delayed bandit algorithm $\alg$ to obtain a base policy $\bp'_\tau(i)$. It then mixes in a $\frac{1}{T}$-amount of uniform exploration (see \pref{line: uniform}) to form the strategy $\bp_\tau(i)$, which is then played for all time steps $t \in \mathcal{T}_\tau$.
    \item Second, while playing these actions, the agent performs two parallel tasks: (i) it computes the sum $\sum_{t\in\mathcal{T}_\tau}\bellhat_t(i)$ of its local loss estimates for the \emph{current} block; and (ii) it runs $B$ rounds of accelerated gossip (see \pref{line: gossip_alg}) to mix the estimators collected in the \emph{previous} block, $\tau-1$.
    \item Finally, at the end of block $\tau$ (for $\tau\ge 2$), the result of this gossip process, denoted $\bz_\tau^B(i)$, is fed back to $\alg$. This vector serves as a local approximation of the global average loss estimate $\barbz_\tau$ for block $\tau-1$, where
\begin{equation}\label{eqn:barz}
    \barbz_\tau \triangleq \frac{1}{N}\sum_{i=1}^N \sum_{t\in\mathcal{T}_{\tau-1}}\bellhat_{t}(i), \qquad \tau \geq 2~.
\end{equation}
\end{itemize}
Since the information from block $\tau-1$ is processed during block $\tau$ and delivered only at its conclusion, the learner effectively operates with a constant delay of one block.

\begin{algorithm}[t]
   \caption{Reduction from delayed OLO to OLO}
   \label{alg: bold}
    \textbf{Input:} Agent index $i \in V$, time horizon $T$, base OLO algorithm $\mathcal{B}$
    
    \textbf{Initialization:} Instances $\mathcal{B}^{(0)}(i),\mathcal{B}^{(1)}(i)$ of $\mathcal{B}$ with initial distributions $\bq_1^{(0)}(i), \bq_1^{(1)}(i) \in \Delta(K)$.
    
   \For{$\tau =1,\cdots,T/B$}{
        Query $\mathcal{B}^{(\tau\!\!\mod 2)}(i)$ and receive $\bq_{\lfloor(\tau+1)/2\rfloor}^{(\tau\!\!\mod 2)}(i)$

        Send $\bp'_{\tau}(i) = \bq_{\lfloor(\tau+1)/2\rfloor}^{(\tau\!\!\mod 2)}(i)$ to \pref{alg: black-box}
     
        \If{$\tau\geq 2$}{
            Receive $\big(\tau-1, \bz_{\tau}^B(i)\big)$ from \pref{alg: black-box} and 
            send $\bz_{\tau}^B(i)$ to $\mathcal{B}^{(\tau - 1 \!\! \mod 2)}(i)$
        }
   }
\end{algorithm}

\noindent As our reduction is agnostic to the choice of the delayed bandit algorithm, we choose a general approach and instantiate $\alg$ as \pref{alg: bold} using the technique of \citet{joulani2013online}, which converts any non-delayed OLO algorithm $\mathcal{B}$ into an OLO algorithm robust to delay. While this framework typically requires maintaining $d_{\max}$ parallel instances to handle a maximum delay of $d_{\max}$, our reduction imposes a fixed delay of only one block. Consequently, \pref{alg: bold} needs to instantiate only two independent copies of $\mathcal{B}$ (alternating between even and odd blocks), so each instance effectively observes
non-delayed feedback. Specifically, we employ Follow-the-Regularized-Leader (FTRL) as the base learner $\calB$ (\pref{alg: ftrl}). While in this section we instantiate $\psi$ with the negative entropy regularizer, in the next section we derive adaptive regret bounds using different regularizers.

\begin{algorithm}[ht]
   \caption{Follow the Regularized Leader for MAB}
   \label{alg: ftrl}
    \textbf{Input:} regularizer $\psi_t$ for $t \in [T]$

    \textbf{Initialize:} $\bq_1 = \frac{1}{K}\bone$
    
   \For{$t = 1,\ldots,T$}{
       Output $\bq_t$ and receive $\bz_t$

       Update ${\displaystyle \bq_{t+1} = \argmin_{\bq\in \Delta(K)}\big\{\sum_{s\in[t]} \inner{\bz_s,\bq} +\psi_t(\bq)\big\} }$
   }
\end{algorithm}

It remains to specify the block length $B$ and the mixing coefficient $\kappa$. Our choice is driven by the convergence properties of the accelerated gossip protocol \citep[Proposition 11]{ye2023multi}, included as \pref{lem: acc_gossip_ye} for completeness. That result shows that, for an appropriate choice of $\kappa$, the approximation error to the global average loss estimate shrinks exponentially fast with the number of gossip rounds $B$, scaled by the norm of the initial gap to the global average loss estimate. Note that the $\frac{1}{T}$ uniform exploration ensures that the importance-weighted loss estimates have an $\ell_2$-norm bounded by $KT$. Consequently, we can drive an initial error of magnitude $\order(KT)$ down to $\mathrm{poly}\big((KT)^{-1}\big)$ with a number of rounds logarithmic in $KT$. $B$ and $\kappa$ are thus set as follows:
\begin{equation}
\label{eqn: block}
    B  = \left\lceil\frac{\ln\big((KT)^6\sqrt{14 N}\big)}{(1-1/\sqrt{2})\sqrt{1-\sigma_2(W)}}\right\rceil
    \qquad \text{and} \qquad
    \kappa = \frac{1}{1+\sqrt{1-\sigma_2^2(W)}}\;.
\end{equation}
With these choices, the following lemma shows that the gossip process leads to a local approximation of global average loss estimate with only a $\order((KT)^{-5})$ error. The full proof is deferred to \pref{app: cons}.
\begin{restatable}{rlemma}{apprerror}
\label{lem: cons}
If all agents $i \in V$ run \pref{alg: black-box} with gossip matrix $W$ and parameters $\kappa,B$ defined in~\pref{eqn: block}, then
\begin{equation}
    \max_{i \in V} \max_{\tau \in [T/B]} \left\|\bz_{\tau}^B(i)-\barbz_{\tau} \right\|_2 \leq \frac{2 B}{(KT)^5},
\end{equation}
where $\barbz_{\tau} \triangleq\frac{1}{N} \sum_{i=1}^N \sum_{t \in \mathcal{T}_{\tau-1}}\bellhat_t(i)$ is defined in \pref{eqn:barz} for all $\tau\in [T/B]$.
\end{restatable}
We now present our main result for distributed $K$-armed bandits. The following theorem shows that our algorithm achieves a regret bound of $\otil(\sqrt{(\rho^{-1/2}+K/N)T})$.
\begin{restatable}{rtheorem}{mainmab}
\label{thm:mainmab}
Let $\alg$ be \pref{alg: bold} run with base OLO algorithm $\mathcal{B}$ set to \pref{alg: ftrl} with $\psi_t(\bq) = \frac{1}{\eta}\sum_{k=1}^K q(k)\log q(k)$ for all $t\in[T]$ where $\eta=\sqrt{\frac{\log K}{2(B+\frac{3K}{N})T}}$. Assume all agents $i \in V$ run~\pref{alg: black-box} with bandit algorithm $\alg$, gossip matrix $W$, and parameters $\kappa$ and $B$ defined in~\pref{eqn: block}. Then the regret is bounded as
\begin{align*}
    \Reg_T &= 2\sqrt{2(\log K)\left(B+\frac{3K}{N}\right)T} + 10
    = \otil \left(\sqrt{\left(\rho^{-1/2} + \frac{K}{N}\right)T}\right)~.
\end{align*}
\end{restatable}
\pref{thm:mainmab} significantly improves over the state-of-the-art by \citet{yi2023doubly}, which obtained a suboptimal regret of $\otil\big(\rho^{-1/3}(KT)^{2/3}\big)$. Moreover, we show that our bound is minimax optimal up to logarithmic factors, as we provide a matching lower bound in \pref{thm:lowerbound}. To gain a better understanding of this result, note that our regret guarantee decomposes into two distinct terms: a \emph{communication cost} of $\otil\big(\sqrt{\rho^{-1/2}T}\big)$, which scales with the block length $B$ (and thus inversely with the spectral gap), and a \emph{bandit information cost} of $\otil\big(\sqrt{KT/N}\big)$, accounting for the acceleration provided by $N$ cooperating agents. The full proof is deferred to \pref{app:dmab}.


\paragraph{Proof Sketch of \pref{thm:mainmab}.}
The proof follows by combining two key ingredients. 
The first ingredient is the lemma showing that if, for all $i\in V$, the instance of $\calA$ run by $i$ enjoys a bounded regret with respect to the one-block-delayed estimates $\bz_\tau^B(i)$, then the regret of each agent with respect to the true global average loss is also bounded, up to a constant additive overhead. Formally, this reduction is stated as follows:

\begin{restatable}{rlemma}{delayreduction}\label{lem:delay_reduction}
Suppose that every agent $i\in V$ runs an instance of \pref{alg: black-box} using a base algorithm $\alg$ that guarantees
\begin{align}      
    \label{eq:delayed_guarantee}
    \max_{i\in V}\max_{k\in[K]}\E\left[\sum_{\tau=1}^{T/B}\left\langle \bp'_\tau(i)-\be_k, \bz_{\tau+1}^B(i)\right\rangle\right] \leq \Regdel \;,
\end{align}
for some $\Regdel \geq 0$, then we have $\Reg_T \le \Regdel + 4$~.
\end{restatable}
The full proof of \pref{lem:delay_reduction} is deferred to \pref{app: delay_reduction}, here we only sketch the main argument. A direct calculation shows that the overall regret is bounded by the regret of $\calA$ (the left-hand side of \pref{eq:delayed_guarantee}) plus two residual error terms:
(i) $(1-\alpha)\E[\sum_{\tau=1}^{T/B}\inner{\bp'_\tau(i)-\be_k, \barbz_{\tau+1}-\bz_{\tau+1}^{B}(i)}]$, representing the gossip approximation error; and
(ii) $\alpha\,\E[\sum_{\tau=1}^{T/B}\inner{\bq- \be_k,\barbz_{\tau+1}}]$, representing the cost of uniform exploration. To control the first term, we utilize our key concentration result \pref{lem: cons}, which guarantees that the gossip error $\|\barbz_{\tau+1}-\bz_{\tau+1}^{B}(i)\|_2$ is bounded by $\order(B/(TK)^5)$. Since $\|\bp'_\tau(i)-\be_k\|_2 \leq \sqrt{2}$, summing this error over all $T/B$ blocks yields a total contribution of $\order(1)$. To control the second term, we recall that $\barbz_{\tau+1}$ is an unbiased estimator of the cumulative global average loss over a block $\frac{1}{N}\sum_{i=1}^N\sum_{t\in\calT_{\tau}}\bell_t(i)$, which has magnitude $\order(B)$. Summing over $T/B$ blocks, and recalling that $\alpha = \order(T^{-1})$, the total expected cost becomes $\order(\alpha (T/B) B) = \order(1)$. Combining these bounds proves \pref{lem:delay_reduction}.

The second ingredient is to show that our subroutine $\calA$ defined in \pref{thm:mainmab} satisfies \pref{eq:delayed_guarantee} with $\Regdel=\otil\big(\sqrt{(\rho^{-1/2}+K/N)T}\big)$, formally stated in \pref{lem: delayftrlmab}.
\begin{restatable}{rlemma}{delayftrlmab}
\label{lem: delayftrlmab}
Under the same assumptions as~\pref{thm:mainmab},
\begin{align}\label{eqn:delayed_guarantee}
   \max_{i \in V}\max_{k \in [K]} \mathbb{E} \left[\sum_{\tau=1}^{T/B}\left\langle \bp'_{\tau}(i)-\be_k, \bz_{\tau+1}^B(i)\right\rangle\right]
\leq
    2\sqrt{2\left(B+\frac{3K}{N}\right)(\log K)T}+6~.
\end{align}
\end{restatable}
The proof of \pref{lem: delayftrlmab} is deferred to \pref{app: delayftrlmab} and we provide a sketch here. Fix an agent $i \in V$. Recall that \pref{alg: bold} handles the one-block delay by alternating two FTRL instances across odd and even blocks. This mechanism ensures that each instance of $\calB$ receives the feedback from its previous execution before it is queried again, so each instance effectively observes
non-delayed feedback. Consequently, we can decompose the analysis by considering the sequences generated by the two instances separately and then summing their regrets.

Focusing on the odd blocks, let $\{\bp'_{2\tau-1}(i)\}_{\tau=1}^{T/2B}$ denote the sequence of action distributions generated by the first instance, $\calB^{(1)}(i)$, based on the feedback $\bz_{2\tau}^B(i)$. For the analysis, we introduce a ghost sequence $\{\barbp_{2\tau-1}\}_{\tau=1}^{T/2B}$, generated by a hypothetical FTRL instance receiving the true global average loss estimates $\barbz_{2\tau}$. We then decompose the regret of $\calB^{(1)}(i)$ into three parts: (i) $\E [\sum_{\tau} \langle \bp'_{2\tau-1}(i) - \barbp_{2\tau-1}, \barbz_{2\tau} \rangle]$, capturing the discrepancy between the actual and ghost policies; (ii) $\E [\sum_{\tau} \langle \bp'_{2\tau-1}(i) - \be_k, \bz_{2\tau}^B(i) - \barbz_{2\tau} \rangle]$, which is the approximation error of the global average loss estimate; and (iii) $\E [\sum_{\tau} \langle \barbp_{2\tau-1} - \be_k, \barbz_{2\tau} \rangle]$, which is the regret of the ghost policy on the true global average losses.

We bound these terms as follows. For term (i), \pref{lem: cons} together with the stability properties of FTRL imply that $\|\barbp_{2\tau-1} - \bp'_{2\tau-1}(i)\|_1$ is bounded by $\order(B/(T^4K^5))$. Since uniform exploration ensures $\|\bellhat_t(i)\|_2 \leq KT$ (and thus $\|\barbz_{2\tau}\|_2 \le BKT$), summing these over the blocks yields an $\order(1)$ bound. Similarly, term (ii) is bounded by $\order(1)$ as the approximation error $\|\bz_{2\tau}^B(i) - \barbz_{2\tau}\|_2$ is negligible by \pref{lem: cons}. For term (iii), a standard FTRL analysis leads to a bound of $\order(\log K/\eta + \eta\sum_\tau \sum_k \barbp_{2\tau-1}(k)\mathbb{E}[\barbz_{2\tau}(k)^2])$. A key deviation from the classic analysis is that $\barbz_{2\tau}$ is constructed using actions sampled from the actual policy $\bp_{2\tau-1}(i)$, instead of the ghost policy $\barbp_{2\tau-1}$. Fortunately, we show in \pref{lem: stablity} (using \pref{lem: cons} again) that $\barp_{2\tau-1}(k) \leq 3p_{2\tau-1}(i,k)$ for all $i\in [N]$ and $k\in [K]$. This constant-factor bound on the importance weight mismatch results in an $\order(\log K/\eta + \eta(B+K/N)T)$ upper bound on term (iii). Picking $\eta$ as defined in \pref{thm:mainmab} and summing the contributions from both subsequences yields the stated bound.

\section{Adaptive Bounds for Distributed \texorpdfstring{$K$}{K}-armed Bandits}
\label{sec: ada_bound}
To demonstrate the versatility of our black-box reduction beyond worst-case minimax guarantees, we show in this section that our framework naturally yields \emph{adaptive} regret bounds for distributed bandits. The underlying principle is straightforward: if \pref{alg: black-box} is instantiated with a delayed bandit algorithm $\calA$ that achieves a specific adaptive regret guarantee under delayed feedback (implying that \pref{lem:delay_reduction} holds with $\Regdel$ being a data-dependent quantity), then this same bound carries over to the distributed setting via our reduction.
For illustration, we continue to use \pref{alg: bold} as our subroutine $\calA$, but we modify its base algorithm $\mathcal{B}$ to target specific guarantees. We explore two regimes: \emph{small-loss} guarantees in \pref{sec: small_loss} and \emph{best-of-both-worlds} guarantees in \pref{sec: bobw}.

To obtain small-loss guarantees, we follow \citet{van2022nonstochastic} who derived small-loss bounds for delayed adversarial bandits using FTRL with a hybrid negative entropy and log-barrier regularizer. To obtain \emph{best-of-both-worlds} guarantees, we use FTRL with a negative Tsallis entropy and log-barrier regularizer, as proposed by \citet{masoudian2022best}. These examples highlight a general principle: as long as an adaptive bound is achievable in the delayed-feedback model, our reduction automatically transfers it to the distributed bandit setting. The full proof is deferred to \pref{app: Omitted Proof Details for Small-loss Bounds}.

\subsection{Small-Loss Bound}\label{sec: small_loss}
We first formalize our small-loss guarantee. In this regime, our goal is to improve the regret bound by replacing the worst-case dependence on the horizon $T$ with the cumulative global average loss of the optimal action, $L^\star$, defined as
$
    L^\star\triangleq\min_{k\in [K]} \sum_{t=1}^T \ellbar_t(k) = \min_{k\in [K]}\sum_{t=1}^T\frac{1}{N}\sum_{i=1}^N \ell_t(i,k).
$
Since $\ell_t(i,k)\in[0,1]$, we have $L^\star \le T$. Therefore, replacing $T$ with $L^\star$ keeps the worst-case guarantee while offering better bounds in benign environments, where the optimal action incurs low cumulative loss. The following theorem establishes that instantiating \pref{alg: bold} with FTRL (\pref{alg: ftrl}) using a hybrid negative entropy and log-barrier regularizer achieves an optimal small-loss guarantee for distributed bandits.
\begin{restatable}{rtheorem}{tdelayftrlmabsl}
    \label{thm:main_mab_sl}
Let $\alg$ be \pref{alg: bold} run with base OLO algorithm $\mathcal{B}$ set to \pref{alg: ftrl} with $\psi_t(\bq) = \frac{1}{\eta}\sum_{i=1}^K q(i)\log q(i) - \frac{1}{\gamma} \sum_{k=1}^K \log q(k)$ for all $t\in[T]$ where $\eta = \min \big\{\frac{1}{4B},\sqrt{\frac{\log K}{BL^{\star}}} \big\}$ and $\gamma = \min \big\{\frac{N}{12},\sqrt{\frac{KN\log T}{L^{\star}}} \big\}$.
Assume all agents $i \in V$ run~\pref{alg: black-box} with bandit algorithm $\alg$, gossip matrix $W$, and $\kappa$ and $B$ defined in~\pref{eqn: block}. Then
\begin{align*}
\Reg_T &= \mathcal{O}\Big(\sqrt{B L^{\star} \log K}+ \sqrt{\frac{K L^{\star} \log T}{N}}+  B \log K  +  \frac{K \log T}{ N}\Big)
\\&= \otil\Big(\sqrt{\rho^{-1/2} L^{\star}}+ \sqrt{\frac{K L^{\star}}{N}}+  \rho^{-1/2}+  \frac{K}{N}\Big)~.
\end{align*}
\end{restatable}
To the best of our knowledge, \pref{thm:main_mab_sl} provides the first small-loss regret guarantee for distributed $K$-armed bandits. We remark that while \pref{thm:main_mab_sl} requires tuning the learning rates $\eta$ and $\gamma$ based on the unknown quantity $L^\star$, this restriction can be removed using a standard doubling trick, since the approximation error of the global average loss estimate is negligible---of order $\order(B/(TK)^5)$. We refer the reader to \citet{wei2018more,pmlr-v125-lee20a} for a detailed exposition on adapting to unknown $L^\star$.

\subsection{Best of Both Worlds}\label{sec: bobw}
In this section, we extend our results to provide simultaneous guarantees for both adversarial and stochastic environments. Specifically, for the stochastic setting, we assume that for each agent $i\in V$ and arm $k\in [K]$, the loss $\ell_t(i,k)\in[0,1]$ at each round $t\in[T]$ is drawn i.i.d.\ from a distribution with mean $\mu(i,k) = \mathbb{E}[\ell_t(i,k)]$. We define the global average mean loss as $\bmu \in [0,1]^K$, where $\mu(k) \triangleq \frac{1}{N}\sum_{i=1}^N\mu(i,k)$. Following previous works~\citep{zimmert2020optimal}, we assume there exists a unique optimal arm $k^\star \triangleq \argmin_{k\in[K]}\mu(k)$ and define the sub-optimality gaps $\bdelta \in [0,1]^K$ as $\delta(k) \triangleq \mu(k)-\mu(k^\star)$. In this environment, the regret defined in \pref{eq: regret} can be equivalently written as $\Reg_T = \max_{i\in V}\mathbb{E}\big[\sum_{t=1}^T \inner{\bp_t(i), \bdelta}\big]$.

The following theorem shows that by instantiating the base algorithm $\mathcal{B}$ with FTRL using a hybrid regularizer, specifically a combination of Shannon entropy and Tsallis entropy, we obtain bounds that are optimal up to logarithmic factors in the adversarial setting, while also enjoying logarithmic instance-dependent regret in the stochastic setting.
\begin{restatable}{rtheorem}{tdelayftrlmabbobw}
    \label{thm:main_mab_bbobw}
Let $\alg$ be \pref{alg: bold} run with base OLO algorithm $\mathcal{B}$ set to \pref{alg: ftrl} with $\psi_t(\bq) = \frac{1}{\eta_t}\sum_{i=1}^K q(i)\log q(i) - \frac{2}{\gamma_t} \sum_{k=1}^K \sqrt{q(k)}$ where $
    \eta_t = \min \Big\{\frac{1}{B},\sqrt{\frac{\log K}{tB^2}} \Big\}$
and $\gamma_t = \sqrt{\frac{N}{tB}}$
and. Assume all agents $i \in V$ run~\pref{alg: black-box} with bandit algorithm $\alg$, gossip matrix $W$, and parameters $\kappa$ and $B$ defined in~\pref{eqn: block}, Then, in the adversarial environment, the agents' regret is bounded as
\[
\Reg_T =\mathcal{O}\Big(\sqrt{B T \log K}+ \sqrt{\frac{K T}{N}}+  B \log K\Big) = \otil\Big(\sqrt{ \rho^{-1/2}T}+ \sqrt{\frac{K T}{N}}+  \rho^{-1/2}\Big)~.
\]
In the stochastic environment, the regret of the agents is bounded as
\[
\Reg_T = \order \Big( \frac{B}{N}  \sum_{k \neq k^{\star}} \frac{\log(T/B)}{\delta(k)} + \sum_{k \neq k^{\star}} \frac{B^2}{\delta(k)\log K}  + B^2\log K\Big)~.
\]
\end{restatable}
Since $B=\Theta(\rho^{-1/2}{\log(KNT)})$, the bound for the stochastic setting is dominated by the second term, which scales as $\order(\sum_{k\neq k^\star}\frac{\log^2(KTN)}{\rho\cdot \delta(k)\log K})$. To the best of our knowledge, no prior work establishes minimax rates for best-of-both-worlds guarantees in distributed bandits, not even for the pure stochastic setting with heterogeneous losses and gossip-only communication.
Specifically, in the pure stochastic setting, \citet{ZhangWCQYG25} achieve $\order (\sum_{k \neq k^{\star}} \frac{\log T}{N\delta(k)} + \frac{K\sqrt{N}}{\rho} )$; however, their protocol requires additional communication overhead beyond the standard gossip permitted in our framework. Similarly, \citet{martinez2019decentralized} obtain $\order (\sum_{k \neq k^{\star}} \frac{\log T}{N\delta(k)} + \rho^{-1/2}K\log N )$, but their result is restricted to the homogeneous setting where $\mu(i,k)=\mu(k)$ for all agents. Determining the tight minimax rate for distributed stochastic bandits with heterogeneous losses remains an interesting open direction.

\section{Distributed Adversarial Linear Bandits}
\label{sec: alb}

In this section, we extend our black-box reduction to the setting of adversarial linear bandits with $K$ actions in $\R^d$. To adapt the importance-weighting estimator (\pref{line:ipw} in \pref{alg: black-box}) from $K$-armed to linear bandits, a standard approach is to construct a local loss coefficient estimator $\bthetahat_t(i) = M_\tau^{-1}(i) \ba_{A_t(i)} \ell_t(i,A_t(i))$ at each round $t \in \calT_{\tau}$, where $M_{\tau}(i) = \sum_{j\in[K]} p_{\tau}(i,j) \ba_j \ba_j^\top$ denotes the correlation matrix of the sampling distribution.

\begin{algorithm}[ht]
   \caption{Black-box reduction for linear bandits}
   \label{alg: black-box-linear bandit}
    \textbf{Input:} Agent index $i$, time horizon $T$, block length $B$, mixing coefficients $\kappa$, gossip matrix $W$, delayed bandit algorithm $\alg$, action set $\Omega = \left\{\ba_1,\ba_2, \ldots,\ba_K\right\}$, volumetric spanner $\calS = \{\bb_1,\bb_2, \ldots,\bb_{|\calS|}\} \subseteq \Omega$, exploration parameters $\beta$.
    
    \textbf{Initialization:} Set $\bz_1^{-1}(i) = \bz_1^{0}(i) = \mathbf{0} \in \R^{|\calS|}$, exploration parameters $\alpha=\frac{1}{T}$.
    
    \For{$\tau=1,\cdots ,T/B$}{

    Receive $\bp'_{\tau}(i)$ from $\alg$ and compute $\bp_{\tau}(i)=\big(1-\alpha-\beta\big)\,\bp'_{\tau}(i) + \frac{\alpha}{K}\mathbf{1}+ \frac{\beta}{|\calS|}\bone_{\calS}$

    Compute $M_{\tau}(i) = \sum_{j\in[K]} p_{\tau}(i,j) \ba_j \ba_j^\top $

    Set $b=0$
      
    \For{$t \in \{(\tau-1) B+1, \ldots, \tau B\}$}{
                Play $A_t(i) \sim  \bp_{\tau}(i)$ and observe $\ell_t\big(i,A_t(i)\big) = \langle \ba_{A_t(i)},\btheta_t(i)\rangle$

                Compute $\bthetahat_{t}(i)=M_{\tau}(i)^{-1}\ba_{A_t(i)}\ell_t\big(i,A_t(i)\big)$

                Compute the estimator $\bellhat_{t}(i) \in \R^{|\calS|}$:
                For each $k \in \big[|\calS|\big]$, $\ellhat_{t}(i,k)=\langle \bb_k,\bthetahat_{t}(i)\rangle$
                
                Update ${\displaystyle \bz_{\tau}^{b+1}(i)=(1+\kappa) \sum_{j \in \calN(i)} W(i,j) \bz_{\tau}^{b}(j)-\kappa \bz_{\tau}^{b-1}(i) }$ and set $b \leftarrow b+1$
        }        
        \lIf{$\tau\geq 2$}{
            send $\left(\tau-1, \bz_{\tau}^B(i)\right)$ to $\alg$
        }   
        
        Set $\bz_{\tau+1}^{-1}(i)= \bz_{\tau+1}^{0}(i) = \sum_{t=(\tau-1) B+1}^{\tau B} \bellhat_{t}(i)$
    } 
\end{algorithm}

A naive strategy to estimate the global average loss would be to directly gossip $\bthetahat_t(i)$. However, unlike the MAB setting where the estimator's norm is bounded by $KT$ due to the uniform exploration, $\bthetahat_t(i)$ may have an arbitrarily large $\ell_2$-norm, preventing the direct application of \pref{lem: cons} to control the consensus error. Alternatively, one could explicitly control the scale of the loss estimator by forcing exploration on a certain basis, and then gossiping the reconstructed loss estimate $\widehat{\ell}_t(i,k)=\langle\ba_k,\bthetahat_t(i)\rangle$ for all $k \in [K]$. While this ensures boundedness of $\bellhat_t(i)$, it will incur an $\order(K)$ communication cost per agent per round, which is prohibitive in linear bandits, where $K$ is typically large. 
Therefore, in order to achieve $\order(d)$ communication cost, matching the communication efficiency of previous D-OCO protocols, we exploit the linear structure of the losses. Specifically, we restrict the gossip process to loss estimates defined only on a \emph{volumetric spanner} $\mathcal{S} \subseteq \Omega$. As formalized below, this spanner forms a good basis for $\Omega$, ensuring that any action can be represented as a linear combination with bounded coefficients.
\begin{definition}[Volumetric Spanner \citep{hazan2016volumetric}]
\label{def: vul_span}
    A subset $\calS = \{\bb_1,\bb_2,\ldots,\bb_{|\calS|}\} \subset \R^d$ is a volumetric spanner of $\Omega = \{\ba_1, \ba_2 ,\ldots,\ba_K\} \subset \R^d$ if $\calS \subseteq \Omega$ and for every $\ba_{k} \in \Omega$, $\ba_k$ can be expressed as $\ba_k = \sum_{j=1}^{|\calS|}\lambda^{(k)}(j)\,\bb_j$ for some coefficients $\blambda^{(k)} \in \R^{|\calS|}$ satisfying $\|\blambda^{(k)}\|_2 \leq 1$.
\end{definition}
This property allows us to reduce the communication complexity from $K$ to $\order(d)$. Specifically, agents need only gossip the estimated losses for the basis vectors in $\calS$. Using the linear relation in \pref{def: vul_span}, each agent can then locally reconstruct an approximation of the global average loss estimate for any action $\ba_k \in \Omega$. Crucially, the constraint $\|\blambda^{(k)}\|_2 \leq 1$ ensures that the consensus error does not grow during this reconstruction. The following proposition guarantees that such a spanner set has size at most $3d$ and can be found efficiently. In light of this, we assume that all agents pre-compute a common spanner $\calS$ during initialization.
\begin{proposition}[\citet{bhaskara2023tight}]
\label{prop: vol}
    Given a set $\Omega \subset \R^d$ of size $K$, there exists an efficient algorithm to compute a volumetric spanner $\calS$ of $\Omega$ of size $|\calS| \le 3d$ in time $\order(K d^3 \log d)$. 
\end{proposition}
With the volumetric spanner $\calS$ in hand, we present our distributed linear bandit algorithm in \pref{alg: black-box-linear bandit}. While the overall structure mirrors the MAB reduction (\pref{alg: black-box}), we introduce three specific modifications to leverage the linear geometry and control communication costs. First, in each block $\tau$, the sampling distribution $\bp_\tau(i)$ mixes the base policy not only with $\frac{1}{T}$-uniform exploration over the action set $\Omega$ but also with an explicit exploration term over the spanner $\calS$ (controlled by parameter $\beta$). This ensures the second moment matrix $M_\tau(i)$ is well-conditioned, making the estimated loss $\langle\ba_k, \bthetahat_t(i)\rangle$ for $ k \in [K]$ bounded by $|\calS|/\beta$ (as shown in \pref{lem:bounded_volumetric_l} in \pref{app: Omitted Proof Details for Linear Bandit}). Second, after playing an action $A_t(i)$ and observing the feedback, the agent constructs the standard unbiased local loss coefficient estimate $\bthetahat_t(i)$. Instead of communicating this potentially unbounded $\bthetahat_t(i)$ directly, the agent projects the loss onto the spanner basis and computes $\bellhat_t(i) \in \R^{|\calS|}$ where $\ellhat_t(i,k) = \langle\bb_k, \bthetahat_t(i)\rangle$. Finally, the agents run $B$ rounds of accelerated gossip on these $|\calS|$-dimensional vectors. When the delayed feedback $\bz_\tau^B(i)$ (an approximation of the global average loss estimate over the spanner) is finally passed to the base learner $\mathcal{B}$, we employ a modified base FTRL algorithm (\pref{alg: ftrl-linear} deferred to \pref{app: alg_base_linear}). Specifically, the learner first \emph{reconstructs} the estimated loss for every original action $\ba_k \in \Omega$ using the coefficients $\blambda^{(k)}$ from \pref{def: vul_span}, and then performs a standard FTRL update with entropy regularizer over $\Omega$~\citep{dani2007price,cesa2012combinatorial, bubeck2012towards}.
The following theorem shows the guarantee of \pref{alg: black-box-linear bandit} for the linear bandits setting.
\begin{restatable}{rtheorem}{mainalb}
\label{thm:mainalb}
    Let $\alg$ be \pref{alg: bold} run with base OLO algorithm $\mathcal{B}$ set to \pref{alg: ftrl-linear} with $\psi(\bq) = \frac{1}{\eta}\sum_{k=1}^K q(k)\log q(k)$ and $\eta = \min\Big\{\frac{1}{6Bd},\sqrt{\log K/(dTB+\frac{dT}{N}})\Big\}$.
    Assume all agents $i \in V$ run \pref{alg: black-box-linear bandit} with linear bandit algorithm $\alg$, action set $\Omega$, volumetric spanner $\calS \subseteq \Omega$, gossip matrix $W$, $\beta = 3Bd\eta$, and exploration parameters  $\kappa$ and $B$ defined in~\pref{eqn: block}. Then, the regret of each agent is bounded as
	\begin{align}
		\Reg_T = \order\left(\sqrt{(\log K)\Bigl(B+\tfrac{1}{N}\Bigr)dT}+dB\log K\right) \notag
        = \otil\left(\sqrt{\Bigl(\rho^{-1/2}+\tfrac{1}{N}\Bigr)dT}+d\rho^{-1/2}\right)~.
	\end{align}
\end{restatable}
The proof is deferred to \pref{app: Omitted Proof Details for Linear Bandit}. To the best of our knowledge, \pref{thm:mainalb} provides the first regret guarantee for adversarial distributed linear bandits using only gossip communication. Crucially, our algorithm achieves this rate with $\order(d)$ communication per round and a dependence on the action set size $K$ that is merely logarithmic. In \pref{sec: lowerbound}, we complement this result with a lower bound of $\Omega(\rho^{-1/4}\sqrt{T}+\sqrt{dT/N})$. Comparing the two results reveals a gap of $\sqrt{d\log KNT}$ in the communication-dependent term. This is likely an artifact of our block-based approach, and we leave closing this gap as an open problem.

\section{Lower Bound}
\label{sec: lowerbound}
In this section, we provide lower bounds for both distributed $K$-armed and distributed linear bandits. Recall that these bounds feature the sum of a communication cost term and a bandit cost term. For each setting, the proof shows that there exist problem instances forcing either term in the bound. Hence, the overall lower bound is asymptotic to the sum of these two terms. We start with the lower bound for $K$-armed bandits, which is a combination of the lower bound construction 
for the $K$-armed bandits with $NT$ observed losses~\citep{seldin2014prediction} and the one in~\citet{yi2023doubly}. The full proof is deferred to \pref{app: lower}.

\begin{restatable}{rtheorem}{lowerBound}\label{thm:lowerbound}
For any distributed $K$-armed bandit algorithm with $K \ge 2$, and for any $N,T$ large enough, there exists a communication graph $G$ with $N$ nodes, a gossip matrix $W$, and a sequence $\ell_1,\ldots,\ell_T$ of loss matrices $\ell_t(\cdot,\cdot) \in [0,1]^{N\times K}$ such that
\(
    \Reg_T = \Omega\Big(\sqrt{\rho^{-1/2}T\log K + \frac{KT}{N}}\Big).
\)
\end{restatable}
Next, we state the lower bound for the linear setting. The proof is a direct combination of the $K$-armed lower bound and a lower bound for cooperative linear bandits~\citep{ito2020delay}. The full proof is deferred to \pref{app: lower}.
\begin{restatable}{rtheorem}{LinearLowerBound}\label{thm:linearlowerbound}
For any distributed linear bandit algorithm, and for any $N,T$ large enough, there exists an action set $\Omega \subset \R^d$ with $|\Omega| = K \ge 2$, a communication graph $G$ with $N$ nodes, a gossip matrix $W$ supported on $G$, and local loss coefficients $\btheta_t(i) \in \R^d$ for $t \in T$ and $i \in V$ such that
$
    \Reg_T = \Omega\left(\sqrt{(\log K)\left(\rho^{-1/2} + \frac{d}{N}\right)T}\right).
$
\end{restatable}
%

\acks{
NCB acknowledges the financial support from the EU Horizon CL4-2021-HUMAN-01 research and innovation action under grant agreement 101070617, project ELSA (European Lighthouse on Secure and Safe AI) and the EU Horizon CL4-2022-HUMAN-02 research and innovation action under grant agreement 101120237, project ELIAS (European Lighthouse of AI for Sustainability).
}

\bibliography{ref}

\appendix

\crefalias{section}{appendix} 

\section{Auxiliary Results}
\label{app: aux}
\begin{rlemma}\label{lem:stab}
Let $\psi:\Delta(K)\to\mathbb{R}\cup\{+\infty\}$ be a proper, convex function that is
$\mu$-strongly convex with respect to $\|\cdot\|_2$ on $\Delta(K)$, i.e., for all $\bx,\by\in\Delta(K)$ and all $\bg\in\partial\psi(\bx)$,
\[
\psi(\by) \ge \psi(\bx) + \langle \bg, \by-\bx\rangle + \frac{\mu}{2}\|\by-\bx\|_2^2~.
\]
Let $\bxi : \mathbb{R}^K \to \Delta(K)$ be the argmin map
\[
    \bxi(\ba) = \argmin_{\bu\in\Delta(K)}\ \langle \ba,\bu\rangle+\psi(\bu)~.
\]
Then the argmin map is $1/\mu$-Lipschitz w.r.t.\ the $\ell_2$-norm,
\[
    \|\bxi(\ba)-\bxi(\boldb)\|_2 \le \frac{1}{\mu}\,\|\ba-\boldb\|_2~.
\]
\end{rlemma}
\begin{proof}
~Pick $\ba,\boldb \in \mathbb{R}^K$ and let $\bx = \bxi(\ba)$ and $\by = \bxi(\boldb)$. By first-order optimality of $\bx,\by$, there exist $\bg_x\in\partial\psi(\bx)$ and $\bg_y\in\partial\psi(\by)$ such that
\[
    \langle \ba+\bg_x, \bu-\bx \rangle \ge 0
\qquad\text{and}\qquad
    \langle \boldb + \bg_y, \bu-\by\rangle \ge 0 \qquad \forall \bu\in\Delta(K)~.
\]
Taking $\bu=\by$ in the first inequality and $\bu=\bx$ in the second inequality, and adding them together gives
\[
    \langle \ba - \boldb, \by - \bx \rangle \geq \langle \bg_x - \bg_y, \bx - \by \rangle~.
\]
Direct calculation shows that strong convexity implies strong monotonicity of the subgradient mapping:
\[
\langle \bg_x - \bg_y, \bx - \by \rangle \ge \mu\| \bx - \by \|_2^2.
\]
Hence $\langle \ba - \boldb, \by - \bx \rangle \ge \mu\| \bx - \by \|_2^2$. By Cauchy-Schwarz inequality, we know that
$\langle \ba - \boldb, \bx - \by \rangle \le \| \ba - \boldb \|_2\, \| \bx - \by \|_2$. Dividing $\|\bx-\by\|_2$ on both sides yields the claim.
\end{proof}

\subsection{Convergence properties for accelerated gossip}
\label{app: gossio_prop}
Given a set of vectors denoted as $\bx(1),\bx(2), \ldots, \bz(N) \in \mathbb{R}^K$ and let $\bx^0(i) = \bx^{-1}(i) = \bx(i)$ for all $i \in V$. To approximate the average $\overline{\bx}=\frac{1}{N} \sum_{i \in V} \bx(i)$, \cite{liu2011accelerated} considers the following accelerated gossip process:
\begin{equation}
\label{eqn:gossip_vector_1}
    \bx^{b+1}(i)=(1+\kappa) \sum_{j \in \calN(i)} W(i, j) \bx^b(j)-\kappa \bx^{b-1}(i)~,
\end{equation}
for all $b \geq 0$, where $\kappa$ is the mixing coefficient.
Let $X^b \triangleq \big[\bx^b(1)^{\top}, \bx^b(2)^{\top}, \ldots, \bx^b(N)^{\top}\big] \in \mathbb{R}^{N \times K}$ and $\barX \triangleq \big[\barbx^{\top},\barbx^{\top}, \ldots, \barbx^{\top}\big]\in \mathbb{R}^{N \times K}$. According to iteration in~\pref{eqn:gossip_vector_1}, it is not hard to verify that
\begin{equation*}
    X^{b+1}=(1+\kappa) W X^b-\kappa X^{b-1}~,
\end{equation*}
for all $b \geq 0$. \citet{ye2023multi} shows the following convergence property.

\begin{rlemma}[Proposition 1 in \citet{ye2023multi}] 
\label{lem: acc_gossip_ye}
For $B \geq 1$, the iterations of \pref{eqn:gossip_vector_1} with gossip matrix $W$ and parameters $\kappa$ defined in~\pref{eqn: block} ensure
\begin{equation*}
    \left\|X^B-\barX\right\|_F \leq \sqrt{14}\left(1-\left(1-\frac{1}{\sqrt{2}}\right) \sqrt{1-\sigma_2(W)}\right)^B\left\|X^0-\bar{X}\right\|_F.
\end{equation*}
\end{rlemma}

\section{Omitted Proof Details for Distributed \texorpdfstring{$K$}{K}-armed Bandits}
\label{app:dmab}


\subsection{Omitted proof details for \pref{lem: cons}}
\label{app: cons}

For completeness, we restate \pref{lem: cons} and provide its proof as follows. 
\apprerror*
\begin{proof}
Fix an epoch index $\tau$ and define $X_{\tau}^b \triangleq \big[\bz_{\tau}^b(1)^{\top}, \bz_{\tau}^b(2)^{\top}, \ldots, \bz_{\tau}^b(N)^{\top}\big] \in \mathbb{R}^{N \times K}$ for any integer $b \geq -1$. Let $\calT_{\tau}\triangleq\{(\tau-1)B+1, (\tau-1)B+2,\dots, \tau B\}$ be the round index in epoch $\tau$.
Recall that for each $i \in V$,
\begin{equation*}
    \bz_{\tau}^{-1}(i) = \bz_{\tau}^0(i) = \sum_{t \in \mathcal{T}_{\tau-1}}\bellhat_{t}(i)~,
\end{equation*}
which means that
\begin{equation}
\label{eq:X0}
    X_\tau^{-1} = X_\tau^{0} = \left[ \sum_{t \in \mathcal{T}_{\tau-1}}\bellhat_{t}(1)^\top, \ldots, \sum_{t \in \mathcal{T}_{\tau-1}}\bellhat_{t}(N)^\top\right]\;.
\end{equation}
According to the dynamic of $\bz_{\tau}^b(i)$ defined in \pref{alg: black-box}, it is straightforward to verify that for all $b\in\{0,1,\dots, B-1\}$,
\begin{equation}
\label{eq: matrix_gossip}
  X_\tau^{b+1}=(1+\kappa) W X_\tau^b-\kappa X_\tau^{b-1}.
\end{equation}
In order to bound $\|\barbz_\tau - \bz_\tau^B(i)\|_2$, we define $\barX_\tau \triangleq \left[\barbz_{\tau}^{\top},\ldots,\barbz_{\tau}^{\top}\right]$. When $\tau=1$, since $\barbz_1 = \bz_1^B(i)=\bzero$, the conclusion holds trivially. When $\tau\geq 2$, by definition of $X_\tau^B$ and $\overline{X}_\tau$, we know that $\left\|\bz_{\tau}^B(i)-\barbz_{\tau}\right\|_2 \le \left\|X_\tau^B-\barX_\tau\right\|_F$. Then, according to \citet[Proposition~1]{ye2023multi} (included in \pref{app: aux} as \pref{lem: acc_gossip_ye}), we have 
\begin{equation}
\label{ineq: goss_z}
    \left\|X_\tau^B-\barX_\tau\right\|_F \leq \sqrt{14}\left(1-c \sqrt{1-\sigma_2(W)}\right)^B\left\|X_\tau^0-\barX_\tau\right\|_F\;,
\end{equation}
where $c=1-1/\sqrt{2}$. Further plugging in the definition of $B$ in \pref{eqn: block} leads to the following
\begin{align}
\label{ineq: goss_alge}
   &\sqrt{14}\left(1-c \sqrt{1-\sigma_2(W)}\right)^B\left\|X_\tau^0-\barX_\tau\right\|_F \notag\\
&\leq
    \sqrt{14}\left(1-c \sqrt{1-\sigma_2(W)}\right)^{\frac{\ln(K^6T^6 \sqrt{14 N})}{c \sqrt{1-\sigma_2(W)}}}\left\|X_\tau^0-\barX_\tau\right\|_F
\nonumber
\\&\le
    \sqrt{14}\left(1-c \sqrt{1-\sigma_2(W)}\right)^{\frac{\ln (K^6T^6 \sqrt{14 N})}{-\ln(1-c \sqrt{1-\sigma_2(W)})}}\left\|X_\tau^0-\barX_\tau\right\|_F
\nonumber
\\& =\frac{\left\|X_\tau^0-\barX_\tau\right\|_F}{T^6 K^6 \sqrt{N}}\;,
\end{align}
where the first inequality is because of the definition of $B$ and the second inequality is due to $-\ln x \geq 1-x$ for any $x>0$.

Combining \pref{ineq: goss_z} and \pref{ineq: goss_alge}, for any $i \in V$ and $\tau \geq 2$, we have
\begin{align}
     \left\|\bz_{\tau}^B(i)-\barbz_{\tau}\right\|_2 \nonumber
&\le
    \frac{\left\|X_\tau^0-\barX_\tau\right\|_F}{T^6 K^6  \sqrt{N}} 
\\&\le
    \frac{\left\|X_\tau^0\right\|_F+\|\barX_\tau\|_F}{T^6 K^6 \sqrt{N}}  \tag{triangle inequality}
\\&\le
    \frac{2 \sqrt{\sum_{i=1}^N \sum_{k=1}^K \left(\sum_{t \in \mathcal{T}_{\tau-1}}\ellhat_t(i,k)\right)^2}}{T^6 K^6 \sqrt{N}}\;,
    \label{eq: gossip_inequality}
\end{align}
where the last inequality uses \pref{eq:X0} and the fact that $\|\barX_\tau\|_F \le \|X_\tau^0\|_F$ due to Jensen's inequality.

To further bound the right-hand side of \pref{eq: gossip_inequality}, since $\ell_t(i,k)\in[0,1]$ and $p_{\tau}(i,k)\geq 1/(KT)$ for all $t \in [T]$, $i \in V$, $k \in [K]$, and $\tau\geq 1$, and $\widehat \bell_{t}(i)$ has at most one nonzero coordinate, we know that that $\|\bell_t (i) \|_2\leq KT$. Therefore, by triangular inequality,
\begin{align}\label{eqn:norm_sum}
    \left\|\sum_{t \in \mathcal{T}_{\tau}}\bellhat_t(i)\right\|_2 \leq \sum_{t \in \mathcal{T}_{\tau}} \left\|\bellhat_t(i)\right\|_2  \leq |\mathcal{T}_{\tau}| KT \leq BKT,
\end{align}
where $B$ is an upper bound on the block length (i.e., $\left|\mathcal{T}_\tau\right| \leq B$). Substituting \pref{eqn:norm_sum} into \pref{eq: gossip_inequality} yields the final conclusion:
\begin{equation*}
    \left\|\bz_\tau^B(i)-\barbz_\tau\right\|_2 \leq \frac{2 B}{T^5 K^5}.
\end{equation*}
\end{proof}

\subsection{Omitted Proof Details for \pref{lem:delay_reduction}}
\label{app: delay_reduction}

In this subsection, we present the proof for \pref{lem:delay_reduction}, which formalizes the connection between the performance of the subroutine $\mathcal{A}$ used in \pref{alg: black-box} and our global regret objectives. The result shows that a regret guarantee for $\mathcal{A}$ on the given loss vectors implies a corresponding bound on the target regret, subject only to a small additive constant.

\delayreduction*

\begin{proof}
~Recall that \pref{alg: black-box} uses
$\bp_\tau(i)=(1-\alpha)\bp'_\tau(i)+\alpha\bq$ where $\bq=\frac{1}{K}\bm{1}$, $\alpha = \frac{1}{T}$, and $\bp'_\tau(i)$ is the distribution over $\Delta(K)$ received by \alg. Let \[k \triangleq \argmax_{k' \in [K]} \E\left[\sum_{\tau=1}^{T/B}\left\langle \bp_\tau(i)-\be_{k'},  \sum_{t \in \mathcal{T}_{\tau}} \bellbar_t \right\rangle\right]\;.\] Then, we decompose the regret as follows:
\begin{align}
    \Reg_T(i)
&=
    \E\left[\sum_{\tau=1}^{T/B}\left\langle \bp_\tau(i)-\be_k,  \sum_{t \in \mathcal{T}_{\tau}} \bellbar_t \right\rangle\right]
\notag
\\&=
    \E\left[\sum_{\tau=1}^{T/B}\left\langle \bp_\tau(i)-\be_k,\barbz_{\tau+1}\right\rangle\right] \tag{by definition of $\barbz_{\tau+1}$ in \pref{eqn:barz}}
\\&=
    \E\left[\sum_{\tau=1}^{T/B}\left\langle \left(1-\alpha\right)\bp'_\tau(i)+\alpha\bq-\be_k,\bz_{\tau+1}^{B}(i)+\barbz_{\tau+1}-\bz_{\tau+1}^{B}(i)\right\rangle\right]
\tag{by definition of $\bp_\tau(i)$}
\\&=
    (1-\alpha)\E\left[\sum_{\tau=1}^{T/B}\left\langle \bp'_\tau(i)-\be_k, \bz_{\tau+1}^{B}(i)\right\rangle\right]
\nonumber
\\&\quad
    + (1-\alpha)\E\left[\sum_{\tau=1}^{T/B}\left\langle \bp'_\tau(i)-\be_k, \barbz_{\tau+1}-\bz_{\tau+1}^{B}(i)\right\rangle\right]
\notag
\\&\quad
    + \alpha\,\E\left[\sum_{\tau=1}^{T/B}\left\langle\bq-\be_k,\barbz_{\tau+1}\right\rangle\right]
\notag
\\&\le
    \Regdel    \tag{using \pref{eq:delayed_guarantee}}
\notag
\\&\quad
    + (1-\alpha)\E\left[\sum_{\tau=1}^{T/B}\left\langle \bp'_\tau(i)-\be_k, \barbz_{\tau+1}-\bz_{\tau+1}^{B}(i)\right\rangle\right]
\label{eq:second-term}
\\&\quad
    + \alpha\,\E\left[\sum_{\tau=1}^{T/B}\left\langle \bq- \be_k,\barbz_{\tau+1}\right\rangle\right]~.
\label{eq:third-term}
\end{align}

We now bound \pref{eq:second-term} and \pref{eq:third-term} respectively.
\paragraph{Bounding~\pref{eq:second-term}.}
By Cauchy-Schwarz inequality, for each block $\tau\geq 1$,
\[
\left\langle \bp'_\tau(i)-\be_k,\barbz_{\tau+1}-\bz_{\tau+1}^{B}(i)\right\rangle
\le \big\|\bp'_\tau(i)-\be_k\big\|_2\cdot
\big\|\barbz_{\tau+1}-\bz_{\tau+1}^{B}(i)\big\|_2~.
\]
Since $\bp'_\tau(i)\in\Delta(K)$, $\|\bp'_\tau(i)-\be_k\|_2 \le \sqrt{2}$.
Moreover, by \pref{lem: cons}, we know that the second term is bounded as follows:
\[
\left\|\barbz_{\tau+1}-\bz_{\tau+1}^{B}(i)\right\|_2 \le \frac{2B}{T^5K^5}.
\]
Therefore,
\[
\left\langle \bp'_\tau(i)-\be_k, \barbz_{\tau+1}-\bz_{\tau+1}^{B}(i)\right\rangle
\le
    \frac{2\sqrt{2} B}{T^5K^5}~.
\]
Taking summation over $\tau=1,\dots,T/B$ yields
\begin{equation}\label{eq:disc_term_bound}
\left(1-\alpha\right)\E\left[\sum_{\tau=1}^{T/B}\left\langle \bp'_\tau(i)-\be_k,\barbz_{\tau+1}-\bz_{\tau+1}^{B}(i)\right\rangle\right]
\le \frac{T}{B}\, \frac{2\sqrt{2}B}{T^5K^5}
= \frac{2\sqrt{2}}{T^4K^5} \leq 3,
\end{equation}
where the last inequality uses the fact that $T\geq 1$ and $K\geq 1$.

\paragraph{Bounding~\pref{eq:third-term}.} Since $\ell_t(i,k)\in[0,1]$ for all $i\in[N]$, $t\in[T]$, and $k\in[K]$, we know that each coordinate of $\barbz_{\tau+1}$ is non-negative. Therefore, $
\left\langle \bq-\be_k,\barbz_{\tau+1}\right\rangle
\le \left\langle \bq,\barbz_{\tau+1}\right\rangle,
$
which leads to
\begin{align*}
\E\left[\sum_{\tau=1}^{T/B} \big\langle \bq-\be_k,\barbz_{\tau+1}\big\rangle \right]
&\le \sum_{\tau=1}^{T/B} \E\big[\left\langle \bq,\barbz_{\tau+1}\right\rangle\big] =\sum_{\tau=1}^{T/B} \frac{1}{K}\sum_{k=1}^K \E\big[\barz_{\tau+1}(k)\big].
\end{align*}
Since $\E\left[\ellhat_t(i,k)\right] = \ell_t(i,k)$ via a direct calculation, we know that
\[
\E \big[\barz_{\tau+1}(k)\big]
=\frac{1}{N}\sum_{i=1}^N\sum_{t\in\mathcal T_\tau}\E\Big[\ellhat_t(i,k)\Big]
\le \frac{1}{N}\sum_{i=1}^N\sum_{t\in\mathcal T_\tau} 1
= B
\]
and hence $\E\big[\langle q,\barbz_{\tau+1}\rangle\big]\le B$. This implies that
\begin{equation}\label{eq:smoothing_term_bound}
\alpha\,\E\left[\sum_{\tau=1}^{T/B}\left\langle \bq-\be_k, \barbz_{\tau+1}\right\rangle\right]
\le \alpha\,\sum_{\tau=1}^{T/B} B
=\alpha\,T
=1
\end{equation}
where we used $\alpha=\frac{1}{T}$ in the last equality. Finally, plugging in \pref{eq:disc_term_bound} and \pref{eq:smoothing_term_bound}
into the regret decomposition proves the statement.
\end{proof}

\subsection{Omitted Proof Details for~\pref{lem: delayftrlmab}}
\label{app: delayftrlmab}
In this section, we present the proof for \pref{lem: delayftrlmab}, which provides the concrete regret guarantee for \pref{alg: bold} with base OLO algorithm $\calB$ set to \pref{alg: ftrl}, using the entropy regularizer and $\eta$ defined in \pref{thm:mainmab}.

\delayftrlmab*
\begin{proof}
~Let $M\triangleq T/B$ denote the number of blocks and define the parity sets
\[
\mathcal{P}_0\triangleq\{\tau\in[M]: \tau\ \text{is even}\}~,
\qquad
\mathcal{P}_1\triangleq\{\tau\in[M]: \tau\ \text{is odd}\}.
\]
Let $m_r = |\mathcal{P}_r|$ for $r\in\{0,1\}$, i.e.,
\[
m_0=\left\lfloor\frac{M}{2}\right\rfloor~,
\qquad
m_1=\left\lceil\frac{M}{2}\right\rceil.
\]
Define the local-to-global index maps as follows:
\begin{equation}
\label{eqn: pi}
\pi_0(s)\triangleq2s-1,\quad s=1,\dots,m_0
\qquad\text{and}\qquad
\pi_1(s)\triangleq2s,\quad s=1,\dots,m_1
\end{equation}
so that $\mathcal{P}_r=\{\pi_r(s):s\in[m_r]\}$ for each $r\in\{0,1\}$.
%
Fix an agent $i \in V$ and $k\in[K]$, we decompose the regret as follows:
\begin{align*}
&\mathbb{E}\left[\sum_{\tau=1}^{M}\Big\langle \bp'_{\tau}(i)-\be_k, \bz_{\tau+1}^{B}(i)\Big\rangle\right]\notag\\
&=\underbrace{\mathbb{E}\left[\sum_{s=1}^{m_0}\left\langle \bp'_{\pi_0(s)}(i)-\be_k, \bz_{\pi_0(s)+1}^B(i)\right\rangle\right]}_{\triangleq \Reg_{\mathcal{P}_0}(i)}
+\underbrace{\mathbb{E} \left[\sum_{s=1}^{m_1}\left\langle \bp'_{\pi_1(s)}(i)-\be_k, \bz_{\pi_1(s)+1}^B(i)\right\rangle\right]}_{\triangleq \Reg_{\mathcal{P}_1}(i)}.
\end{align*}
We now analyze $\Reg_{\mathcal{P}_1}(i)$, the analysis for $\Reg_{\mathcal{P}_0}(i)$ is analogous. 

According to the update rule of $\bq_t$ in \pref{alg: ftrl}, we know that $\bp'_{\pi_1(s)}(i)$ is computed as follows:
\begin{align}\label{eqn:p_prime_ftrl}
\bp'_{\pi_1(s)}(i) &= \argmin_{\bq \in \Delta(K)} \sum_{s'=1}^{s-1} \left \langle \bz_{\pi_1(s')+1}^B(i), \bq \right \rangle + \frac{1}{\eta} \sum_{k=1}^K q(k) \log \left(q(k)\right) 
\end{align}
To analyze $\Reg_{\mathcal{P}_1}(i)$, we recall the definition of $\barbz_\tau$ defined in~\pref{eqn:barz}
\begin{align*}
\barbz_{\pi_1(s)+1} = \frac{1}{N}\sum_{i=1}^N \sum_{t\in\mathcal{T}_{\pi_1(s)}}\bellhat_{t}(i)
\end{align*}
and define $\barbp_{\pi_1(s)}$ as follows
\begin{align}\label{eqn:p_bar_ftrl}
\barbp_{\pi_1(s)} \triangleq \argmin_{\bq \in \Delta(K)} \sum_{s'=1}^{s-1} \left \langle \barbz_{\pi_1(s')+1},\bq \right \rangle + \frac{1}{\eta} \sum_{k=1}^K q(k) \log \left(q(k)\right),
\end{align}
which is the strategy output by FTRL with loss vectors $\{\barbz_{\pi_1(s')+1}\}_{s'\in[s-1]}$.
Now we decompose the regret into the following three terms and bound each term separately.
\begin{align}
    \Reg_{\mathcal{P}_1}(i)
&=
    \E \left[\sum_{s=1}^{m_1}\inner{\bp'_{\pi_1(s)}(i) - \be_k, \bz_{\pi_1(s)+1}^B(i)}\right]
\notag
\\&= \underbrace{\E \left[\sum_{s=1}^{m_1} \left\langle \bp'_{\pi_1(s)}(i) - \barbp_{\pi_1(s)}, \barbz_{\pi_1(s)+1} \right\rangle\right]}_{\clubsuit}
\notag
\\&\quad
    + \underbrace{\E \left[\sum_{s=1}^{m_1} \left\langle \bp'_{\pi_1(s)}(i) - \be_k,\bz_{\pi_1(s)+1}^B(i) - \barbz_{\pi_1(s)+1} \right\rangle\right]}_{\heartsuit}
\nonumber
\\&
    \quad + \underbrace{\E \left[\sum_{s=1}^{m_1} \left\langle \barbp_{\pi_1(s)} - \be_k, \barbz_{\pi_1(s)+1} \right\rangle\right]}_{\spadesuit}~.
\label{eq:regret_decomp_P1}
\end{align}

\paragraph{Bounding the term $\clubsuit$.}
By Cauchy-Schwarz inequality, for each $s\in[m_1]$, we have
\[
    \left\langle \bp'_{\pi_1(s)}(i) - \barbp_{\pi_1(s)},\barbz_{\pi_1(s)+1} \right\rangle
\le
    \| \bp'_{\pi_1(s)}(i) - \barbp_{\pi_1(s)} \|_2 \cdot  \|\barbz_{\pi_1(s)+1} \|_2~.
\]
Since $\bp'_{\pi_1(s)}(i)$ and $\barbp_{\pi_1(s)}$ follow the update rule of \pref{eqn:p_prime_ftrl} and \pref{eqn:p_bar_ftrl} and using $\psi(\bq)=\sum_{k=1}^Kq(k)\log(q(k))$ is $1$-strongly convex w.r.t. $\ell_2$-norm, \pref{lem: stablity} yields that
\[
\| \bp'_{\pi_1(s)}(i) - \barbp_{\pi_1(s)} \|_2 \le \frac{2\eta B}{T^4K^5}~.
\]
Moreover, since each $\bellhat_t(i)$ has at most one non-zero coordinate and $p_{\pi_1(s)}(i,j)\ge \frac{\alpha}{K} = 1/(KT)$ for all $j\in [K]$, we know that $\|\bellhat_t(i)\|_2 \le KT$, and hence
\[
    \|\barbz_{\pi_1(s)+1}\|_2
=
    \left\|\frac{1}{N}\sum_{t\in\mathcal T_{\pi_1(s)}}\sum_{i=1}^N \bellhat_t(i)\right\|_2
\le
    \frac{1}{N}\sum_{t\in\mathcal T_{\pi_1(s)}}\sum_{i=1}^N \|\bellhat_t(i)\|_2
\le
    BKT~.
\]
Therefore,
\begin{align}
    \clubsuit
\le
    \sum_{s=1}^{m_1}\E\Big[ \|\bp'_{\pi_1(s)}(i) - \barbp_{\pi_1(s)}\|_2 \cdot \|\barbz_{\pi_1(s)+1} \|_2 \Big]
\le
    \sum_{s=1}^{m_1} \frac{2\eta B}{T^4K^5}\cdot BKT
=
    \frac{2\eta\,m_1\,B^2}{T^3 K^4}
\le
    2
\label{ineq: clubsuit}
\end{align}
where the last inequality uses $m_1\le T/B$, $B\le T$ and $\eta\leq 1$.

\paragraph{Bounding the term $\heartsuit$.}
By Cauchy-Schwarz inequality, for each $s\in[m_1]$, we have
\[
    \left\langle \bp'_{\pi_1(s)}(i) - \be_k, \bz_{\pi_1(s)+1}^B(i) - \barbz_{\pi_1(s)+1}\right\rangle
\le
    \big\|\bp'_{\pi_1(s)}(i) - \be_k\big\|_2 \cdot \big\|\bz_{\pi_1(s)+1}^B(i) - \barbz_{\pi_1(s)+1} \big\|_2~.
\]
Since $\bp'_{\pi_1(s)}(i)\in\Delta(K)$ and $\be_k$ is a vertex of the simplex, we have $\|\bp'_{\pi_1(s)}(i) - \be_k \|_2 \le \sqrt{2}$.
Moreover, according to \pref{lem: cons}, we know that
\[
    \big\|\bz_{\pi_1(s)+1}^B(i) - \barbz_{\pi_1(s)+1}\big\|_2
\le
    \frac{2B}{T^5K^5}~.
\]
Therefore,
\begin{align}
    \heartsuit
=
    \E\left[\sum_{s=1}^{m_1} \left\langle \bp'_{\pi_1(s)}(i) - \be_k , \bz_{\pi_1(s)+1}^B(i) - \barbz_{\pi_1(s)+1} \right\rangle\right]
\le
    \sum_{s=1}^{m_1} \frac{2\sqrt{2} B}{T^5K^5}
\leq
    \frac{3m_1 B}{T^5 K^5}
\le 
    1,
\label{ineq: heartsuit}
\end{align}
where the last inequality uses $m_1\le T$ and $B\le T$.

\paragraph{Bounding the term $\spadesuit$.} 
Recall that the definitions of $\barbz_{\pi_1(s)+1}$ and $\ellhat_t(i,j)$ are as follows
\begin{equation*}
    \barbz_{\pi_1(s)+1} = \frac{1}{N}\sum_{t\in\mathcal T_{\pi_1(s)}} \sum_{i=1}^N \bellhat_t(i),~~~\ellhat_t(i,j) = \frac{\ell_t(i,j)}{p_{\pi_1(s)}(i,j)}\ind\{A_t(i)=j\}~. 
\end{equation*}

Since $\barbp_{\pi_1(s)}$ follows the update rule shown in \pref{eqn:p_bar_ftrl}, using the standard analysis of FTRL (e.g. Theorem 5.2 in \citet{hazan2016introduction}), we obtain
\begin{align}  
    \spadesuit
&=
    \mathbb{E}\left[\sum_{s=1}^{m_1}\left\langle \barbp_{\pi_1(s)} - \be_k, \barbz_{\pi_1(s)+1} \right\rangle\right]
\notag
\\&\le
    \frac{\log K}{\eta } + 2\eta\sum_{s=1}^{m_1} \sum_{j=1}^K \E \left[\barp_{\pi_1(s)}(j) \left(\frac{1}{N}\sum_{t\in \mathcal{T}_{\pi_1(s)}} \sum_{i=1}^N \ellhat_t(i,j)\right)^2 \right] \label{eq: exp3_analysis}
\end{align}

Let $\calF_{\pi_1(s)}$ be the filtration of all random events observed up to the beginning of block $\pi_1(s) $. Since $\barp_{\pi_1(s)}(j)$ is $\mathcal F_{\pi_1(s)}$-measurable, we know that
\begin{align}
    \E\Big[\barp_{\pi_1(s)}(j) \barz^2_{\pi_1(s)+1}(j)\Big]
&=
    \E\Big[ \E\big[\barp_{\pi_1(s)}(j) \barz^2_{\pi_1(s)+1}(j) \mid \mathcal F_{\pi_1(s)}\big]\Big]
\nonumber
\\&=
    \E\Big[ \barp^2_{\pi_1(s)}(j)\E\big[\barz_{\pi_1(s)+1}(j) \mid \mathcal F_{\pi_1(s)}\big]\Big]~.
\label{eq:tower_step}
\end{align}
Conditioned on $\mathcal F_{\pi_1(s)}$, for all $ t\in\mathcal{T}_{\pi_1(s)}$, direct calculation shows that
\[
    \E\left[\ellhat_t(i,j) \mid \mathcal F_{\pi_1(s)}\right] = \ell_t(i,j)
\qquad\text{and}\qquad
    \E\left[\ellhat_t(i,j)^2 \mid \mathcal F_{\pi_1(s)}\right]
=
    \frac{\ell^2_t(i,j)}{p_{\pi_1(s)}(i,j)}~.
\]
Since $A_t(i)$ are independently drawn across $t\in \mathcal{T}_{\pi_1(s)}$ and $i \in V$ conditioned on $\mathcal F_{\pi_1(s)}$, direct calculation shows that 
\begin{align}
    \E&\left[\left.\left(\sum_{t\in\mathcal T_{\pi_1(s)}}\sum_{i=1}^N \widehat\ell_t(i,j)\right)^2 \,\right\vert\, \mathcal F_{\pi_1(s)}\right]
    \notag\\
    &=\left(\E\left[\left.\sum_{t\in\mathcal T_{\pi_1(s)}}\sum_{i=1}^N \widehat\ell_t(i,j) \,\right\vert\, \mathcal F_{\pi_1(s)}\right]\right)^2 + \E\left[\left.\left(\sum_{t\in\mathcal T_{\pi_1(s)}}\sum_{i=1}^N (\widehat\ell_t(i,j)-\ell_t(i,j))\right)^2 \,\right\vert\, \mathcal F_{\pi_1(s)}\right]
    \notag\\ 
    &=
     \left(\sum_{t\in\mathcal T_{\pi_1(s)}}\sum_{i=1}^N \ell_t(i,j)\right)^2 + \sum_{t\in\mathcal T_{\pi_1(s)}}\sum_{i=1}^N \E\left[\left(\ellhat_t(i,j) - \ell_t(i,j)\right)^2 \mid \mathcal F_{\pi_1(s)} \right]
    \notag
\\&=
     \left(\sum_{t\in\mathcal T_{\pi_1(s)}}\sum_{i=1}^N \ell_t(i,j)\right)^2 + \sum_{t\in\mathcal T_{\pi_1(s)}}\sum_{i=1}^N \left(\frac{\ell^2_t(i,j)}{p_{\pi_1(s)}(i,j)}-\ell^2_t(i,j)\right) \label{eq: variance_loss}
\\&\le
    (NB)^2 + \sum_{t\in\mathcal T_{\pi_1(s)}}\sum_{i=1}^N \frac{1}{p_{\pi_1(s)}(i,j)}
    \notag
\end{align}
where the second equality uses $\E\left[\ellhat_t(i,j) \mid \mathcal F_{\pi_1(s)}\right] = \ell_t(i,j)$ and the conditional independence of $A_t(i)$ among $t\in \calT_{\pi_1(s)}$ and $i\in V$, and the last inequality holds because $\ell_t(i,j)\in[0,1]$. Hence,
\begin{align}
    \E\left[z_{\pi_1(s)+1}(j)^2 \mid \mathcal F_{\pi_1(s)}\right]
&=
    \frac{1}{N^2}\,\E\left[\left.\left(\sum_{t\in\mathcal T_{\pi_1(s)}}\sum_{i=1}^N \widehat\ell_t(i,j)\right)^2 \,\right\vert\, \mathcal F_{\pi_1(s)}\right]
\nonumber
\\&\le
    B^2 + \frac{1}{N^2}\sum_{t\in\mathcal T_{\pi_1(s)}}\sum_{i=1}^N\frac{1}{p_{\pi_1(s)}(i,j)}~.
\label{eq:cond_second_moment}
\end{align}
Combining \pref{eq:tower_step} and \pref{eq:cond_second_moment} yields
\begin{align}
    \spadesuit
&\le
    \mathbb{E}\left[\sum_{s=1}^{m_1}\left\langle \barbp_{\pi_1(s)} - \be_k, \barbz_{\pi_1(s)+1} \right\rangle\right]
\notag
\\&\le
    \frac{\log K}{\eta} + 2\eta\sum_{s=1}^{m_1} \sum_{j=1}^K \barp_{\pi_1(s)}(j)\left(B^2 + \frac{1}{N^2}\sum_{t\in\mathcal T_{\pi_1(s)}}\sum_{i=1}^N\frac{1}{p_{\pi_1(s)}(i,j)} \right)
\notag\\
&\le
    \frac{\log K}{\eta} + 2\eta \left(m_1 B^2 + \frac{3KBm_1}{N}\right)
\label{ineq: spadesuit}
\end{align}
where the last inequality uses $\sum_{j=1}^K \barp_s(j)=1$ and the fact that for all $s\in[m_1],\ i\in[N],\ j\in[K]$, with $\alpha=\frac{1}{T}$,
\begin{equation}
\label{eq: ratio}
    \frac{\barp_{\pi_1(s)}(j)}{p_{\pi_1(s)}(i,j)} = \frac{\barp_{\pi_1(s)}(j)}{(1-\alpha)p'_{\pi_1(s)}(i,j) + \alpha/K} 
\le
    3,
\end{equation}
where the equality is by the definition of $p_{\pi_1(s)}(i,j)$, and the inequality is due to \pref{lem: stablity}.

Substituting \pref{ineq: spadesuit}, \pref{ineq: clubsuit}, and \pref{ineq: heartsuit}
into the decomposition \pref{eq:regret_decomp_P1}, we obtain
\begin{align*}
    \Reg_{\mathcal{P}_1(i)} \leq \frac{\log K}{\eta }  + 2\eta \left(m_1B^2 + \frac{3KBm_1}{N}\right) + 3.
\end{align*}
The analysis for $\Reg_{\mathcal{P}_0}(i)$ is identical. Summing the bounds over the two parity
subsequences and using $m_1+m_2=M = T/B$ yields
\begin{align}
    \mathbb{E}\left[\sum_{\tau=1}^{M}\Big\langle \bp'_{\tau}(i) - \be_k, \bz_{\tau+1}^{B}(i)\Big\rangle\right]
&\le
    \frac{\log K}{\eta} + 2\eta\left(TB+\frac{3KT}{N}\right) + 6
\notag
\\&\le
    2\sqrt{2\log K\left(B+\frac{3K}{N}\right)T} + 6
\label{eq:final_bound_eta_opt}
\end{align}
where the second inequality follows by choosing
\[
\eta=\sqrt{\frac{\log K}{2(TB+\frac{3KT}{N})}}.
\]
\end{proof}



\subsection{Omitted Proof Details for~\pref{lem: stablity}}
\label{omitlem: stablity}
\begin{rlemma}
\label{lem: stablity}
Let $\alg$ be \pref{alg: bold} and $\mathcal{B}$ be an instance of \pref{alg: ftrl} with a regularizer $\psi$ that is $1$-strongly convex w.r.t.\ the $\ell_2$-norm. Suppose that each agent uses an instance of \pref{alg: black-box} with $\kappa$ and $B$ defined in \pref{eqn: block}. Define $\barbq_s^{(1)}$ and $\bq_s^{(1)}(i)$ for all $i\in[N]$ as follows
\begin{align*}
    \bq_s^{(1)}(i)
&=
    \argmin_{\bq \in \Delta(K)} \left\{\sum_{s'=1}^{s-1} \left \langle \bz_{\pi_1(s')+1}^B(i), \bq \right \rangle + \frac{1}{\eta} \psi(\bq)\right\}
\\
    \barbq_s^{(1)}
&=
    \argmin_{\bq \in \Delta(K)} \left\{\sum_{s'=1}^{s-1} \left \langle \barbz_{\pi_1(s')+1}, \bq \right \rangle + \frac{1}{\eta} \psi(\bq)\right\}
\end{align*}
where $\barbz_\tau$ is defined in~\pref{eqn:barz} and $\pi_1$ is defined in~\pref{eqn: pi}. Then we have
\begin{equation}
    \|\barbq_{s}^{(1)} - \bq_{s}^{(1)}(i)\|_2 \le \frac{2\eta B}{T^4K^5} \qquad \text{for all $i \in [N]$.}
\end{equation}
and 
\begin{equation*}
        \frac{\barq_{s}^{(1)}(k)}{(1-\alpha)q_{s}^{(1)}(i,k)+\alpha/K} \le 3 \quad \text{for all $i \in [N]$ and $k \in [K]$}
\end{equation*}
where $\alpha=\frac{1}{T}\leq \frac{1}{2}$.
\end{rlemma}
\begin{proof}
Since $\psi$ is $1$-strongly convex w.r.t. $\| \cdot \|_2$, we have
\begin{align*}
    \big\|\barbq_{s}^{(1)} - \bq_{s}^{(1)}(i)\big\|_2
&\le
    \eta \left\|\sum_{s'=1}^{s -1} \barbz_{\pi_1(s')+1} - \sum_{s'=1}^{s -1} \bz_{\pi_1(s')+1}^B(i)\right\|_2   \tag{\pref{lem:stab}}
\\&\le
    \eta\sum_{s'=1}^{s -1}\big\|\barbz_{\pi_1(s')+1} - \bz_{\pi_1(s')+1}^B(i)\big\|_2   \tag{triangle inequality}
\\&\le
    \frac{2\eta B}{T^4K^5}   \tag{\pref{lem: cons}}
\\&\le
    \frac{2}{KT}~.    \tag{$B \leq T$}
\end{align*}
To prove the second inequality, using $\|\cdot\|_{\infty} \leq\|\cdot\|_2$ we obtain  
\begin{align*}
    \Big| \barq_{s}^{(1)}(k) - {q}_{s}^{(1)}(i,k) \Big| \le \frac{2}{KT} \qquad \text{for all $k \in [K]$.}
\end{align*}
Now direct calculation shows that
\begin{align*}
     \frac{\barq_{s}^{(1)}(k)}{(1-\alpha){q}_{s}^{(1)}(i,k)+\alpha/K}
&\le
    \frac{{q}_{s}^{(1)}(i,k)+\frac{2}{KT}}{(1-\alpha){q}_{s}^{(1)}(i,k)+\alpha/K} 
\\&=
    1  + \frac{\alpha\,{q}_{s}^{(1)}(i,k) + \frac{2}{KT} - \alpha/K }{(1-\alpha){q}_{s}^{(1)}(i,k)+\alpha/K}
\\&\leq 1+\max\left\{\frac{\alpha}{1-\alpha}, \frac{2}{\alpha T}\right\}=    3
\end{align*}
where the last inequality uses the fact that $\alpha = 1/T\leq \frac{1}{2}$.
\end{proof}

\section{Omitted Details for Adaptive Bounds for Distributed \texorpdfstring{$K$}{K}-armed Bandits}
\subsection{Omitted Proof Details for \pref{thm:main_mab_sl}}
\label{app: Omitted Proof Details for Small-loss Bounds}

In this section, we present the proof for \pref{thm:main_mab_sl}, which provides the small-
loss guarantee. For completeness, we restate \pref{thm:main_mab_sl} and provide its proof as follows.

\tdelayftrlmabsl*

\begin{proof}
~We only show the steps that differ from the proof of \pref{thm:mainmab}, or more specifically \pref{lem:delay_reduction} and \pref{lem: delayftrlmab}. 
Note that the only algorithmic change compared to \pref{thm:mainmab} is to use \pref{alg: ftrl} with a different regularizer as the subroutine $\calB$. Since \pref{lem: cons} and \pref{lem:delay_reduction} is independent of the choice of $\calB$, the conclusion from \pref{lem: cons} and \pref{lem:delay_reduction} still hold. To analyze the three terms $\clubsuit$, $\heartsuit$, and $\spadesuit$ shown in the regret decomposition \pref{eq:regret_decomp_P1}, since $\psi$ is still $1$-strongly convex with respect to $\ell_2$-norm and $\bz_{\pi_1(s)+1}(i)$ still follows the accelerated gossip dynamic, we still have $\clubsuit\leq 2$ and $\heartsuit\leq 1$. Therefore, we have
\begin{align}
    \Reg_{\mathcal{P}_1}(i)
&\leq
    \underbrace{\E \left[\sum_{s=1}^{m_1} \left\langle \barbp_{\pi_1(s)} - \be_k, \barbz_{\pi_1(s)+1} \right\rangle\right]}_{\spadesuit} + 3~.
\label{eq:regret_decomp_P1_sl}
\end{align}
It remains to analyze $\spadesuit$ term $\E\left[\sum_{s=1}^{m_1}\left\langle \barbp_{\pi_1(s)} - \be_k, \barbz_{\pi_1(s)+1} \right\rangle\right]$.
Let $D_{\psi}(\cdot, \cdot)$ be the Bregman divergence based on $\psi$, and define $\wtbp_{\pi_1(s+1)}\triangleq \argmin _{\bq \in \R_{\geq 0}^d}\inner{\bq, \barbz_{\pi_1(s)+1}} + D_{\psi}(\bq, \barbp_{\pi_1(s)})$. 
Picking $\bp^{\star}=\left(1-\frac{1}{T}\right) \be_{k}+\frac{1}{T K} \mathbf{1}$ and using  Lemma~7.16 in \citet{orabona2025modernintroductiononlinelearning}, we obtain
\begin{align}
        \spadesuit
&=
    \mathbb{E}\left[\sum_{s=1}^{m_1}\left\langle \barbp_{\pi_1(s)} - \be_k, \barbz_{\pi_1(s)+1} \right\rangle\right] 
\notag
\\& = \mathbb{E}\left[\sum_{s=1}^{m_1}\left\langle \barbp_{\pi_1(s)} - \be_k + \bp^{\star}- \bp^{\star}, \barbz_{\pi_1(s)+1} \right\rangle\right]
\notag
\\ &\leq  \E \left[\sum_{s=1}^{m_1} \left\langle \bar \bp_{\pi_1(s)} - \bp^{\star}, \barbz_{\pi_1(s)+1}\right\rangle\right] + 2 \notag
   \\ & \leq  \psi(\bp^{\star}) -  \psi(\barbp_{\pi_1(1)})  +2
\notag
   \\& \quad + \sum_{s=1}^{m_1} \sum_{j=1}^K \left(\frac{\eta \gamma y^2_{\pi_1(s)}(j)}{\gamma y_{\pi_1(s)}(j) +\eta}\right)\E \left[ \barbz_{\pi_1(s)+1}^2 \right] \notag
\\ & \leq  \frac{\log K}{\eta} + \frac{K\log T}{\gamma } +2
\notag
   \\& \quad + \sum_{s=1}^{m_1} \sum_{j=1}^K \left(\frac{\eta \gamma y^2_{\pi_1(s)}(j)}{\gamma y_{\pi_1(s)}(j) +\eta}\right)\E \left[ \barbz_{\pi_1(s)+1}^2 \right]\label{eqn:decopsmallloss1}
\end{align}
where $\by_{\pi_1(s)}$ lies on the line segment between $\barbp_{\pi_1(s)}$ and $\wtbp_{\pi_1(s+1)}$, and the last inequality is because $\barbp_{\pi_1(1)}$ is the uniform distribution. According to the update rule of $\wtbp_{\pi_1(s+1)}$, we know $\wtbp_{\pi_1(s+1)}$ satisfies that for all $j\in[K]$,
\begin{equation*}
    \widetilde{p}_{\pi_1(s+1)}(j)+\log \widetilde{p}_{\pi_1(s+1)}(j)-\frac{\eta /\gamma }{\widetilde{p}_{\pi_1(s+1)}(j)}=\barp_{\pi_1(s)}(j)+\log \barp_{\pi_1(s)}(j)-\frac{\eta /\gamma}{\barp_{\pi_1(s)}(j)}-\eta \barz_{\pi_1(s)+1}(j),
\end{equation*}
Since $x \mapsto x+\log x- c / x$ with $c >0$ is strictly increasing on $(0, \infty), \barz_{\pi_1(s)+1}(j)\geq 0$ implies that for all $j\in[K]$,
\begin{equation*}
   \widetilde{p}_{\pi_1(s+1)}(j) \leq \barp_{\pi_1(s)}(j),
\end{equation*}
meaning that $\by_{\pi_1(s)}(j)\leq \barp_{\pi_1(s)}(j)$.
Plugging this inequality to \pref{eqn:decopsmallloss1} and using the fact that $\frac{x^2}{\gamma x+\eta}$ is increasing in $x$ for $x\in(0,1)$, we have that
\begin{align}
           \spadesuit & \leq  \frac{\log K}{\eta} + \frac{K\log T}{\gamma } +2  + \sum_{s=1}^{m_1} \sum_{j=1}^K \left(\frac{\eta \gamma \barp^2_{\pi_1(s)}(j)}{\gamma \barp_{\pi_1(s)}(j) +\eta}\right)\E \left[ \barbz_{\pi_1(s)+1}^2 \right] \notag
    \\& \leq  \frac{\log K}{\eta} + \frac{K\log T}{\gamma}   +2
   \\& \quad + \E \left
   [\frac{1}{N^2}\sum_{s=1}^{m_1} \sum_{j=1}^K \left(\frac{\eta \gamma \barp^2_{\pi_1(s)}(j)}{\gamma \barp_{\pi_1(s)}(j) +\eta}\right)\left(\left(\sum_{t\in\mathcal T_{\pi_1(s)}}\sum_{i=1}^N \ell_t(i,j)\right)^2 + \sum_{t\in\mathcal T_{\pi_1(s)}}\sum_{i=1}^N \left(\frac{\ell^2_t(i,j)}{p_{\pi_1(s)}(i,j)}\right)\right)\right] \tag{using \pref{eq: variance_loss}}
    \\& \leq  \frac{\log K}{\eta} + \frac{K\log T}{\gamma}   +2 + \E \left
   [\frac{B}{N}\sum_{s=1}^{m_1} \sum_{j=1}^K \eta  \barp_{\pi_1(s)}(j)\left(\sum_{t\in\mathcal T_{\pi_1(s)}}\sum_{i=1}^N \ell_t(i,j)\right)\right]  \notag
   \\&\quad +\E \left
   [\sum_{s=1}^{m_1} \sum_{j=1}^K \gamma \barp^2_{\pi_1(s)}(j)\left(\frac{1}{N^2}\sum_{t\in\mathcal T_{\pi_1(s)}}\sum_{i=1}^N\frac{\ell^2_t(i,j)}{p_{\pi_1(s)}(i,j)}\right)\right] \notag
    \\& \leq  \frac{\log K}{\eta} + \frac{K\log T}{\gamma}   + 2 + \E \left
   [\frac{B}{N}\sum_{s=1}^{m_1} \sum_{j=1}^K \eta  \barp_{\pi_1(s)}(j)\left(\sum_{t\in\mathcal T_{\pi_1(s)}}\sum_{i=1}^N\ell_t(i,j)\right)\right] \notag
   \\&\quad +3\E \left
   [\sum_{s=1}^{m_1} \sum_{j=1}^K \gamma \barp_{\pi_1(s)}(j)\left(\frac{1}{N^2}\sum_{t\in\mathcal T_{\pi_1(s)}}\sum_{i=1}^N\ell_t(i,j)\right) \right] \nonumber \\
   &\leq \frac{\log K}{\eta} + \frac{K\log T}{\gamma}   + 2 + \left(B\eta+\frac{3\gamma}{N}\right)\cdot \E \left
   [\sum_{s=1}^{m_1} \inner{\barp_{\pi_1(s)},\barbz_{\pi_1(s)+1}}\right], \label{eq: spade_sl_1}
\end{align}
where the fourth inequality holds by using \pref{eq: ratio} induced by \pref{lem: stablity}.
Rearranging \pref{eq: spade_sl_1} and using definition of $\barbz_{\pi_1(s)+1}$, we have
\begin{align*}       &(1-B\eta-\frac{3\gamma}{ N})\mathbb{E}\left[\sum_{s=1}^{m_1}\left\langle \barbp_{\pi_1(s)} -\be_k, \barbz_{\pi_1(s)+1} \right\rangle\right] \\& \qquad \qquad \qquad\leq \frac{\log K}{\eta} + \frac{K \log T}{\gamma} + 2 +  \left(B\eta + \frac{3\gamma}{N} \right) \sum_{s=1}^{m_1}\sum_{t\in \calT_{\pi_1(s)}}\ellbar_t(k).
\end{align*}
Picking $\eta \leq \frac{1}{4B}$ and $\gamma \leq \frac{N}{12}$ and rearranging the terms lead to
\begin{align*}       
\mathbb{E}\left[\sum_{s=1}^{m_1}\left\langle \barbp_{\pi_1(s)} -\be_k, \barbz_{\pi_1(s)+1} \right\rangle\right] \leq \frac{2\log K}{\eta} + \frac{2K \log T}{\gamma} + 4 +  \left(2B\eta + \frac{6\gamma}{N} \right) \sum_{s=1}^{m_1}\sum_{t\in \calT_{\pi_1(s)}}\ellbar_t(k).
\end{align*}

The analysis for $\Reg_{\mathcal{P}_0}(i)$ is identical. Summing the bounds over the two parity
subsequences, using $m_1+m_2=M = T/B$ and picking $\eta = \min \left\{\frac{1}{4B},\sqrt{\frac{\log K}{BL^{\star}}} \right\} $ and $\gamma = \min \left\{\frac{N}{12},\sqrt{\frac{KN\log T}{L^{\star}}} \right\} $, we obtain that
\begin{align*}
\mathbb{E}\left[\sum_{\tau=1}^{T/B}\left\langle \barbp_{\tau} -\be_k, \barbz_{\tau+1} \right\rangle\right] &\leq  \left(\frac{4\log K}{\eta} + \frac{4K \log T}{\gamma} + 8\right) +  2\left(B\eta + \frac{3\gamma}{N} \right) \sum_{t=1}^{T}\ellbar_t(k)
    \\& \leq \mathcal{O}\left(\sqrt{B L^{\star} \log K}+ \sqrt{\frac{K L^{\star} \log T}{N}}+  B \log K  +  \frac{K \log T}{ N}\right)~.
\end{align*}
Combining the bounds for $\heartsuit$ and $\clubsuit$, we have
\begin{align*}
    &\max_{i \in V}\max_{k \in [K]} \mathbb{E} \left[\sum_{\tau=1}^{T/B}\left\langle \bp'_{\tau}(i)-\be_k, \bz_{\tau+1}^B(i)\right\rangle\right] \\& \qquad \qquad \qquad\leq \mathcal{O}\left(\sqrt{B L^{\star} \log K}+ \sqrt{\frac{K L^{\star} \log T}{N}}+  B \log K  +  \frac{K \log T}{ N}\right).
\end{align*}
Combining \pref{lem:delay_reduction} with the above inequality finishes the proof.
\end{proof}

\subsection{Omitted Proof Details for \pref{thm:main_mab_bbobw}}

In this section, we present the proof for \pref{thm:main_mab_bbobw}, which provides the Best of Both Worlds guarantee. For completeness, we restate \pref{thm:main_mab_bbobw} and provide its proof as follows.

\tdelayftrlmabbobw*

\begin{proof}
We only show the steps that differ from the worst-case bound for distributed bandits.
Changing $\calB$ does not affect \pref{lem: cons} and \pref{lem:delay_reduction}.
Because $\psi_t(\bq)$ is $1$-strongly convex w.r.t. the $\ell_2$-norm, all the proof steps in \pref{lem: delayftrlmab} are the same , except $\spadesuit$ terms in $\Reg_{\mathcal{P}_0}$ and $\Reg_{\mathcal{P}_1}$.
We consider \begin{align*}
\barbp_{\pi_1(s)} \triangleq \argmin_{\bq \in \Delta(K)} \sum_{s'=1}^{s-1} \left \langle \wtbz_{\pi_1(s')+1},\bq \right \rangle + \psi_t(\bq)
\end{align*}
where 
\begin{equation}
    \widetilde{z}_{\pi_1(s')+1}(k) = \frac{1}{N}\sum_{i=1}^N \sum_{t\in\mathcal{T}_{\pi_1(s)}} \left(\widehat{\ell}_{t}(i,k) - \ell_{t}(i,k^{\star})\right)
\end{equation}
and 
\begin{equation*}
    \psi_t(\bq) = \frac{1}{\eta_{\pi_1(s)}}\sum_{k=1}^K q(k)\log q(k) - \frac{2}{\gamma_t} \sum_{k=1}^K \sqrt{q(k)}~.
\end{equation*}
Note that $\wtbz_{\pi_1(s')+1}(k) \geq -B$ for all $k \in [K]$ and  $s' \in [m_1]$. This centering does not change the argmin,
hence we also have $$\barbp_{\pi_1(s)+1}=\argmin_{\bq \in \Delta(K)} \sum_{s'=1}^{s-1} \left \langle \barbz_{\pi_1(s')+1},\bq \right \rangle + \psi_t(\bq).$$ We now analyze $\Reg_{\mathcal{P}_1}(i)$, and the analysis for $\Reg_{\mathcal{P}_0}(i)$ is analogous.  Let $D_{\psi}(\cdot, \cdot)$ be the Bregman divergence based on $\psi$, and define $\wtbp_{\pi_1(s+1)}\triangleq \argmin_{\bq \in \R_{\geq 0}^d}\inner{\bq, \wtbz_{\pi_1(s)+1}} + D_{\psi}\left(\bq, \barbp_{\pi_1(s)}\right)$. 
Using  Lemma~7.16 in \citet{orabona2025modernintroductiononlinelearning}, we obtain
\begin{align}
        \spadesuit
&=
    \mathbb{E}\left[\sum_{s=1}^{m_1}\left\langle \barbp_{\pi_1(s)} - \be_k, \wtbz_{\pi_1(s)+1} \right\rangle\right] 
\\&=
    \mathbb{E}\left[\sum_{s=1}^{m_1}\left\langle \barbp_{\pi_1(s)} - \be_k, \barbz_{\pi_1(s)+1} \right\rangle\right] 
\notag
   \\ & \leq \psi_{\pi_1(m_1)}(\be_k) - \psi_{\pi_1(1)}(\barp_{\pi_1(1)}) + \sum_{s=1}^{m_{1}-1}\left(\psi_{\pi_1(s)}\left(\barbp_{\pi_1(s+1)}\right)-\psi_{\pi_1(s+1)}\left(\barbp_{\pi_1(s+1)}\right)\right)
\notag
   \\& \quad + \sum_{s=1}^{m_1} \sum_{j=1}^K \left(\frac{2\eta_{\pi_1(s)} \gamma_t y^{3/2}_{\pi_1(s)}(j)}{\gamma_t \sqrt{y_{\pi_1(s)}(j)} +\eta_{\pi_1(s)}}\right)\E \left[ \wtbz_{\pi_1(s)+1}^2(j) \right] \label{eqn:decop_bobw}
\end{align}
where $\by_{\pi_1(s)}$ lies on the line segment between $\barbp_{\pi_1(s)}$ and $\wtbp_{\pi_1(s+1)}$. Moreover, $\by_{\pi_1(s)}$ has the following update expression:
\begin{equation*}
    \log y_{\pi_1(s)}(j)-\frac{\eta_{\pi_1(s)} /\gamma_{\pi_1(s)} }{\sqrt{y_{\pi_1(s)}(j)}}=\log \barp_{\pi_1(s)}(j)-\frac{\eta_{\pi_1(s)}/\gamma_{\pi_1(s)}}{\sqrt{\barp_{\pi_1(s)}(j)}}-\eta_{\pi_1(s)} \wtz_{\pi_1(s)+1}(j),  \quad \text{for all } j \in [K].
\end{equation*}
Rearranging the terms leads to
\begin{align*}
    \log y_{\pi_1(s)}(j)-\log \barp_{\pi_1(s)}(j) \leq \frac{\eta_{\pi_1(s)} /\gamma_{\pi_1(s)} }{\sqrt{y_{\pi_1(s)}(j)}} - \frac{\eta_{\pi_1(s)}/\gamma_{\pi_1(s)}}{\sqrt{\barp_{\pi_1(s)}(j)}}-\eta_{\pi_1(s)} \wtz_{\pi_1(s)+1}(j).
\end{align*}
We claim that this implies $y_{\pi_1(s)}(j)\leq 3\barp_{\pi_1(s)}(j)$. To see this, we consider the following two cases. If $y_{\pi_1(s)}(j)\leq \barp_{\pi_1(s)}(j)$, then we surely have $y_{\pi_1(s)}(j)\leq 3\barp_{\pi_1(s)}(j)$. Otherwise, we know that $\log y_{\pi_1(s)}(j)-\log \barp_{\pi_1(s)}(j) \leq -\eta_{\pi_1(s)}\wtz_{\pi_1(s)+1}(j)\leq \eta_{\pi_1(s)}B$, leading to the following
\begin{equation*}
    y_{\pi_1(s)}(j) \leq e^{\eta_{\pi_1(s)} B}\cdot\barp_{\pi_1(s)}(j)\leq 3\barp_{\pi_1(s)}(j)
\end{equation*}
where the last inequality holds because $\eta_{\pi_1(s)} \leq 1/B$. This concludes that $y_{\pi_1(s)}(j)\leq 3\barp_{\pi_1(s)}(j)$. Plugging this inequality to \pref{eqn:decop_bobw}, we have
\begin{align}
           \spadesuit & \leq  \psi_{\pi_1(m_1)}(\be_k) - \psi_{\pi_1(1)}(\barbp_{\pi_1(1)}) + \sum_{s=1}^{m_{1}-1}\left(\psi_{\pi_1(s)}\left(\barbp_{\pi_1(s+1)}\right)-\psi_{\pi_1(s+1)}\left(\barbp_{\pi_1(s+1)}\right)\right)
\notag
   \\& \quad + 6\sum_{s=1}^{m_1} \sum_{j=1}^K \left(\frac{\eta_{\pi_1(s)} \gamma_t \barp_{\pi_1(s)}^{3/2}(j)}{\gamma_t \sqrt{\barp_{\pi_1(s)}(j)} +\eta_{\pi_1(s)}}\right)\E \left[ \wtz_{\pi_1(s)+1}^2(j). \right]
   \label{eq:bobw_decomp}
\end{align}
Next, we derive our regret guarantees for the adversarial and the stochastic environment separately.

\paragraph{Adversarial environment guarantees.}

We first analyze the first summation of \pref{eq:bobw_decomp}.
\begin{align}
&\psi_{\pi_1(m_1)}(\be_k) - \psi_{\pi_1(1)}(\barbp_{\pi_1(1)}) +\sum_{s=1}^{m_1-1}\left(\psi_{\pi_1(s)}\left(\barbp_{\pi_1(s+1)}\right)-\psi_{\pi_1(s+1)}\left(\barbp_{\pi_1(s+1)}\right)\right) \notag
\\&\leq \frac{\log K}{\eta_{\pi_1(1)}} + \frac{2\sqrt{K}}{\gamma_{\pi_1(1)}} - \frac{2}{\gamma_{\pi_1(m_1)}}+\sum_{s=1}^{m_1-1}\left(\psi_{\pi_1(s)}\left(\barbp_{\pi_1(s+1)}\right)-\psi_{\pi_1(s+1)}\left(\barbp_{\pi_1(s+1)}\right)\right) \notag
\\&\leq  \frac{\log K}{\eta_{\pi_1(1)}} + \frac{2\sqrt{K}-2}{\gamma_{\pi_1(1)}} \notag
\\&\quad+ 2\sum_{s=1}^{m_1-1}\left(\frac{1}{\gamma_{\pi_1(s+1)}}-\frac{1}{\gamma_{\pi_1(s)}}\right)
\left(\sum_{j=1}^{K}\sqrt{\barp_{\pi_1(s+1)}
(j)}-1\right) \notag
\\&\quad- \sum_{s=2}^{m_1}  \left(\frac{1}{\eta_{\pi_1(s)}} - \frac{1}{\eta_{\pi_1(s-1)}}\right) \sum_{j=1}^K\barp_{\pi_1(s)}(j)\log \barp_{\pi_1(s)}(j). \label{eq:bobw_adv_bias}
\end{align}
We first consider the first summation in \pref{eq:bobw_adv_bias}. Recall the definition of $\gamma_t$ and $\pi_1(s)=2s$ for all $s \in [m_1]$ , we have 
\begin{align*}
\frac{1}{\gamma_{\pi_1(s+1)}}-\frac{1}{\gamma_{\pi_1(s)}} \le \sqrt{\frac{B}{N}}\cdot
\frac{\pi_1(s+1)-\pi_1(s)}{\sqrt{\pi_1(s+1)}} \le \sqrt{\frac{B}{N}}\cdot
\frac{2}{\sqrt{\pi_1(s+1)}}~.
\end{align*}
and
\begin{align*}
\frac{1}{\eta_{\pi_1(s+1)}}-\frac{1}{\eta_{\pi_1(s)}} \le \frac{B}{\sqrt{\log K}}\cdot
\frac{\pi_1(s+1)-\pi_1(s)}{\sqrt{\pi_1(s+1)}} \le \frac{B}{\sqrt{\log K}}\cdot
\frac{2}{\sqrt{\pi_1(s+1)}}~.
\end{align*}

Hence, we have
\begin{align}
&2\sum_{s=1}^{m_1-1}\left(\frac{1}{\gamma_{\pi_1(s+1)}}-\frac{1}{\gamma_{\pi_1(s)}}\right)
\left(\sum_{k=1}^{K}\sqrt{\barp_{\pi_1(s+1)}(k)}-1\right)
\notag
\\&\le \frac{4\sqrt{B}}{\sqrt{N}}
\sum_{s=1}^{m_1-1}\frac{1}{\sqrt{\pi_1(s+1)}}
\left(\sum_{k=1}^{K}\sqrt{\barp_{\pi_1(s+1)}(k)}-1\right) \notag\\
&\le \frac{4\sqrt{B}}{\sqrt{N}}
\sum_{s=1}^{m_1}\frac{1}{\sqrt{\pi_1(s)}}
\left(\sum_{k=1}^{K}\sqrt{\barp_{\pi_1(s)}(k)}-1\right) \label{eq:bobw_adv_bias_1}\\
&\le \order\left(\sqrt{\frac{TK}{N}}\right).\label{eq:bobw_adv_bias_order_1}
\end{align}
Regarding the second summation in \pref{eq:bobw_adv_bias}, we have
\begin{align}
    &-\sum_{s=2}^{m_1}  \left(\frac{1}{\eta_{\pi_1(s)}} - \frac{1}{\eta_{\pi_1(s-1)}}\right) \sum_{j=1}^K\barp_{\pi_1(s)}(j)\log \barp_{\pi_1(s)}(j)  \notag
    \\&\leq  - \sum_{s=1}^{m_1-1} \frac{B}{\sqrt{\log K}} \cdot \frac{2}{\sqrt{\pi_1(s)}}\sum_{j=1}^K\barp_{\pi_1(s)}(j)\log \barp_{\pi_1(s)}(j) \label{eq:bobw_adv_bias_2} \\
    &\leq \order\left(\sqrt{BT\log K}\right),\label{eq:bobw_adv_bias_order_2}
\end{align}
where the last inequality is because $-\sum_{j=1}^K\barp_{\pi_1(s)}(j)\log \barp_{\pi_1(s)}(j)\leq\log K$.
Plugging \pref{eq:bobw_adv_bias_order_1} and \pref{eq:bobw_adv_bias_order_2} to \pref{eq:bobw_adv_bias}, we have
\begin{align}
    &\psi_{\pi_1(m_1)}(\be_k) - \psi_{\pi_1(1)}(\barbp_{\pi_1(1)}) +\sum_{s=1}^{m_1-1}\left(\psi_{\pi_1(s)}\left(\barbp_{\pi_1(s+1)}\right)-\psi_{\pi_1(s+1)}\left(\barbp_{\pi_1(s+1)}\right)\right) \nonumber\\
    &\leq \frac{\log K}{\eta_{\pi_1(1)}} + \frac{2\sqrt{K}-2}{\gamma_{\pi_1(1)}} +\order\left(\sqrt{BT\log K}+\sqrt{\frac{TK}{N}}\right)\nonumber\\
    &\leq \order\left(B\log K + \sqrt{\frac{BK}{N}} + \sqrt{\frac{TK}{N}} + \sqrt{BT\log K}\right) \nonumber \\
    &\leq \order\left(B\log K + \sqrt{\frac{TK}{N}} + \sqrt{BT\log K}\right),\label{eq:bobw_adv_bias_order_3}
\end{align}
where the last inequality uses $B\leq T$.

Next, we analyze the second summation of \pref{eq:bobw_decomp}. Let $\calF_{\pi_1(s)}$ be the filtration of all random events observed up to the beginning of block $\pi_1(s) $. Since $\barp_{\pi_1(s)}(j)$ is $\mathcal F_{\pi_1(s)}$-measurable, we know that
\begin{align}
    \E \left [\wtz^2_{\pi_1(s)+1}(k) \middle| \calF_{\pi_1(s)}  \right]&= \E\left [\left(\frac{1}{N}\sum_{i=1}^N \sum_{t\in\mathcal{T}_{\pi_1(s)}} (\ellhat_{t}(i,k) - \ell_{t}(i,k^{\star}))\right)^2\middle| \calF_{\pi_1(s)}\right]\notag \\
    &= \frac{1}{N^2}\E\left[\left(\sum_{i=1}^N\sum_{t\in\calT_{\pi_1(s)}}(\ellhat_t(i,k)-\ell_t(i,k))\right)^2\middle| \calF_{\pi_1(s)}\right]\nonumber\\
    &\qquad + \frac{1}{N^2}\left(\E\left [\sum_{i=1}^N \sum_{t\in\mathcal{T}_{\pi_1(s)}} (\ellhat_{t}(i,k) - \ell_{t}(i,k^{\star}))\middle| \calF_{\pi_1(s)}\right]\right)^2\notag \nonumber \tag{using $\E[x^2]=\E[(x-\E[x])^2]+\E[x]^2$}\\
    &= \frac{1}{N^2}\sum_{i=1}^N\sum_{t\in\calT_{\pi_1(s)}}\E\left[(\ellhat_t(i,k)-\ell_t(i,k))^2\middle| \calF_{\pi_1(s)}\right] \notag \\
    &\qquad + \frac{1}{N^2} \left(\sum_{i=1}^N \sum_{t \in \mathcal{T}_{\pi_1(s)}}\left(\ell_t(i, k)-\ell_t\left(i, k^{\star}\right)\right)\right)^2\tag{since $A_t(i)$ are drawn independently over $i\in[N]$}
    \\&= \frac{1}{N^2} \sum_{i=1}^N \sum_{t \in \mathcal{T}_{\pi_1(s)}} \ell^2_t(i, k)\left(\frac{1}{p_{\pi_1(s)}(i, k)}-1\right)\notag
    \\&\quad+ \frac{1}{N^2} \left(\sum_{i=1}^N \sum_{t \in \mathcal{T}_{\pi_1(s)}}\left(\ell_t(i, k)-\ell_t\left(i, k^{\star}\right)\right)\right)^2\notag
    \\&= \frac{1}{N^2} \sum_{i=1}^N \sum_{t \in \mathcal{T}_{\pi_1(s)}} \left(\frac{\ell^2_t(i, k)}{p_{\pi_1(s)}(i, k)}-2\ell_t(i, k)\ell_t(i, k^{\star})+ \ell^2_t(i, k^{\star}) \right) \notag
    \\&\quad+ \frac{1}{N^2}\sum_{\substack{(i,t)\neq (j,t')\\ t,t'\in\mathcal{T}_{\pi_1(s)} \\ i, j \in [N]}}
\left(\ell_t(i, k)-\ell_t\left(i, k^{\star}\right)\right)
\left(\ell_{t'}(j, k)-\ell_{t'}\left(j, k^{\star}\right)\right)~.\label{eq:bobw_z}
\end{align}
Plugging \pref{eq:bobw_z} and \pref{eq:bobw_adv_bias_order_3} in \pref{eq:bobw_decomp}, we obtain
\begin{align}
    \spadesuit & \leq  \order\left(B\log K + \sqrt{\frac{TK}{N}} + \sqrt{BT\log K}\right)  \notag
    \\&+ \frac{6}{N^2} \sum_{s=1}^{m_1} \gamma_{\pi_1(s)} \sum_{k=1}^K \barp_{\pi_1(s)}^{3/2}(k)  \sum_{i=1}^N \sum_{t \in \mathcal{T}_{\pi_1(s)}} \left(\frac{\ell^2_t(i, k)}{p_{\pi_1(s)}(i, k)}-2\ell_t(i, k)\ell_t(i, k^{\star})+ \ell^2_t(i, k^{\star}) \right) \notag
    \\& + \E \left [\frac{6}{N^2} \sum_{s=1}^{m_1} \eta_{\pi_1(s)}\sum_{k=1}^K \barp_{\pi_1(s)}^{3/2}(k) \sum_{\substack{(i,t)\neq (j,t')\\ t,t'\in\mathcal{T}_{\pi_1(s)} \\ i, j \in [N]}}
\left(\ell_t(i, k)-\ell_t\left(i, k^{\star}\right)\right)
\left(\ell_{t'}(j, k)-\ell_{t'}\left(j, k^{\star}\right)\right) \right]~. \label{eq:bobw_spade_2}
\end{align}
Regarding the first summation in \pref{eq:bobw_spade_2}, we have
\begin{align}
   & \frac{6}{N^2} \sum_{s=1}^{m_1} \gamma_{\pi_1(s)} \sum_{k=1}^K \barp_{\pi_1(s)}^{3/2}(k)  \sum_{i=1}^N \sum_{t\in \mathcal{T}_{\pi_1(s)}}  \left(\frac{\ell^2_{t}(i, k)}{p_{\pi_1(s)}(i, k)}-2\ell_{t}(i, k)\ell_{t}(i, k^{\star})+ \ell^2_{t}(i, k^{\star}) \right) \notag
   \\& = \frac{6}{N^2} \sum_{s=1}^{m_1} \gamma_{\pi_1(s)} \sum_{k=1}^K \sum_{i=1}^N \sum_{t\in \mathcal{T}_{\pi_1(s)}}  \left( \barp_{\pi_1(s)}^{3/2}(k)\frac{\ell^2_{t}(i, k)}{p_{\pi_1(s)}(i, k)} -  \barp_{\pi_1(s)}^{3/2}(k)\ell_t(i, k)\ell_t(i, k^{\star})\right)  \notag
   \\& \qquad  + \frac{6}{N^2} \sum_{s=1}^{m_1} \gamma_{\pi_1(s)} \sum_{k=1}^K \sum_{i=1}^N \sum_{t\in \mathcal{T}_{\pi_1(s)}} \barp_{\pi_1(s)}^{3/2}(k) \left(\ell^2_{t}(i, k^{\star}) -\ell_{t}(i, k)\ell_{t}(i, k^{\star})\right)~.\label{eq:bobw_zero}
\end{align}
For the first summation of \pref{eq:bobw_zero}, the terms corresponding to $k \neq k^{\star}$ are together bounded by 
\begin{equation}
    \frac{18B}{N} \sum_{s=1}^{m_1}\gamma_{\pi_1(s)}\sum_{k \neq k^{\star}} \sqrt{\barp_{\pi_1(s)}(k)}
    \label{eq:bobw_first}
\end{equation}
by ignoring the second negative term and using the fact that $\frac{\barp_{\pi_1(s)}(j)}{p_{\pi_1(s)}(i,j)} \le 3$ according to \pref{lem: stablity} and the fact that $\ell_{t}(i, k) \in [0,1]$; the term corresponding to $k=k^{\star}$ is bounded as follows:
\begin{align}
     &\frac{6}{N^2} \sum_{s=1}^{m_1}\gamma_{\pi_1(s)}\sum_{i=1}^N \sum_{t\in \mathcal{T}_{\pi_1(s)}}  \left( \barp^{3/2}_{\pi_1(s)}(k^\star)\frac{\ell^2_{t}(i, k^\star)}{p_{\pi_1(s)}(i, k^\star)} -  \barp^{3/2}_{\pi_1(s)}(k^\star)\ell_t^2(i, k^\star)\right)  \notag
     \\&  =  \frac{6}{N^2}\sum_{s=1}^{m_1}\gamma_{\pi_1(s)}\sum_{i=1}^N \sum_{t\in \mathcal{T}_{\pi_1(s)}}  \frac{\barp^{3/2}_{\pi_1(s)}(k^\star)}{p_{\pi_1(s)}(i, k^\star)} \left( 1 -  p_{\pi_1(s)}(i, k^\star)\right) \ell^2_t(i, k^{\star}) \notag
     \\&\leq \frac{18B}{N^2}\sum_{s=1}^{m_1}\gamma_{\pi_1(s)} \sum_{i=1}^N \left( 1 -  p_{\pi_1(s)}(i, k)\right) \tag{according to \pref{lem: stablity} and $|\calT_{\pi_1(s)}|\leq B$}
     \\& = \frac{18B}{N^2} \sum_{s=1}^{m_1}\gamma_{\pi_1(s)} \sum_{k \neq k^{\star}} \sum_{i=1}^N  p_{\pi_1(s)}(i, k) \notag
     \\& \leq \frac{18B}{N}\sum_{s=1}^{m_1}\gamma_{\pi_1(s)} \sum_{k \neq k^{\star}} \left(\barp_{\pi_1(s)}(k) + \frac{2\eta B}{K^4T^5}\right)\tag{using \pref{lem: stablity}}
    \\& \leq \frac{18B}{N}\sum_{s=1}^{m_1}\gamma_{\pi_1(s)} \sum_{k \neq k^{\star}} \sqrt{\barp_{\pi_1(s)}(k)} + 36~.
    \label{eq:bobw_second}
\end{align}
Similarly, for the second summation of \pref{eq:bobw_zero}, the terms corresponding to $k \neq k^{\star}$ are together bounded by 
\begin{equation}
\frac{6B}{N} \sum_{s=1}^{m_1}\gamma_{\pi_1(s)}\sum_{k \neq k^{\star}} \sqrt{\barp_{\pi_1(s)}(k)}
    \label{eq:bobw_third}
\end{equation}
again by ignoring the second negative term; the term corresponding to $k= k^{\star}$ is simply 0. 

Regarding the second summation in \pref{eq:bobw_spade_2}, using $\ell_{t}(i, k) \in [0,1]$ we have
\begin{align}
&\E \left [\frac{6}{N^2} \sum_{s=1}^{m_1} \eta_{\pi_1(s)} \sum_{k=1}^K \barp_{\pi_1(s)}^{3/2}(k) \sum_{\substack{(i,t)\neq (j,t')\\ t,t'\in\mathcal{T}_{\pi_1(s)} \\ i, j \in [N]}}
\left(\ell_t(i, k)-\ell_t\left(i, k^{\star}\right)\right)
\left(\ell_{t'}(j, k)-\ell_{t'}\left(j, k^{\star}\right)\right) \right]\notag
\\& \leq 6 B^2 \sum_{s=1}^{m_1}\eta_{\pi_1(s)}  \sum_{k \neq k^*} \barp_{\pi_1(s)}(k)~. \label{eq:bobw_fourth}
\end{align}
Plugging \pref{eq:bobw_zero}, \pref{eq:bobw_first}, \pref{eq:bobw_second}, \pref{eq:bobw_third} and \pref{eq:bobw_fourth} in \pref{eq:bobw_spade_2}, we have
\begin{align}       
\spadesuit &= \mathbb{E}\left[\sum_{s=1}^{m_1}\left\langle \barbp_{\pi_1(s)} -\be_k, \barbz_{\pi_1(s)+1} \right\rangle\right]  \notag
\\&\leq \order\left(B\log K + \sqrt{\frac{TK}{N}} + \sqrt{BT\log K}\right) \nonumber \\
&\qquad + \frac{42B}{N} \sum_{s=1}^{m_1} \gamma_{\pi_1(s)}\sum_{k \neq k^{\star}} \sqrt{\barp_{\pi_1(s)}(k)} + 6 B^2 \sum_{s=1}^{m_1}\eta_{\pi_1(s)} \sum_{k \neq k^*} \barp_{\pi_1(s)}(k) + 36~.
\end{align}

The analysis for $\Reg_{\mathcal{P}_0}(i)$ is identical. Summing the bounds over the two parity
subsequences, using $m_1+m_2=M = T/B$ and using $\eta_t = \min \left\{\frac{1}{B},\sqrt{\frac{\log K}{tB^2}} \right\}$ and $\gamma_t = \sqrt{\frac{N}{tB}}$, we obtain 
\begin{align}
   \mathbb{E}\left[\sum_{s=1}^{T/B}\left\langle \barbp_{\pi_s} -\be_k, \barbz_{s+1} \right\rangle\right] &\leq  \order\left(B\log K + \sqrt{\frac{TK}{N}} + \sqrt{BT\log K}\right) \notag
   \\&\qquad + \frac{42B}{N} \sum_{s=1}^{T/B}\gamma_{s}\sum_{k \neq k^{\star}} \sqrt{\barp_{s}(k)} + 6 B^2 \sum_{s=1}^{T/B} \eta_s \sum_{k \neq k^*} \barp_{s}(k)\nonumber 
   \\& \leq  \mathcal{O} \left(\sqrt{BT \log K} + \sqrt{\frac{K}{N}T}  +B \log K \right)~,
\end{align}
where the last inequality holds by bounding $\sum_{k}\sqrt{\barp_s(k)}$ by $\sqrt{K}$. 
Combining the bounds for $\heartsuit$ and $\clubsuit$, we have
\begin{align*}
    \max_{i \in V}\max_{k \in [K]} \mathbb{E} \left[\sum_{\tau=1}^{T/B}\left\langle \bp'_{\tau}(i)-\be_k, \bz_{\tau+1}^B(i)\right\rangle\right] \leq \mathcal{O} \left(\sqrt{BT \log K} + \sqrt{\frac{K}{N}T}  +B \log K \right)
\end{align*}
Combining \pref{lem:delay_reduction} with the above inequality finishes the proof of the adversarial bound.

\paragraph{Stochastic environment guarantees.}
When the losses are stochastic, we still analyze the terms in \pref{eq:bobw_decomp}. For the first three terms in \pref{eq:bobw_decomp}, according to \pref{eq:bobw_adv_bias_1}, we know that
\begin{align}
&2\sum_{s=1}^{m_1-1}\left(\frac{1}{\gamma_{\pi_1(s+1)}}-\frac{1}{\gamma_{\pi_1(s)}}\right)
\left(\sum_{k=1}^{K}\sqrt{\barp_{\pi_1(s+1)}(k)}-1\right)
\notag\\
&\le \frac{4\sqrt{B}}{\sqrt{N}}
\sum_{s=1}^{m_1}\frac{1}{\sqrt{\pi_1(s)}}
\left(\sum_{k=1}^{K}\sqrt{\barp_{\pi_1(s)}(k)}-1\right) \notag\\
&\le \frac{4\sqrt{B}}{\sqrt{N}}
\sum_{s=1}^{m_1}\frac{1}{\sqrt{\pi_1(s)}}
\sum_{k\neq k^*}\sqrt{\barp_{\pi_1(s)}(k)}~.
\label{eq:bobw_bias_stochastic_1}
\end{align}
We also have the following according to \pref{eq:bobw_adv_bias_2}
\begin{align}
    &-\sum_{s=2}^{m_1}  \left(\frac{1}{\eta_{\pi_1(s)}} - \frac{1}{\eta_{\pi_1(s-1)}}\right) \sum_{j=1}^K\barp_{\pi_1(s)}(j)\log \barp_{\pi_1(s)}(j)  \notag
    \\&\leq  - \sum_{s=1}^{m_1-1} \frac{B}{\sqrt{\log K}} \cdot \frac{2}{\sqrt{\pi_1(s)}}\sum_{j=1}^K\barp_{\pi_1(s)}(j)\log \barp_{\pi_1(s)}(j) \notag
    \\&\leq \sum_{s=1}^{m_1} -\frac{B}{\sqrt{\log K}} \cdot \frac{2}{\sqrt{\pi_1(s)}}\sum_{k\neq k^{\star}}\barp_{\pi_1(s)}(k)\log \barp_{\pi_1(s)}(k) \notag
    \\& \qquad  + \sum_{s=1}^{m_1} \frac{B}{\sqrt{\log K}} \cdot \frac{2}{\sqrt{\pi_1(s)}}\sum_{k\neq k^{\star}}\barp_{\pi_1(s)}(k)~,\label{eq:bobw_bias_stochastic_2}
\end{align}
where the second inequality holds using
\begin{equation*}
    -\barp_{\pi_1(s)}(k^{\star})\log \barp_{\pi_1(s)}(k^{\star}) \leq (1- \barp_{\pi_1(s)}(k^{\star})) = \sum_{k \neq k^{\star}} \barp_{\pi_1(s)}(k^{\star})~.
\end{equation*}
Therefore, we know that
\begin{align}
&\psi_{\pi_1(m_1)}(\be_k) - \psi_{\pi_1(1)}(\barbp_{\pi_1(1)}) +\sum_{s=1}^{m_1-1}\left(\psi_{\pi_1(s)}\left(\barbp_{\pi_1(s+1)}\right)-\psi_{\pi_1(s+1)}\left(\barbp_{\pi_1(s+1)}\right)\right) \notag\\
&\leq \frac{\log K}{\eta_{\pi_1(1)}} + \frac{2\sqrt{K}-2}{\gamma_{\pi_1(1)}} + \frac{4\sqrt{B}}{\sqrt{N}}
\sum_{s=1}^{m_1}\frac{1}{\sqrt{\pi_1(s)}}
\sum_{k\neq k^*}\sqrt{\barp_{\pi_1(s)}(k)} \notag
    \\& \qquad + \sum_{s=1}^{m_1} -\frac{B}{\sqrt{\log K}} \cdot \frac{2}{\sqrt{\pi_1(s)}}\sum_{k\neq k^{\star}}\barp_{\pi_1(s)}(k)\log \barp_{\pi_1(s)}(k) \notag
    \\&\qquad+ \sum_{s=1}^{m_1} \frac{B}{\sqrt{\log K}} \cdot \frac{2}{\sqrt{\pi_1(s)}}\sum_{k\neq k^{\star}}\barp_{\pi_1(s)}(k)~.\label{eq:bow_bias_stochastic_final}
\end{align}

Plugging \pref{eq:bobw_zero}, \pref{eq:bobw_first}, \pref{eq:bobw_second}, \pref{eq:bobw_third}, \pref{eq:bobw_fourth}, \pref{eq:bow_bias_stochastic_final} and \pref{eq:bobw_z} in \pref{eq:bobw_decomp}, we obtain
\begin{align}
    \spadesuit &= \frac{\log K}{\eta_{\pi_1(1)}} + \frac{2\sqrt{K}-2}{\gamma_{\pi_1(1)}} + \frac{4\sqrt{B}}{\sqrt{N}}
\sum_{s=1}^{m_1}\frac{1}{\sqrt{\pi_1(s)}}
\sum_{k\neq k^*}\sqrt{\barp_{\pi_1(s)}(k)} \notag
    \\& \quad + \sum_{s=1}^{m_1} -\frac{B}{\sqrt{\log K}} \cdot \frac{2}{\sqrt{\pi_1(s)}}\sum_{k\neq k^{\star}}\barp_{\pi_1(s)}(k)\log \barp_{\pi_1(s)}(k) \notag
    \\&\quad+ \sum_{s=1}^{m_1} \frac{B}{\sqrt{\log K}} \cdot \frac{2}{\sqrt{\pi_1(s)}}\sum_{k\neq k^{\star}}\barp_{\pi_1(s)}(k)
    +\frac{42B}{N} \sum_{s=1}^{m_1} \gamma_{\pi_1(s)}\sum_{k \neq k^{\star}} \sqrt{\barp_{\pi_1(s)}(k)} \notag
    \\& \quad + \E \left [\frac{6}{N^2} \sum_{s=1}^{m_1} \eta_{\pi_1(s)} \sum_{k=1}^K \barp_{\pi_1(s)}^{3/2}(k) 
\sum_{\substack{(i,t)\neq (j,t')\\ t,t'\in\mathcal{T}_{\pi_1(s)} \\ i, j \in [N]}}
\left(\ell_t(i, k)-\ell_t\left(i, k^{\star}\right)\right)
\left(\ell_{t'}(j, k)-\ell_{t'}\left(j, k^{\star}\right)\right) \right] \label{eq:bobw_stochastic_spade}
\end{align}
Then we analyze the last summation of \pref{eq:bobw_stochastic_spade}, we have
\begin{align}
&\E \left [\frac{6}{N^2} \sum_{s=1}^{m_1} \eta_{\pi_1(s)} \sum_{k=1}^K \barp_{\pi_1(s)}^{3/2}(k)
\sum_{\substack{(i,t)\neq (j,t')\\ t,t'\in\mathcal{T}_{\pi_1(s)} \\ i, j \in [N]}}
\left(\ell_t(i, k)-\ell_t\left(i, k^{\star}\right)\right)
\left(\ell_{t'}(j, k)-\ell_{t'}\left(j, k^{\star}\right)\right) \right]\notag
\\& \leq\frac{6}{N^2} \sum_{s=1}^{m_1} \eta_{\pi_1(s)} \sum_{k=1}^K \barp_{\pi_1(s)}^{3/2}(k)
\sum_{\substack{(i,t)\neq (j,t')\\ t,t'\in\mathcal{T}_{\pi_1(s)} \\ i, j \in [N]}}
\left(\mu(i, k)-\mu\left(i, k^{\star}\right)\right)
\left(\mu(j, k)-\mu\left(j, k^{\star}\right)\right) \notag\\
&= \frac{6}{N^2} \sum_{s=1}^{m_1} \eta_{\pi_1(s)} \sum_{k=1}^K \barp_{\pi_1(s)}^{3/2}(k)
\Bigg[
\Big(\sum_{t\in\mathcal{T}_{\pi_1(s)}}\sum_{i=1}^N \big(\mu(i, k)-\mu(i, k^{\star})\big)\Big)^2
\\&\quad-\sum_{t\in\mathcal{T}_{\pi_1(s)}}\sum_{i=1}^N \big(\mu(i, k)-\mu(i, k^{\star})\big)^2
\Bigg]\notag\\
&= \frac{6}{N^2} \sum_{s=1}^{m_1} \eta_{\pi_1(s)} \sum_{k=1}^K \barp_{\pi_1(s)}^{3/2}(k)
\Bigg[
|\mathcal{T}_{\pi_1(s)}|^2\Big(\sum_{i=1}^N \big(\mu(i, k)-\mu(i, k^{\star})\big)\Big)^2
\\&\quad-|\mathcal{T}_{\pi_1(s)}|\sum_{i=1}^N \big(\mu(i, k)-\mu(i, k^{\star})\big)^2
\Bigg]\notag\\
&\le \frac{6}{N^2} \sum_{s=1}^{m_1} \eta_{\pi_1(s)} \sum_{k=1}^K \barp_{\pi_1(s)}^{3/2}(k)\;
|\mathcal{T}_{\pi_1(s)}|^2\left(\sum_{i=1}^N \big(\mu(i, k)-\mu(i, k^{\star})\big)\right)^2 \notag\\
&= \frac{6}{N^2} \sum_{s=1}^{m_1} \eta_{\pi_1(s)} \sum_{k=1}^K \barp_{\pi_1(s)}^{3/2}(k)
|\mathcal{T}_{\pi_1(s)}|^2 N^2\delta(k)^2\notag\\
&\le 6B^2\sum_{s=1}^{m_1} \eta_{\pi_1(s)} \sum_{k \neq k^{\star}}\barp_{\pi_1(s)}(k)\delta(k)
\label{eq:Bobw_entropy_stochast}
\end{align}
Plugging \pref{eq:Bobw_entropy_stochast} in \pref{eq:bobw_stochastic_spade}, we have
\begin{align}       
\spadesuit &= \mathbb{E}\left[\sum_{s=1}^{m_1}\left\langle \barbp_{\pi_1(s)} -\be_k, \barbz_{\pi_1(s)+1} \right\rangle\right]  \notag
\\&\leq  \frac{\log K}{\eta_{\pi_1(1)}} + \frac{2\sqrt{K}-2}{\gamma_{\pi_1(1)}} + \frac{46 \sqrt{B}}{\sqrt{N}} \sum_{s=1}^{m_1} \frac{1}{\sqrt{\pi_1(s)}}\sum_{k \neq k^{\star}} \sqrt{\barp_{\pi_1(s)}(k)} + 36 \notag
\\&\quad + \sum_{s=1}^{m_1} -\frac{B}{\sqrt{\log K}} \cdot \frac{2}{\sqrt{\pi_1(s)}}\sum_{k\neq k^{\star}}\barp_{\pi_1(s)}(k)\log \barp_{\pi_1(s)}(k) \notag
\\& \quad + \sum_{s=1}^{m_1} \frac{B}{\sqrt{\log K}} \cdot \frac{2}{\sqrt{\pi_1(s)}}\sum_{k\neq k^{\star}}\barp_{\pi_1(s)}(k) \notag
\\& \quad + 6B\sum_{s=1}^{m_1} \sqrt{\frac{\log K}{\pi_1(s)}}\sum_{k \neq k^{\star}}^K \barp_{\pi_1(s)}(k)\delta(k)
\end{align}
The analysis for $\Reg_{\mathcal{P}_0}(i)$ is identical. Summing the bounds over the two parity
subsequences, using $m_1+m_2=M = T/B$, we obtain that

\begin{align*}
\mathbb{E}\left[\sum_{s=1}^{T/B}\left\langle \barbp_{s} -\be_k, \barbz_{s+1} \right\rangle\right] &\leq  \frac{\log K}{\eta_{\pi_0(1)}} + \frac{2\sqrt{K}-2}{\gamma_{\pi_0(1)}}+ \frac{\log K}{\eta_{\pi_1(1)}} + \frac{2\sqrt{K}-2}{\gamma_{\pi_1(1)}} + 72
\\& \quad + \frac{C}{2} \Bigg (\frac{\sqrt{B}}{\sqrt{N}} \sum_{s=1}^{T/B} \frac{1}{\sqrt{s}}\sum_{k \neq k^{\star}} \sqrt{\barp_{\pi_1(s)}(k)} 
\\& \quad + \sum_{s=1}^{T/B} -\frac{B}{\sqrt{\log K}} \cdot \frac{1}{\sqrt{s}}\sum_{k\neq k^{\star}}\barp_{s}(k)\log \barp_{s}(k) 
\\& \quad + \sum_{s=1}^{T/B} \frac{B}{\sqrt{\log K}} \cdot \frac{1}{\sqrt{s}}\sum_{k\neq k^{\star}}\barp_{s}(k)
\\&\quad+ B\sum_{s=1}^{T/B} \sqrt{\frac{\log K}{s}} \sum_{k \neq k^{\star}} \barp_{s}(k)\delta(k) \Bigg),
\end{align*}
where $C> 0$ is a universal constant. Note that by definition of $\delta(k)$, we have
\begin{equation*}
    \mathbb{E}\left[\sum_{s=1}^{T/B}\left\langle \barbp_{s} -\be_k, \barbz_{s+1} \right\rangle\right] = \sum_{s=1}^{T/B} \sum_{k \neq k^{\star}} \barp_s(k) \delta(k)~.
\end{equation*}
Therefore, we are able to rewrite the regret as follows:
\begin{align}
   &\mathbb{E}\left[\sum_{s=1}^{T/B}\left\langle \barbp_{s} -\be_k, \barbz_{s+1} \right\rangle\right] \\&\qquad=  2\mathbb{E}\left[\sum_{s=1}^{T/B}\left\langle \barbp_{s} -\be_k, \barbz_{s+1} \right\rangle\right] - \mathbb{E}\left[\sum_{s=1}^{T/B}\left\langle \barbp_{s} -\be_k, \barbz_{s+1} \right\rangle\right]\notag
   \\ &\qquad\leq \order\left(B \log K + \sqrt{\frac{BK}{N}}\right)\notag
\\& \qquad\quad + C \frac{\sqrt{B}}{\sqrt{N}} \sum_{s=1}^{T/B} \frac{1}{\sqrt{s}}\sum_{k \neq k^{\star}} \sqrt{\barp_{\pi_1(s)}(k)} - \frac{1}{4}\sum_{s=1}^{T/B} \sum_{k \neq k^{\star}} \barp_s(k) \delta(k)\notag
\\& \qquad\quad + C \sum_{s=1}^{T/B} -\frac{B}{\sqrt{\log K}} \cdot \frac{1}{\sqrt{s}}\sum_{k\neq k^{\star}}\barp_{s}(k)\log \barp_{s}(k) -\frac{1}{4}\sum_{s=1}^{T/B} \sum_{k \neq k^{\star}} \barp_s(k) \Delta(k)\notag
\\& \qquad\quad + C \sum_{s=1}^{T/B} \frac{B}{\sqrt{\log K}} \cdot \frac{1}{\sqrt{s}}\sum_{k\neq k^{\star}}\barp_{s}(k)-\frac{1}{4}\sum_{s=1}^{T/B} \sum_{k \neq k^{\star}} \barp_s(k) \delta(k)\notag
\\&\qquad\quad+ C B\sum_{s=1}^{T/B} \sqrt{\frac{\log K}{s}}\sum_{k \neq k^{\star}}^K \barp_{s}(k)\delta(k) - \frac{1}{4}\sum_{s=1}^{T/B} \sum_{k \neq k^{\star}} \barp_s(k) \delta(k). \label{eq: bobw_stochastic_decop}
\end{align}
Next, we analyze the four summations in \pref{eq: bobw_stochastic_decop} separately.
\paragraph{Bounding the first summation in \pref{eq: bobw_stochastic_decop}.} Direct calculation shows that
\begin{align}
  \sum_{s=1}^{T/B}\sum_{k \neq k^{\star}} \left(   C\sqrt{\frac{B}{N}} \frac{1}{\sqrt{s}}  \sqrt{\barp_{\pi_1(s)}(k)} - \frac{1}{4}\barp_{\pi_1(s)}(k) \delta(k)\right)&\leq  \sum_{s=1}^{T/B}\sum_{k \neq k^{\star}} \frac{BC^2}{Ns\delta(k)} \notag
    \\&\leq \order \left(\frac{B}{N} \ln \left(\frac{T}{B}\right) \sum_{k \neq k^{\star}} \frac{1}{\delta(k)}\right)~.\label{eq:bobw_stochastic_first}
\end{align}
where the first inequality holds by using $2 \sqrt{x y}-y \leq x$ for all $x, y\geq 0$.

\paragraph{Bounding the second summation in \pref{eq: bobw_stochastic_decop}.} We define
\[
b_s \triangleq \frac{CB}{4\sqrt{s\log K}}~.
\]
Fix $k\neq k^\star$ and $\delta(k)>0$. Define $g(z)\triangleq -b_s z\log z-\delta(k)z$.
Direct calculation shows that $g(z)$ is concave and $z^\star\triangleq \argmin_{z}g(z)=\exp(-\delta(k)/b_s-1)$, hence
\[
-b_s z\log z\le \delta(k)z+g(z^\star) = \delta(k)z + b_s\exp\left(-\frac{\delta(k)}{b_s}-1\right).
\]
Summing over all $s\le T/B$ shows that 
\begin{align}
    &\sum_{s=1}^{T/B} -b_sz\log z-\delta(k)z \nonumber\\
    &\leq \sum_{s=1}^{T/B} \sum_{k \neq k^{\star}} b_s\exp\Big(-\frac{\delta(k)}{b_s}-1\Big) \nonumber\\
&= \sum_{s=1}^{T/B} \sum_{k \neq k^{\star}} \frac{BC}{4\sqrt{s\log K}}
\exp\Big(-\frac{4\delta(k)\sqrt{s\log K}}{BC}-1\Big) \notag
\\& \leq \sum_{k \neq k^{\star}}\int_{0}^{\infty} \frac{BC}{4\sqrt{x\log K}}
\exp\Big(-\frac{4\delta(k)\sqrt{x\log K}}{BC}-1\Big)dx\notag
\\&\leq \order \left(\sum_{k \neq k^{\star}} \frac{B^2}{\delta(k)\log K}\right) \label{eq:bobw_stochastic_second}
\end{align}
where the last inequality holds by using
$\int_0^{\infty} \frac{a}{\sqrt{x}}\exp(-\frac{\delta\sqrt{x}}{a}-1)\,dx
= \frac{2a^2}{e\delta}$
with $a=\frac{BC}{4\sqrt{\log K}}$ and $\delta=\delta(k)$.

\paragraph{Bounding the third summation in \pref{eq: bobw_stochastic_decop}.} Define $s^\star(k)\triangleq \left\lceil \frac{16C^2B^2}{\delta(k)^2 \log K} \right \rceil$. Then, we know that for all $s\geq s^{\star}(k)$,
\begin{equation*}
   C\cdot\frac{B}{\sqrt{\log K}} \cdot \frac{1}{\sqrt{s}}\barp_s(k)  - \frac{1}{4}\barp_{s}(k) \delta(k)  \leq 0
\end{equation*}
Therefore, it suffices to bound the third summation for $s$ from $1$ to $s^\star(k)$ for each $k\neq k^\star$. Then direct calculation shows that
\begin{align}
      &\sum_{k \neq k^{\star}} \sum_{s=1}^{s^{\star}(k)}\left(\frac{CB}{\sqrt{\log K}} \cdot \frac{1}{\sqrt{s}} - \frac{1}{4}\barp_{s}(k) \delta(k)\right) \nonumber \\
      &\leq \sum_{k \neq k^{\star}} \sum_{s=1}^{s^{\star}(k)}\frac{CB}{\sqrt{\log K}} \cdot \frac{1}{\sqrt{s}} \nonumber
      \\&= \sum_{k \neq k^{\star}} \sum_{s=1}^{\left\lceil \frac{16C^2B^2}{\delta(k)^2 \log K} \right \rceil}\frac{CB}{\sqrt{\log K}} \cdot \frac{1}{\sqrt{s}} \leq \order\left(\sum_{k \neq k^{\star}} \frac{B^2}{\delta(k)\log K}\right)~.\label{eq:bobw_stochastic_third}
\end{align}

\paragraph{Bounding the fourth summation in \pref{eq: bobw_stochastic_decop}.} Define $s^\star = \left\lceil 16C^2B^2\log K \right \rceil $. Since when $s\geq s^\star$, we have
\begin{equation*}
  BC  \sqrt{\frac{\log K}{s}}  \barp_{s}(k) \delta(k)- \frac{1}{4}\barp_{s}(k) \delta(k) \leq 0~.
\end{equation*}
Therefore, it suffices to bound the fourth summation for $s$ from $1$ to $s^\star$. Direct calculation shows that
\begin{align}
    &\sum_{s=1}^{s^{\star}}\sum_{k \neq k^{\star}} \left( BC  \sqrt{\frac{\log K}{s}}  \barp_{s}(k) \delta(k)- \frac{1}{4}\barp_{s}(k) \delta(k)\right) \notag
    \\& \leq \sum_{s=1}^{ \left\lceil 16C^2B^2\log K \right \rceil}\sum_{k \neq k^{\star}} BC  \sqrt{\frac{\log K}{s}}  \barp_{s}(k) \delta(k) \notag
     \\& \leq \sum_{s=1}^{ \left\lceil 16C^2B^2\log K \right \rceil} BC  \sqrt{\frac{\log K}{s}} \notag
    \\& \leq \order \left(B^2 \log K \right).
    \label{eq:bobw_stochastic_fourth}
\end{align}

Plugging \pref{eq:bobw_stochastic_first}, \pref{eq:bobw_stochastic_second}, \pref{eq:bobw_stochastic_third} and \pref{eq:bobw_stochastic_fourth} in \pref{eq: bobw_stochastic_decop} and combining the bounds for $\heartsuit$ and $\clubsuit$, we have

\begin{align*}
    &\max_{i \in V}\max_{k \in [K]} \mathbb{E} \left[\sum_{\tau=1}^{T/B}\left\langle \bp'_{\tau}(i)-\be_k, \bz_{\tau+1}^B(i)\right\rangle\right] \\& \qquad \qquad \qquad \qquad\leq \order \left( \frac{B}{N} \ln \left(\frac{T}{B}\right) \sum_{k \neq k^{\star}} \frac{1}{\delta(k)} + \sum_{k \neq k^{\star}} \frac{B^2}{\delta(k)\log K}  + B^2\log K\right) 
\end{align*}
Combining \pref{lem:delay_reduction} with the above inequality finishes the proof.

\end{proof}

\section{Omitted Proof Details for Distributed Linear Bandit}
\label{app: Omitted Proof Details for Linear Bandit}

\subsection{Omitted Algorithm Description}\label{app: alg_base_linear}
In this section, we provide the omitted details regarding the description of the base algorithm $\calB$ in \pref{thm:mainalb}. We utilize an instance of FTRL that first reconstructs the estimated loss for every original action $\ba_k \in \Omega$ using the spanner decomposition coefficients $\blambda^{(k)}$, and subsequently performs a standard FTRL update with entropy regularization.
\begin{algorithm}[ht]
   \caption{FTRL for Linear Bandits}
   \label{alg: ftrl-linear}
    \textbf{Input:} Regularizer $\psi$, $K$-sized action set $\Omega$, volumetric spanner $\calS\subset [K]$ of $\Omega$ 

    \textbf{Initialize:} $\bq_1 = \frac{1}{K}\bm{1}\in \Delta(K)$.
    
    \For{$t =1,\cdots T$}{
       Output $\bq_t$ and receive $\bz_t\in \R^{|\calS|}$
       
       Construct $\bztilde_t \in \R^K$ where $\ztilde_t(k) = \sum_{j=1}^ {|\calS|}\lambda^{(k)}(j)z_t(j)$ for all $k\in [K]$
       
       Update ${\displaystyle \bq_{t+1} = \argmin_{\bq\in \Delta(K)} \sum_{\tau\leq t}\inner{\bztilde_\tau, \bq} + \psi(\bq) }$
   }
\end{algorithm}

\subsection{Omitted Proof Details for \pref{lem: cons_linear}}
\label{app: lemcons_linear}
We begin by proving a consensus error bound for \pref{alg: black-box-linear bandit}, analogous to \pref{lem: cons} in the MAB setting. Specifically, we first demonstrate that the local gossip vector at the end of each block concentrates around the network-wide averaged loss defined on the volumetric spanners. This result then implies a corresponding concentration for the reconstructed losses of the original actions.

\begin{restatable}{rlemma}{apprerrorlinear}
\label{lem: cons_linear}
Assume all agents $i \in V$ run \pref{alg: black-box-linear bandit} with an arbitrary linear bandits algorithm \alg.
Then, under the same assumptions as \pref{thm:mainalb},
\begin{equation}
    {\max_{i \in V} \max_{\tau \in [T/B]} \left\|\bz_{\tau}^B(i)-\barbzS_{\tau} \right\|_2 \leq \frac{2}{T^2K^3}\;,}
\end{equation}
where $\barzS_{\tau}(k) = \frac{1}{N}\sum_{i=1}^N \sum_{t\in\mathcal{T}_{\tau-1}}\inner{\bb_k, \bthetahat_t(i)}$ for all $ k \in [|\calS|]$. Moreover, we also have
\begin{equation*}
    \max_{i \in V} \max_{\tau \in [T/B]} \left\|\wt\bz_{\tau}^B(i)-\barbz_{\tau}\right\|_2 \leq \frac{1}{K^{1.5}T^2}\;,
\end{equation*}
where we recall that $\wtz_{\tau}^B(i,k) = \sum_{j=1}^{|\calS|} \lambda^{(k)}(j)\, z_{\tau}^B(i,j)$ and with an abuse of notation, we define
\begin{align}\label{eqn:barz_linear}
    \barz_{\tau}(k)  &\triangleq \sum_{j=1}^{|\calS|}\lambda^{(k)}(j) \barzS_{\tau}(j),
\end{align}
for all $k \in [K]$.
\end{restatable}
\begin{proof}
Following the analysis of \pref{lem: cons} and \pref{eq: gossip_inequality}, we know that
\begin{equation*}
    \left\|\bz_{\tau}^B(i)-\barbzS_{\tau}\right\|_2 \leq  \frac{2 \sqrt{\sum_{i=1}^N\left\|\sum_{t \in \mathcal{T}_{\tau-1}}\ellhat_t (i)\right\|_2^2}}{T^6 K^6 \sqrt{N}}~.
\end{equation*}
According to  \pref{lem:bounded_volumetric_l}, we have 
\begin{equation}
\ellhat(i,k) =\left|\inner{\bb_k, \bthetahat_{t}(i)} \right| \leq {\frac{|\calS|}{\beta}} \qquad \text{for all } k \in [|\calS|].
\end{equation}
Using {$|\calS| \leq  K$}, we have 
\begin{equation*}
    \left\|\bz_\tau^B(i)-\barbzS_\tau\right\|_2 \leq {\frac{2 BK}{T^6 K^6 \beta} \leq \frac{2 B}{T^6 K^5 \beta} \leq  \frac{2}{T^2K^3}.}
\end{equation*}
For the last inequality, it suffices to show that
$\beta\ge \frac{B}{2T^4K^2}$.
Since we pick $\beta=3Bd\eta$ and $\eta = \min\Big\{\frac{1}{6Bd},\sqrt{\log K/(dTB+\frac{dT}{N}})\Big\}$, using $d\le K$, $B\le T$, $N\ge 1$, and $K\ge 2$ (so $\log K\ge \log 2$), we have
{
\[
\beta
= 3Bd\sqrt{\frac{\log K}{dTB+\frac{dT}{N}}}
\ge 3Bd\sqrt{\frac{\log 2}{KT^2+TK}}.
\]}
Moreover, for $T\ge 2$ and $K\ge 2$, we have {$\sqrt{\frac{\log 2}{KT^2+TK}}\ge \frac{\log 2}{2T^2K}$}, hence {$\beta\ge \frac{3Bd\log 2 }{2T^2K} \ge \frac{B}{2T^4K^2}$.}

Next, we bound the reconstruction error on the original action space. Recall from \pref{alg: ftrl-linear} that
\begin{equation*}
    \wtz_{\tau}^B(i, k)
    = \sum_{j=1}^{|\mathcal{S}|} \lambda^{(k)}(j)\, z_{\tau}^B(i,j),
    \qquad \forall k \in [K]~.
\end{equation*}
Note that $\wt\bz_{\tau}^B(i) \in \R^K$ while $\bz_{\tau}^B(i) \in \R^{|\calS|}$. Then, for each $k \in [K]$,
\begin{align}
    \barz_{\tau}(k)
    &=\sum_{j=1}^{|\calS|} \lambda^{(k)}(j)\,\barzS_{\tau}(j)
    \notag\\
    &= \sum_{j=1}^{|\calS|} \lambda^{(k)}(j)
    \left(\frac{1}{N}\sum_{i=1}^N\sum_{t\in\mathcal{T}_{\tau-1}}\inner{\bb_j,\bthetahat_t(i)}\right)
    \notag\\
    &= \frac{1}{N}\sum_{i=1}^N\sum_{t\in\mathcal{T}_{\tau-1}}\inner{\sum_{j=1}^{|\calS|}\lambda^{(k)}(j)\bb_j,\bthetahat_t(i)}
    \notag\\
    &= \frac{1}{N}\sum_{i=1}^N\sum_{t\in\mathcal{T}_{\tau-1}}\inner{\ba_k,\bthetahat_t(i)}.\label{eqn:barz_lienar_alter}
\end{align}
Consequently, for any $k \in [K]$,
\begin{align*}
   \left|\barz_{\tau}(k) - \wtz_{\tau}^B(i, k) \right|
   &= \left|\sum_{j=1}^{|\calS|} \lambda^{(k)}(j) \barzS_{\tau}(j)-\sum_{j=1}^{|\calS|} \lambda^{(k)}(j) z_{\tau}^B(i,j) \right|
   \\&= \left|\sum_{j=1}^{|\calS|} \lambda^{(k)}(j) \left(\barzS_{\tau}(j) - z_{\tau}^B(i,j)\right)\right|
   \\&= \left|\inner{\blambda^{(k)}, \barbzS_{\tau} - \bz_{\tau}^B(i)}\right|
   \\&\le \|\blambda^{(k)}\|_2\,\left\|\barbzS_{\tau} - \bz_{\tau}^B(i)\right\|_2
   \\&\le \left\|\barbzS_{\tau} - \bz_{\tau}^B(i)\right\|_2
   \le {\frac{2}{T^2K^3}~,}
\end{align*}
where we used Cauchy--Schwarz and the volumetric spanner property $\|\blambda^{(k)}\|_2\le 1$.
Therefore,
\begin{equation*}
    \left\|\barbz_{\tau} - \wt\bz_{\tau}^B(i) \right\|_2
    \le \sqrt{K}\cdot \max_{k\in[K]}\left|\barz_{\tau}(k)-\wtz_{\tau}^B(i,k)\right|
    \le {\frac{2\sqrt{K}}{T^2K^3}
    = \frac{2}{T^2K^{5/2}}
    \le \frac{1}{K^{1.5}T^2}\;,}
\end{equation*}
where the last inequality holds by using $K \geq 2$.
\end{proof}

\subsection{Omitted Proof Details for \pref{lem:delay_reduction_l}}
\label{app:delay_reduction_l}
In this subsection, we provide the proof of \pref{lem:delay_reduction_l}, which is the analog of \pref{lem:delay_reduction} in the MAB setting.

\begin{restatable}{rlemma}{delayreductionl}\label{lem:delay_reduction_l}
Assume all agents $i \in V$ run \pref{alg: black-box-linear bandit} with a delayed linear bandits algorithm \alg\ whose predictions $\bp'_1(i),\ldots,\bp'_{T/B}(i)$ for each agent $i \in V$ satisfy
\begin{equation}
\label{eq:delayed_guarantee_l}
\max_{i\in V} \max_{k\in[K]} \E\left[\sum_{\tau=1}^{T/B}\left\langle \bp'_\tau(i)- \be_k, \wtbz_{\tau+1}^{B}(i)\right\rangle\right]
\le \Regdel
\end{equation}
for some $\Regdel>0$.
Then, under the same assumptions as \pref{thm:mainalb}, the agents' regret satisfies
\[
    \max_{i\in V} \Reg_T(i)
\le
    \Regdel
+
    {6\sqrt{dBT\log K}} + 3~.
\]
\end{restatable}
\begin{proof}
~Recall that \pref{alg: black-box-linear bandit} uses
$\bp_\tau(i)=(1-\alpha-\beta)\bp'_\tau(i)+ \frac{\alpha}{K}\mathbf{1}+ \frac{\beta}{|\calS|}\mathbf{1}_{\calS}$,
where $\bp'_\tau(i)\in\Delta(K)$ is the distribution output by \alg.
Fix an agent $i\in V$, and let
\[
    k \in \argmax_{k' \in [K]} \E\left[\sum_{\tau=1}^{T/B}\left\langle \be_{k'}, \sum_{t \in \mathcal{T}_{\tau}} \bellbar_t \right\rangle\right].
\]
Then we can write
\begin{align}
\Reg_T(i)
&= \E\left[\sum_{\tau=1}^{T/B}\left\langle \bp_\tau(i)-\be_k,  \sum_{t \in \mathcal{T}_{\tau}} \bellbar_t \right\rangle\right] \notag\\
&= \E\left[\sum_{\tau=1}^{T/B}\left\langle \bp_\tau(i)-\be_k,\barbz_{\tau+1}\right\rangle\right] \tag{by \pref{eqn:barz}} \\
&=
    \E\left[\sum_{\tau=1}^{T/B}\left\langle \left(1-\alpha-\beta\right)\bp'_\tau(i)+\frac{\alpha}{K}\mathbf{1}+ \frac{\beta}{|\calS|}\mathbf{1}_{\calS}-\be_k,\wtbz_{\tau+1}^B(i)+\barbz_{\tau+1}-\wtbz_{\tau+1}^B(i)\right\rangle\right]
\notag\\
&= \left(1-\alpha-\beta\right)\E\left[\sum_{\tau=1}^{T/B}\left\langle \bp'_\tau(i)-\be_k,\wtbz_{\tau+1}^B(i)\right\rangle\right] \notag
\\& \quad + \left(1-\alpha-\beta\right)\E\left[\sum_{\tau=1}^{T/B}\left\langle \bp'_\tau(i)-\be_k,\barbz_{\tau+1}-\wtbz_{\tau+1}^B(i)\right\rangle\right] \notag
\\& \quad + \alpha\E\left[\sum_{\tau=1}^{T/B}\left \langle \frac{1}{K}\mathbf{1}-\be_k,\barbz_{\tau+1}\right\rangle\right]\notag
\\& \quad + \beta\E\left[\sum_{\tau=1}^{T/B}\left \langle \frac{1}{|\calS|}\mathbf{1}_{\calS}-\be_k,\barbz_{\tau+1}\right\rangle\right] \notag
\\&\le \Regdel \tag{by \pref{eq:delayed_guarantee_l}}
\\& \quad + \left(1-\alpha-\beta\right)\E\left[\sum_{\tau=1}^{T/B}\left\langle \bp'_\tau(i)-\be_k,\barbz_{\tau+1}-\wtbz_{\tau+1}^B(i)\right\rangle\right] \label{eq:second-term-l}
\\& \quad + \alpha\E\left[\sum_{\tau=1}^{T/B}\left \langle \frac{1}{K}\mathbf{1}-\be_k,\barbz_{\tau+1}\right\rangle\right] \label{eq:third-term-l}
\\& \quad + \beta\E\left[\sum_{\tau=1}^{T/B}\left \langle \frac{1}{|\calS|}\mathbf{1}_{\calS}-\be_k,\barbz_{\tau+1}\right\rangle\right]\;.
\label{eq:fourth-term-l}
\end{align}

\paragraph{Bounding~\pref{eq:second-term-l}.}
By the Cauchy--Schwarz inequality, for each block $\tau$,
\[
\left\langle \bp'_\tau(i)-\be_k,\barbz_{\tau+1}-\wtbz_{\tau+1}^{B}(i)\right\rangle
\le \big\|\bp'_\tau(i)-\be_k\big\|_2\;
\big\|\barbz_{\tau+1}-\wtbz_{\tau+1}^{B}(i)\big\|_2~.
\]
Since $\bp'_\tau(i)\in\Delta(K)$, we have
$\|\bp'_\tau(i)-\be_k\|_2 \le \sqrt{2}$.
Using \pref{lem: cons_linear},
\[
\left\|\barbz_{\tau+1}-\wtbz_{\tau+1}^{B}(i)\right\|_2 \le \frac{1}{K^{1.5}T^2}\leq \frac{1}{KT} .
\]
Therefore,
\[
\left\langle \bp'_\tau(i)-\be_k, \barbz_{\tau+1}-\wtbz_{\tau+1}^{B}(i)\right\rangle
\le \frac{\sqrt{2}}{KT}.
\]
Taking expectations and summing over $\tau=1,\dots,T/B$ yields
\begin{equation}\label{eq:disc_term_bound_l}
\left(1-\alpha -\beta\right)\E\left[\sum_{\tau=1}^{T/B}\left\langle \bp'_\tau(i)-\be_k,\barbz_{\tau+1}-\wtbz_{\tau+1}^B(i)\right\rangle\right]
\le \frac{T}{B} \cdot \frac{\sqrt{2}}{KT}
= \frac{\sqrt{2}}{BK} \le 1,
\end{equation}
where the last inequality uses $K \ge 2$ and $B\ge 1$.
\paragraph{Bounding~\pref{eq:third-term-l} and~\pref{eq:fourth-term-l}.}
Taking expectations and using unbiasedness of $\bthetahat_t(i)$ gives
\[
\E[\barz_{\tau+1}(k)]
=\frac{1}{N}\sum_{t\in\mathcal T_\tau}\sum_{i=1}^N \inner{\ba_k,\theta_t(i)}
=\frac{1}{N}\sum_{t\in\mathcal T_\tau}\sum_{i=1}^N \ell_t(i,k).
\]
Since $\ell_t(i,k)\in[-1,1]$ and $|\mathcal T_\tau|=B$, we have
\[
-B \le \E[\barz_{\tau+1}(k)] \le B,
\quad \forall k\in[K]~.
\]
Then, by linearity of expectation,
\[
\E\left[\inner{\frac{1}{K}\mathbf{1}- \be_k,\barbz_{\tau+1}}\right]
= \inner{\frac{1}{K}\mathbf{1}-\be_k,\E[\barbz_{\tau+1}]}
= \inner{\frac{1}{K}\mathbf{1},\E[\barbz_{\tau+1}]}-\inner{\be_k,\E[\barbz_{\tau+1}]}\;.
\]
Hence, we have
\[
\E\left[\inner{\frac{1}{K}\mathbf{1} - \be_k,\barbz_{\tau+1}}\right]\le 2B.
\]
The analysis for \pref{eq:fourth-term-l} is similar, and we obtain
\[
\E\left[\inner{\frac{1}{|\calS|}\mathbf{1}_{\calS} - \be_k,\barbz_{\tau+1}}\right]\le 2B.
\]

Combining the bounds for both terms, we know that
\begin{align}
\alpha\E\left[\sum_{\tau=1}^{T/B} \inner{\frac{1}{K}\mathbf{1}-\be_k,\barbz_{\tau+1}}\right] +
\beta\E\left[\sum_{\tau=1}^{T/B} \inner{\frac{1}{|\calS|}\mathbf{1}_{\calS}-\be_k,\barbz_{\tau+1}}\right]
&\leq 2\alpha T + 2\beta T \notag
\\&\le 2 + {6\sqrt{dBT\log K},}
\label{eq:smoothing_term_bound_l}
\end{align} 
where we used $\alpha=\frac{1}{T}$ and {$\beta=3Bd\eta = 3Bd\cdot\min\left\{\frac{1}{6Bd},\sqrt{\frac{\log K}{dTB+\frac{dT}{N}}}\right\}$} in the last inequality. Plugging \pref{eq:disc_term_bound_l} and \pref{eq:smoothing_term_bound_l}
into the decomposition finishes the proof.

\end{proof}

\subsection{Omitted Proof Details for~\pref{lem: delayftrlmabl}}
In this subsection, we prove \pref{lem: delayftrlmabl}, which adapts the argument of \pref{lem: delayftrlmab} from MAB setting to the linear bandits setting.
\label{app: delayftrlmabl}
\begin{restatable}{rlemma}{delayftrlmabl}
\label{lem: delayftrlmabl}
Under the same assumptions as in \pref{thm:mainalb}, \pref{alg: black-box-linear bandit} guarantees
\begin{align}\label{eqn:delayed_guarantee_l}
        \max_{i\in V}\max_{k\in[K]} \mathbb{E} \left[\sum_{\tau=1}^{M}\left\langle \bp'_{\tau}(i)-\be_k, \wtbz_{\tau+1}^B(i)\right\rangle\right] \leq \order\left(\sqrt{\log K\left(B+\frac{1}{N}\right)dT}+Bd\log K\right).
    \end{align}
\end{restatable}
\begin{proof}
~The analysis is similar to \pref{lem: delayftrlmab}; therefore, we use the same notation, including the parity sets ($\calP_0$ and $\calP_1$) and the associated index maps ($\pi_0(s)$ and $\pi_1(s)$).
Fix an agent $i \in V$ and $k\in[K]$. We decompose the regret as follows:
\begin{align*}
&\mathbb{E}\left[\sum_{\tau=1}^{M}\Big\langle \bp'_{\tau}(i)-\be_k, \wtbz_{\tau+1}^B(i)\Big\rangle\right]\notag\\
&=\underbrace{\mathbb{E}\left[\sum_{s=1}^{m_0}\left\langle \bp'_{\pi_0(s)}(i)-\be_k, \wtbz_{\pi_0(s)+1}^B(i)\right\rangle\right]}_{\triangleq \Reg_{\mathcal{P}_0}(i)}
+\underbrace{\mathbb{E} \left[\sum_{s=1}^{m_1}\left\langle \bp'_{\pi_1(s)}(i)-\be_k, \wtbz_{\pi_1(s)+1}^B(i)\right\rangle\right]}_{\triangleq\Reg_{\mathcal{P}_1}(i)}.
\end{align*}
We now analyze $\Reg_{\mathcal{P}_1}(i)$; the analysis for $\Reg_{\mathcal{P}_0}(i)$ is analogous.
By the update rule of $\bq_t$ in \pref{alg: ftrl-linear},
\begin{align}
\label{eqn:p_prime_ftrl_l}
\bp'_{\pi_1(s)}(i)
= \argmin_{\bq \in \Delta(K)} \left\{\sum_{s'=1}^{s-1} \left \langle \wtbz_{\pi_1(s')+1}^B(i), \bq \right \rangle + \frac{1}{\eta} \sum_{k=1}^K q(k) \log \left(q(k)\right)\right\}.
\end{align}
Recall that $\barbz_\tau$ is defined in \pref{eqn:barz_linear} and derivation in \pref{eqn:barz_lienar_alter} shows that
\begin{align*}
\barbz_{\pi_1(s)+1} = \frac{1}{N}\sum_{i=1}^N \sum_{t\in\mathcal{T}_{\pi_1(s)}}\inner{\ba_k,\bthetahat_t(i)}.
\end{align*}
Define
\begin{align}
\label{eqn:p_bar_ftrl_l}
\barbp_{\pi_1(s)}
\triangleq  \argmin_{\bq \in \Delta(K)} \left\{\sum_{s'=1}^{s-1} \left \langle \barbz_{\pi_1(s')+1},\bq \right \rangle + \frac{1}{\eta} \sum_{k=1}^K q(k) \log \left(q(k)\right)\right\}~,
\end{align}
which is the strategy output by FTRL when fed with loss vectors $\{\barbz_{\pi_1(s')+1}\}_{s'\in[s-1]}$.
Similar to the regret decomposition in \pref{lem: delayftrlmab}, we have
\begin{align}\label{eqn:decompose_linear}
    \Reg_{\mathcal{P}_1}(i)
&=
    \E \left[\sum_{s=1}^{m_1}\inner{\bp'_{\pi_1(s)}(i) - \be_k, \wtbz_{\pi_1(s)+1}^B(i)}\right]
\notag
\\&=
    \underbrace{\E \left[\sum_{s=1}^{m_1} \left\langle \bp'_{\pi_1(s)}(i) - \be_k,\wtbz_{\pi_1(s)+1}^B(i) - \barbz_{\pi_1(s)+1} \right\rangle\right]}_{\heartsuit}
\nonumber
\\&\quad
    + \underbrace{\E \left[\sum_{s=1}^{m_1} \left\langle \bp'_{\pi_1(s)}(i) - \barbp_{\pi_1(s)}, \barbz_{\pi_1(s)+1} \right\rangle\right]}_{\clubsuit}
\notag
\\&
    \quad + \underbrace{\E \left[\sum_{s=1}^{m_1} \left\langle \barbp_{\pi_1(s)} - \be_k, \barbz_{\pi_1(s)+1} \right\rangle\right]}_{\spadesuit}~.
\end{align}

\paragraph{Bounding the term $\clubsuit$.}
By Cauchy--Schwarz inequality, for each $s\in[m_1]$,
\[
    \left\langle \bp'_{\pi_1(s)}(i) - \barbp_{\pi_1(s)},\barbz_{\pi_1(s)+1} \right\rangle
\le
    \| \bp'_{\pi_1(s)}(i) - \barbp_{\pi_1(s)} \|_2  \cdot \|\barbz_{\pi_1(s)+1} \|_2~.
\]
Since $\bp'_{\pi_1(s)}(i)$ and $\barbp_{\pi_1(s)}$ follow the update rule of \pref{eqn:p_prime_ftrl_l} and \pref{eqn:p_bar_ftrl_l} and $\psi(\bq)=\frac{1}{\eta}\sum_{k=1}^Kq(k)\log(q(k))$ is $\frac{1}{\eta}$-strongly convex w.r.t. $\ell_2$-norm, \pref{lem: stablity_l} yields that
\[
\| \bp'_{\pi_1(s)}(i) - \barbp_{\pi_1(s)} \|_2 \le \frac{\eta}{K^{1.5}T^2}~.
\]
Moreover, since $\left|\ellhat_t(i,k)\right| \leq \frac{|\calS|}{\beta}$ for all $k\in\left[|\calS|\right]$ by \pref{lem:bounded_volumetric_l}, we have $|\barz_{\pi_1(s)+1}(k)| \le \frac{B|\calS|}{\beta}$ for all $k\in[K]$.
Hence,
\[
    \|\barbz_{\pi_1(s)+1}\|_2
\le
    \frac{B|\calS|\sqrt{K}}{\beta}~.
\]
Therefore,
\begin{align}
    \clubsuit
\le
    \sum_{s=1}^{m_1}\E\Big[ \|\bp'_{\pi_1(s)}(i) - \barbp_{\pi_1(s)}\|_2 \cdot \|\barbz_{\pi_1(s)+1} \|_2 \Big]
\le
    {m_1 \cdot \frac{\eta}{K^{1.5}T^2}\cdot \frac{B|\calS|\sqrt{K}}{\beta}
\leq
    1}
\label{ineq: clubsuit_l}
\end{align}
where we use { $\beta = 3Bd\eta$, $|\calS| \leq 3d$} and $m_1\le T$.

\paragraph{Bounding the term $\heartsuit$.}
By Cauchy-Schwarz inequality, for each $s\in[m_1]$, we have
\[
    \inner{\bp'_{\pi_1(s)}(i) - \be_k, \wtbz_{\pi_1(s)+1}^B(i) - \barbz_{\pi_1(s)+1}}
\le
    \left\|\bp'_{\pi_1(s)}(i) - \be_k \right\|_2 \cdot \left\|\wtbz_{\pi_1(s)+1}^B(i) - \barbz_{\pi_1(s)+1} \right\|_2~.
\]
Since $\bp'_{\pi_1(s)}(i)\in\Delta(K)$, we have $\|\bp'_{\pi_1(s)}(i) - \be_k \|_2 \le \sqrt{2}$. Moreover, according to \pref{lem: cons_linear}, we know that
\[
    \left\|\wtbz_{\pi_1(s)+1}^B(i) - \barbz_{\pi_1(s)+1}\right\|_2
\le
    \frac{1}{K^{1.5}T^2}~.
\]
Therefore,
\begin{align}
    \heartsuit
=
    \E\left[\sum_{s=1}^{m_1} \left\langle \bp'_{\pi_1(s)}(i) - \be_k , \wtbz_{\pi_1(s)+1}^B(i) - \barbz_{\pi_1(s)+1} \right\rangle\right]
\le
    \frac{\sqrt{2} m_1}{K^{1.5}T^2}
\le 
    2~,
\label{ineq: heartsuit_l}
\end{align}
where the last inequality uses $m_1\le T$.

\paragraph{Bounding the term $\spadesuit$.}
Recall that $\barz_{\pi_1(s)+1}\in\R^K$ and for all $k\in[K]$,
\[
\barz_{\pi_1(s)+1}(k)=\frac{1}{N}\sum_{t\in\mathcal T_{\pi_1(s)}}\sum_{i=1}^N \inner{\ba_k,\bthetahat_t(i)}\;,
\]
where
\[
\bthetahat_{t}(i)=M_{\pi_1(s)}(i)^{-1}\ba_{A_t(i)}\inner{ \ba_{A_t(i)},\btheta_t(i)},
\quad
M_{\pi_1(s)}(i)=\sum_{k\in[K]} p_{\pi_1(s)}(i,k) \ba_k \ba_k^\top ~.
\]
Using \pref{lem:bounded_volumetric_l}, for all $k \in [K]$ we know that {$|\langle \ba_k, \bthetahat_t(i)\rangle| \leq \frac{|\calS|}{\beta}$}, where {$\beta = 3Bd\eta$}. Hence we have $\eta \barz_{\pi_1(s)+1}(k)\in [-1,1]$. Since $\barbp_{\pi_1(s)}$ follows the update rule shown in \pref{eqn:p_bar_ftrl_l}, using the standard analysis of FTRL (e.g. Theorem 5.2 in \citet{hazan2016introduction}), we can obtain that
\begin{align}
\spadesuit
&=\mathbb{E}\left[\sum_{s=1}^{m_1}\left\langle \barbp_{\pi_1(s)}-\be_k, \barbz_{\pi_1(s)+1}\right\rangle\right] \notag\\
&\leq \frac{\log K}{\eta } + 2\eta \sum_{s=1}^{m_1}
\E \left [\sum_{k=1}^K \barp_{\pi_1(s)}(k)
\left(\frac{1}{N}\sum_{t\in \mathcal{T}_{\pi_1(s)}} \sum_{i=1}^N \widehat \ell_t(i,k)\right)^2 \right].
\label{eq: exp3_analysis_linear}
\end{align}
Define $\calF_{\pi_1(s)}$ to be the filtration generated by all random variables revealed up to the beginning of $\pi_1(s)$.
Since $\barp_{\pi_1(s)}(k)$ is $\mathcal F_{\pi_1(s)}$-measurable, by the tower property,
\begin{align}
\E\left[\barp_{\pi_1(s)}(k) \barz_{\pi_1(s)+1}^2(k)\right]
&=\E\left[\E\left[\barp_{\pi_1(s)}(k) \barz_{\pi_1(s)+1}^2(k) \mid \mathcal F_{\pi_1(s)}\right]\right]\notag\\
&=\E\left[\barp_{\pi_1(s)}(k) \E\left[\barz_{\pi_1(s)+1}^2(k) \mid \mathcal F_{\pi_1(s)}\right]\right]~.
\label{eq:tower_step_linear}
\end{align}
Conditioned on $\mathcal F_{\pi_1(s)}$, for all $t \in \mathcal{T}_{\pi_1(s)}$ we have
\begin{equation}
\E\left[\bthetahat_t(i)\mid\mathcal F_{\pi_1(s)}\right]
= M_{\pi_1(s)}^{-1}(i)\E\left[\ba_{A_t(i)}\ba_{A_t(i)}^\top\mid\mathcal F_{\pi_1(s)}\right]\btheta_t(i)
=\btheta_t(i)\;,
\end{equation}
and therefore for all $k \in [K]$,
\begin{equation*}
    \E \left[\inner{\ba_k,\bthetahat_t(i)}\mid\mathcal F_{\pi_1(s)}\right]
=\langle \ba_k,\btheta_t(i)\rangle = \ell_t(i,k)~.
\end{equation*}
Moreover, since $| \inner{\ba_k,\btheta_t(i)}|\le 1$, we have
\begin{align}
\E\left[\inner{\ba_k,\bthetahat_t(i)}^2\mid \mathcal F_{\pi_1(s)}\right] \notag
&=\E\left[\inner{\ba_k, M_{\pi_1(s)}^{-1}(i)\ba_{A_t(i)}}^2\inner{ \ba_{A_t(i)},\btheta_t(i)}^2 \mid \mathcal F_{\pi_1(s)}\right] \\
&\le \E\left[\inner{ \ba_k, M_{\pi_1(s)}^{-1}(i)\ba_{A_t(i)}}^2 \mid \mathcal F_{\pi_1(s)}\right]\notag\\
&= \ba_k^\top M_{\pi_1(s)}^{-1}(i)
\E\left[\ba_{A_t(i)}\ba_{A_t(i)}^\top \mid \mathcal F_{\pi_1(s)}\right]
M_{\pi_1(s)}^{-1}(i) \ba_k \notag\\
&= \ba_k^\top M_{\pi_1(s)}^{-1}(i) \ba_k~.
\label{eq:linear_second_moment}
\end{align}
Using the conditional independence of $A_t(i)$ across $t\in \mathcal{T}_{\pi_t(s)}$ and $i \in V$, we are able to bound $\E\left[\barz_{\pi_1(s)+1}(k)^2 \mid \mathcal F_{\pi_1(s)}\right]$ as follows:
\begin{align}
&\E\left[\barz_{\pi_1(s)+1}(k)^2 \mid \mathcal F_{\pi_1(s)}\right]\nonumber\\
&=\frac{1}{N^2}\E\left[\left.\left(\sum_{t\in\mathcal T_{\pi_1(s)}}\sum_{i=1}^N \inner{\ba_k,\bthetahat_t(i)}\right)^2 \right\vert \mathcal F_{\pi_1(s)}\right] \notag\\
&=\frac{1}{N^2}\left(\E\left[\left.\sum_{t\in\mathcal T_{\pi_1(s)}}\sum_{i=1}^N \inner{\ba_k,\bthetahat_t(i)}\right\vert \mathcal F_{\pi_1(s)}\right]\right)^2 \notag\\
&\qquad+ \frac{1}{N^2}\left(\E\left[\left.\left(\sum_{t\in\mathcal T_{\pi_1(s)}}\sum_{i=1}^N \inner{\ba_k,\bthetahat_t(i)-\btheta_t(i)}\right)^2\right\vert \mathcal F_{\pi_1(s)}\right]\right) \notag\\
&=\frac{1}{N^2}\left(\sum_{t\in\mathcal T_{\pi_1(s)}}\sum_{i=1}^N \ell_t(i,k)\right)^2 + \frac{1}{N^2}\sum_{t\in\mathcal T_{\pi_1(s)}}\sum_{i=1}^N\left(\E\left[\left. \inner{\ba_k,\bthetahat_t(i)}^2\right\vert \mathcal F_{\pi_1(s)}\right]\right) \notag\\
&\le B^2
+\frac{1}{N^2}\sum_{t\in\mathcal T_{\pi_1(s)}}\sum_{i=1}^N \ba_k^\top M_{\pi_1(s)}^{-1}(i) \ba_k,
\label{eq:cond_second_moment_linear}
\end{align}
where the last inequality holds because  $|\mathcal T_{\pi_1(s)}| = B$ and  $|\ell_t(i,k)|\le 1$.
Combining \pref{eq:tower_step_linear}--\pref{eq:linear_second_moment} yields
\begin{align}
\E\left[\sum_{k = 1}^{K}\barp_{\pi_1(s)}(k)\barz_{\pi_1(s)+1}(k)\right]
&\le \E\left[\sum_{k = 1}^{K}\barp_{\pi_1(s)}(k)\left(
B^2 + \frac{1}{N^2}\sum_{t\in\mathcal T_{\pi_1(s)}}\sum_{i=1}^N \ba_k^\top M_{\pi_1(s)}^{-1}(i)\ba_k
\right)\right]\notag\\
&= B^2
+ \frac{1}{N^2}\sum_{t\in\mathcal T_{\pi_1(s)}}\sum_{i=1}^N
\E\left[\sum_{k = 1}^{K}\barp_{\pi_1(s)}(k)\ba_k^\top M_{\pi_1(s)}^{-1} (i)\ba_k\right]\notag\\
&= B^2
+ \frac{1}{N^2}\sum_{t\in\mathcal T_{\pi_1(s)}}\sum_{i=1}^N
\E\left[\tr\left(M_{\pi_1(s)}^{-1}(i)\barM_{\pi_1(s)}\right)\right],
\label{eq:trace_step_linear}
\end{align}
where $\barM_{\pi_1(s)} = \sum_{k=1 }^{K}\barp_{\pi_1(s)}(k)\ba_k\ba_k^\top$. Then, according to \pref{lem: stablity_l}, we know that
\[
\frac{\barp_{\pi_1(s)}(k)}{p_{\pi_1(s)}(i,k)}\le\frac{\barp_{\pi_1(s)}(k)}{(1-\alpha -\beta)p'_{\pi_1(s)}(i,k)+\frac{\alpha}{K} + \frac{\beta}{|\calS|}\ind \{\ba_k\in \calS\}  } \le  6
\qquad\text{for all }i\in[N],~k\in[K]~.
\]
Therefore,
\[
\barM_{\pi_1(s)}=\sum_{k=1}^{K}\barp_{\pi_1(s)}(k)\ba_k\ba_k^\top
\preceq 6\sum_{k=1}^Kp_{\pi_1(s)}\ba_k\ba_k^\top
=6M_{\pi_1(s)}(i),
\]
and hence 
\[
\tr\left(M_{\pi_1(s)}(i)^{-1}\barM_{\pi_1(s)}\right)
\le \tr\left(M_{\pi_1(s)}(i)^{-1}\cdot 6M_{\pi_1(s)}(i)\right)
=6\tr(I_d)=6d.
\]
Plugging this into \pref{eq:trace_step_linear} gives
\begin{align}
\E\left[\sum_{k=1}^{K}\barp_{\pi_1(s)}(k)\barz_{\pi_1(s)+1}(k)^2\right]
&\le B^2 + \frac{1}{N^2}\sum_{t\in\mathcal T_{\pi_1(s)}}\sum_{i=1}^N 6d
\leq B^2 + \frac{6B d}{N}.
\label{eq:block_second_moment_bound_linear}
\end{align}
Substituting \pref{ineq: clubsuit_l}, and \pref{ineq: heartsuit_l}, \pref{eq: exp3_analysis_linear} and \pref{eq:block_second_moment_bound_linear}
into the decomposition \pref{eqn:decompose_linear}, we obtain
\begin{align*}
    \Reg_{\mathcal{P}_1(i)} \leq \frac{\log K}{\eta }  + 2\eta \left(m_1B^2 + \frac{6KBm_1}{N}\right) + 3~.
\end{align*}
The analysis for $\Reg_{\mathcal{P}_2}(i)$ is identical. Summing the bounds over the two parity
subsequences and using $m_1+m_2= T/B$ yields
\begin{align*}   \mathbb{E}\left[\sum_{\tau=1}^{M}\Big\langle \bp'_{\tau}(i) - \be_k, \bz_{\tau+1}^B(i)\Big\rangle\right]
&\le
    \frac{2\log K}{\eta} + 2\eta\left(TB+\frac{6dT}{N}\right) + 6
\notag
\\&\le
\frac{2\log K}{\eta} + 12\eta\left(TB+\frac{dT}{N}\right) + 6 \notag
\\&\le
    {\order\left(\sqrt{\log K\left(B+\frac{1}{N}\right)dT}+dB\log K\right)}
\end{align*}
where the last inequality follows by choosing {
\[
\eta= \min\left\{\frac{1}{6Bd},\sqrt{\frac{\log K}{dTB+\frac{dT}{N}}}\right\}.
\]}
\end{proof}


\subsection{Omitted Proof Details for \pref{lem: stablity_l}}
\begin{rlemma}
\label{lem: stablity_l}
Let $\alg$ be \pref{alg: bold} and $\mathcal{B}$ be an instance of \pref{alg: ftrl-linear} with a regularizer that is $\frac{1}{\eta}$-strongly convex in $\ell_2$-norm. Suppose that each agent uses \pref{alg: black-box-linear bandit} with $\kappa$ and $B$ defined in \pref{eqn: block}. Define $\barq_s^{(1)}$ and $q_s^{(1)}(i)$ for all $i\in[N]$
\begin{align*}
    \bq_s^{(1)}(i)&= \argmin_{\bq \in \Delta(K)} \left\{\sum_{s'=1}^{s-1} \left \langle \wtz_{\pi_1(s')+1}^B(i,\cdot), \bq \right \rangle +  \psi(\bq)\right\},\\
    \barbq_s^{(1)} &= \argmin_{q \in \Delta(K)} \left\{\sum_{s'=1}^{s-1} \left \langle \barbz_{\pi_1(s')+1}, q \right \rangle +  \psi(\bq)\right\},
\end{align*}
where $\barbz_\tau$ is defined in \pref{eqn:barz} and $\pi_1$ is defined in \pref{eqn: pi}. Then we have
\begin{equation}
    \|\barbq_{s}^{(1)} - \bq_{s}^{(1)}(i)\|_2 \leq \frac{\eta}{K^{1.5}T^2} \qquad \text{for all } i \in [N].
\end{equation}
    and 
    \begin{equation*}
        \frac{\barq_{s}^{(1)}(k)}{(1-\alpha -\beta)q_{s}^{(1)}(i,k)+\frac{\alpha}{K} + \frac{\beta}{|\calS|}\ind \{\ba_k \in \calS\}  }\leq 6 \quad \text{for all } i \in [N] \text{ and } k \in [K],
    \end{equation*}
    where $
    \alpha = \frac{1}{T}
$
and 
$
    \beta= 
    {3Bd\eta}\leq \frac{1}{2}$.
\end{rlemma}
\begin{proof}
Since $\psi(\bq)$ is $1$-strongly convex w.r.t. $\| \cdot \|_2$, using \pref{lem:stab}, we know that
\begin{align}
    \|\barbq_{s}^{(1)} - \bq_{s}^{(1)}(i)\|_2 &\leq \eta \left\|\sum_{s'=1}^{s -1} \wtbz_{\pi_1(s')+1} - \sum_{s'=1}^{s -1}
    \wtbz_{\pi_1(s')+1}^B(i)\right\|_2 \notag
    \\&\leq \eta\sum_{s'=1}^{s -1}\|\barbz_{\pi_1(s')+1} - \bz_{\pi_1(s')+1}^B(i)\|_2 \notag
    \\&\leq \frac{\eta}{K^{1.5}T}~, \notag
\end{align}
where the last inequality is due to \pref{lem: cons_linear} and $s\leq T$. As $\eta \leq 1$ and $\|a-b\|_{\infty} \leq\|a-b\|_2$ for any $a, b \in \mathbb{R}^K$, we have  
\begin{align*}
    \barq_{s}^{(1)}(k) - q_{s}^{(1)} \leq \frac{1}{KT} \qquad \text{for all } k \in [K].
\end{align*}
Note that
$
    \alpha = \frac{1}{T}
$
and 
$
    \beta= 
    {3Bd\eta}\leq \frac{1}{2}$. Direct calculation shows that for $\ba_k \in  \Omega \setminus \mathcal{S}$ and for $T \geq 3$,
\begin{align*}
     \frac{\barq_{s}^{(1)}(k)}{(1- \alpha -\beta)q_{s}^{(1)}(i,k)+\alpha/K} &\leq \frac{q_{s}^{(1)}(i,k)+\frac{1}{KT}}{(1-\alpha-\beta)q_{s}^{(1)}(i,k)+\alpha/K} \\
     &\leq \frac{1}{1-\alpha-\beta}
     \leq  6,
\end{align*}
and for $\ba_k \in  \mathcal{S}$ and for $T \geq 3$
\begin{align*}
     \frac{\barq_{s}^{(1)}(k)}{(1-\alpha -\beta)q_{s}^{(1)}(i,k)+\alpha/K + \beta/|\calS|} &\leq \frac{q_{s}^{(1)}(i,k)+\frac{1}{KT}}{(1-\alpha -\beta) q_{s}^{(1)}(i,k)+\alpha/K}  \leq 6~.
\end{align*}
\end{proof}

\begin{rlemma}\label{lem:bounded_volumetric_l}
    \pref{alg: black-box-linear bandit} 
    guarantees that $\left|\inner{\ba_k, \bthetahat_t(i)}\right|\leq \frac{|\calS|}{\beta}$ for all $k\in [K]$ and $t\in[T]$.
\end{rlemma}
\begin{proof}
~By definition of $\bthetahat_t(i,k)$ and Cauchy-Schwarz inequality, for all $t \in \mathcal{T}_{\tau}$, we have
\begin{equation}\label{eq:cs}
\left|\inner{\ba_k, \bthetahat_t(i)}\right|
\le \|\ba_k\|_{M_{\tau}(i)^{-1}}\cdot \|\ba_{A_t(i)}\|_{M_{\tau}(i)^{-1}}\cdot \bigl|\ba_{A_t(i)}^\top \btheta_t(i)\bigr|.
\end{equation}
By the definition of $p_{\tau}(i)$, we know that
\[
M_{\tau}(i)\succeq \frac{\beta}{|\calS|} \sum_{k \in [|\calS|]} \bb_k \bb_k^\top \triangleq \frac{\beta}{|\calS|} \Sigma_{\calS},
\]
so $M_{\tau}(i)^{-1}\preceq \frac{|\calS|}{\beta}\Sigma_{\calS}^{-1}$, and therefore for any $j\in [K]$,
\begin{align*}
    \|\ba_j\|_{M_{\tau}(i)^{-1}}^2 = \ba_j^\top M_{\tau}(i)^{-1}\ba_j
\le \frac{|\calS|}{\beta}  \ba_j^\top \Sigma_{\calS}^{-1}\ba_j.
\end{align*}
It remains to show that $\ba_j^\top \Sigma_{\calS}^{-1}\ba_j\le 1$ for all $j\in[K]$.
Let $S:=[\bb_1,\cdots,\bb_{|\calS|}]\in\mathbb{R}^{d\times |\calS|}$, so that
$\Sigma_{\calS}=SS^\top$.
Since $\calS$ is a volumetric spanner of $\Omega$, for every $j \in [K]$ there exists $\blambda^{(j)}\in\mathbb{R}^{|\calS|}$
with $\|\blambda^{(j)}\|_2\le 1$ such that $\ba_j=S\blambda^{(j)}$.
Hence
\begin{align*}
    \ba_j^\top \Sigma_{\calS}^{-1}\ba_j
= \left(\blambda^{(j)}\right)^\top S^\top (SS^\top)^{-1} S \blambda^{(j)}
= \left(\blambda^{(j)}\right)^\top P \blambda^{(j)},
\end{align*}
where $P\triangleq S^\top (SS^\top)^{-1} S$.
Because $SS^\top$ is invertible, $P$ is the orthogonal
projector, and therefore $0\preceq P\preceq I$.
It follows that
\begin{align*}
\ba_j^{\top}\Sigma_{\calS}^{-1}\ba_j
= \left(\blambda^{(j)}\right)^\top P \blambda^{(j)}
\le \left(\blambda^{(j)}\right)^\top I \blambda^{(j)}
= \|\blambda^{(j)}\|_2^2
\le 1~.
\end{align*}
Therefore, we have for all $j\in[K]$,
\[
\|\ba_j\|_{M_{\tau}(i)^{-1}}\le \sqrt{\frac{|\calS|}{\beta}}~.
\]
Finally, since $|\ba_j^\top \btheta_t(i)|\le 1$ for all $j \in [K]$, $i \in [N]$, $t\in [T]$.
plugging into \pref{eq:cs} gives
\[
\left|\inner{\ba_k, \bthetahat_t(i)}\right|
\le \sqrt{\frac{|\calS|}{\beta}}\cdot \sqrt{\frac{|\calS|}{\beta}}\cdot 1
= \frac{|\calS|}{\beta}.
\]
This proves the lemma.
\end{proof}

\section{Omitted Details in \pref{sec: lowerbound}}\label{app: lower}
In this section, we show the omitted proofs in \pref{sec: lowerbound}.
\lowerBound*
\begin{proof}
The $\Omega\big(\sqrt{KT/N}\big)$ term in the lower bound is obtained by applying \citep[Theorem~2]{seldin2014prediction}, which holds for the regret of any single agent in the simpler $K$-armed bandit setting when $G$ is a clique and the local losses are identical for all agents. The other term in the lower bound is obtained by applying the
\begin{equation}
\label{eq:lower-doubly}
    \Omega\left(\left(\frac{\dmax}{\lambda_{N-1}(L)}\right)^{1/4}\sqrt{(\log K)T}\right)
\end{equation}
lower bound proven in \citep[Theorem~3.1]{yi2023doubly} for a certain graph $G$ with Laplacian matrix $L$, maximum degree $\dmax$, and where $\lambda_{N-1}(L)$ is the smallest non-zero eigenvalue of $L$. In order to relate $\lambda_{N-1}(L)$ to $\rho(W) = 1-\sigma_2(W)$, we choose $W = I - \frac{\alpha}{2\dmax} L$ where $0 < \alpha < 1$. It is easy to check that $W$ is a gossip matrix with second largest singular value $\sigma_2(W) = 1 - \frac{\alpha}{2\dmax}\lambda_{N-1}(L)$. Since $0 \le \lambda_i(L) \le 2\dmax$ for all $i\in [N]$ \citep{anderson1985eigenvalues}, we have $\sigma_2(W) > 0$. Hence we can write
\(
    \frac{1}{2(1-\sigma_2(W))} = \frac{\dmax}{\alpha\lambda_{N-1}(L)},
\)
which, substituted in \pref{eq:lower-doubly}, delivers the desired other term in the lower bound.
\end{proof}
\LinearLowerBound*
\begin{proof}
To prove the first term we exploit the fact that $K$-armed bandits are a special case of linear bandits, and invoke \pref{thm:lowerbound}. For the second term, we use \citet[Theorem 4]{ito2020delay} where $\Omega = \{-1,1\}^d$, so that $\log|\Omega| = d$.
\end{proof}

\end{document}